\documentclass{article}
\usepackage{iclr2027_conference,times}
\iclrfinalcopy

\usepackage{amsmath,amsfonts,bm}

\def\eqref#1{equation~\ref{#1}}

\def\1{\bm{1}}

\DeclareMathAlphabet{\mathsfit}{\encodingdefault}{\sfdefault}{m}{sl}
\SetMathAlphabet{\mathsfit}{bold}{\encodingdefault}{\sfdefault}{bx}{n}

\usepackage{amsmath,amsthm,amssymb}
\usepackage{algorithm}
\usepackage{algpseudocode}
\usepackage{array}
\usepackage{booktabs}
\usepackage{capt-of}
\usepackage{courier}
\usepackage{graphicx}
\usepackage{listings}
\usepackage{longtable}
\usepackage{multirow}
\usepackage{placeins}
\usepackage{float}
\usepackage{ragged2e}
\usepackage{tcolorbox}
\usepackage[table]{xcolor}
\definecolor{mydarkred}{rgb}{0.6,0,0}
\definecolor{mydarkgreen}{rgb}{0,0.6,0}
\usepackage[colorlinks,linkcolor=mydarkred,citecolor=mydarkgreen,urlcolor=mydarkred]{hyperref}
\usepackage{url}
\tcbuselibrary{breakable,listings,skins}

\renewcommand{\eqref}[1]{Eq.~(\ref{#1})}

\newtcblisting{promptbox}[1]{
  enhanced,
  breakable,
  listing only,
  listing engine=listings,
  title={#1},
  fonttitle=\bfseries\small,
  coltitle=white,
  colbacktitle=black,
  colback=black!1,
  colframe=black,
  boxrule=0.7pt,
  toptitle=3pt,
  bottomtitle=3pt,
  arc=4pt,
  outer arc=4pt,
  boxsep=0pt,
  left=7pt,
  right=7pt,
  top=5pt,
  bottom=5pt,
  before skip=4pt,
  after skip=7pt,
  listing options={basicstyle=\fontfamily{pcr}\selectfont\scriptsize,columns=fullflexible,keepspaces=true,showstringspaces=false,breaklines=true,breakatwhitespace=true,breakindent=0pt,breakautoindent=false}
}

\newcolumntype{P}[1]{>{\RaggedRight\arraybackslash}p{#1}}
\newcommand{\Tok}{\operatorname{Tok}}
\newcommand{\Dec}{\operatorname{Dec}}

\theoremstyle{plain} \newtheorem{theorem}{Theorem}[section] \newtheorem{lemma}[theorem]{Lemma} \newtheorem{proposition}[theorem]{Proposition} 

\title{DenMark: Robust Semantic Watermarking for Diffusion Language Models}

\author{
\textbf{Tianhao Ma}$^{1}$ \quad
\textbf{Weihao Xuan}$^{1,2}$ \quad
\textbf{Dong-Dong Wu}$^{1,2}$ \quad
\textbf{Farshid Nooshi}$^{1}$\\
\textbf{Takashi Ishida}$^{2,1}$ \quad
\textbf{Gang Niu}$^{2}$ \quad
\textbf{Naoto Yokoya}$^{1,2}$ \quad
\textbf{Masashi Sugiyama}$^{2,1}$\\[0.6em]
\normalfont $^{1}$The University of Tokyo\\
\normalfont $^{2}$RIKEN Center for Advanced Intelligence Project\\
}
\date{}
\begin{document}

\maketitle
\lhead{arXiv preprint}

\begin{abstract}
Semantic text watermarks encode signals in meaning rather than surface token choices, offering robustness to paraphrasing and other semantic-preserving edits. Existing semantic watermarking methods are primarily designed for autoregressive language models (ARLMs), where completed candidate units can be generated and scored before generation proceeds. This paradigm does not naturally extend to diffusion language models (DLMs), where semantic units remain incomplete during intermediate denoising steps and tokens may be updated in flexible orders. We propose \textbf{DenMark}, a semantic watermarking framework that injects key-dependent signals directly into the DLM denoising process. DenMark partitions the output into fixed token regions and uses temporary rollouts as semantic lookahead: conditional completions estimate the eventual semantics of an incomplete region, enabling DenMark to select local updates with higher estimated semantic watermark scores. Repeating this procedure across denoising steps progressively accumulates watermark evidence in the final output. For detection, DenMark uses calibrated scanning over candidate unit sizes to remain robust to boundary shifts introduced by semantic attacks.
Across four DLMs, three datasets, and four semantic attacks, DenMark gives the strongest overall robustness and achieves the best TPR at each reported operating point in 46 of the 48 backbone--dataset--attack combinations. Overall, DenMark provides an effective semantic watermark for DLMs. Code is available at \url{https://github.com/TianhaoMa5/DenoMark}.
\end{abstract}

\section{Introduction}

Large language models (LLMs) have achieved strong performance across a broad range of language generation and reasoning tasks~\citep{brown2020language,chen2021evaluating,chowdhery2023palm,touvron2023llama,kasneci2023chatgpt,singhal2023large,lin2025creative4u,wang2026badit}. Most modern language models are autoregressive language models (ARLMs), generating tokens sequentially from left to right conditioned on the preceding prefix~\citep{brown2020language,chowdhery2023palm,touvron2023llama}. Recently, diffusion language models (DLMs) have emerged as an alternative generation paradigm. Rather than predicting the next token from a completed prefix, DLMs iteratively denoise partially masked sequences, enabling bidirectional context and flexible updates across multiple positions~\citep{lou2024discrete,sahoo2024simple,gong2024scaling,nie2025large,ye2025dream}. This formulation supports arbitrary-order generation, infilling, parallel refinement, and flexible quality--speed trade-offs.

The growing deployment of LLM-generated text has raised concerns about misinformation, malicious use, bias, and failures in high-stakes domains, motivating reliable mechanisms for identifying model-generated content~\citep{bender2021dangers,bommasani2021opportunities,weidinger2021ethical}. Text watermarking embeds a hidden statistical signal into generated text while preserving its visible content~\citep{kirchenbauer2023watermark,dathathri2024scalable}. Most existing language-model watermarks encode token-level signals by modifying the decoding distribution or token-selection process~\citep{kuditipudi2023robust,hu2023unbiased,dathathri2024scalable}; a prominent example uses key-dependent green--red vocabulary partitions~\citep{kirchenbauer2023watermark}. Because these signals depend on token identities, paraphrasing and rewriting can substantially weaken them while preserving the underlying meaning. Semantic watermarks instead encode signals in sentence or span representations and have shown stronger robustness to such transformations~\citep{hou2024semstamp,ren2024robust,hou2024ksemstamp,dabiriaghdam2025simmark,huo2026pmark}.

A prominent line of semantic watermarking methods relies on complete sentence-level candidate generation and selection~\citep{hou2024semstamp,hou2024ksemstamp,dabiriaghdam2025simmark,huo2026pmark}. Given a fixed prefix, these methods generate a set of complete sentence candidates, score them using a key-dependent semantic watermark criterion, reject candidates that do not satisfy the criterion, and commit an accepted candidate before continuing generation. This procedure is naturally compatible with ARLMs, whose prefix remains fixed while the next sentence is completed and evaluated.

This sentence-level procedure does not directly extend to DLMs. During iterative denoising, tokens at multiple locations can be updated before a sentence or span is fully resolved~\citep{lou2024discrete,sahoo2024simple,nie2025large,ye2025dream}. Under this generation process, it is impractical to construct a set of complete semantic-unit candidates during normal decoding, score and reject them, commit one candidate, and then resume generation. A semantic watermark for DLMs therefore needs to inject semantic signals directly through local updates as denoising proceeds, even when the surrounding semantic unit remains incomplete. Figure~\ref{fig:denomark_motivation} illustrates this mismatch and the rollout-guided alternative implemented by DenMark.

\begin{figure*}[t]
\centering
\includegraphics[width=\textwidth]{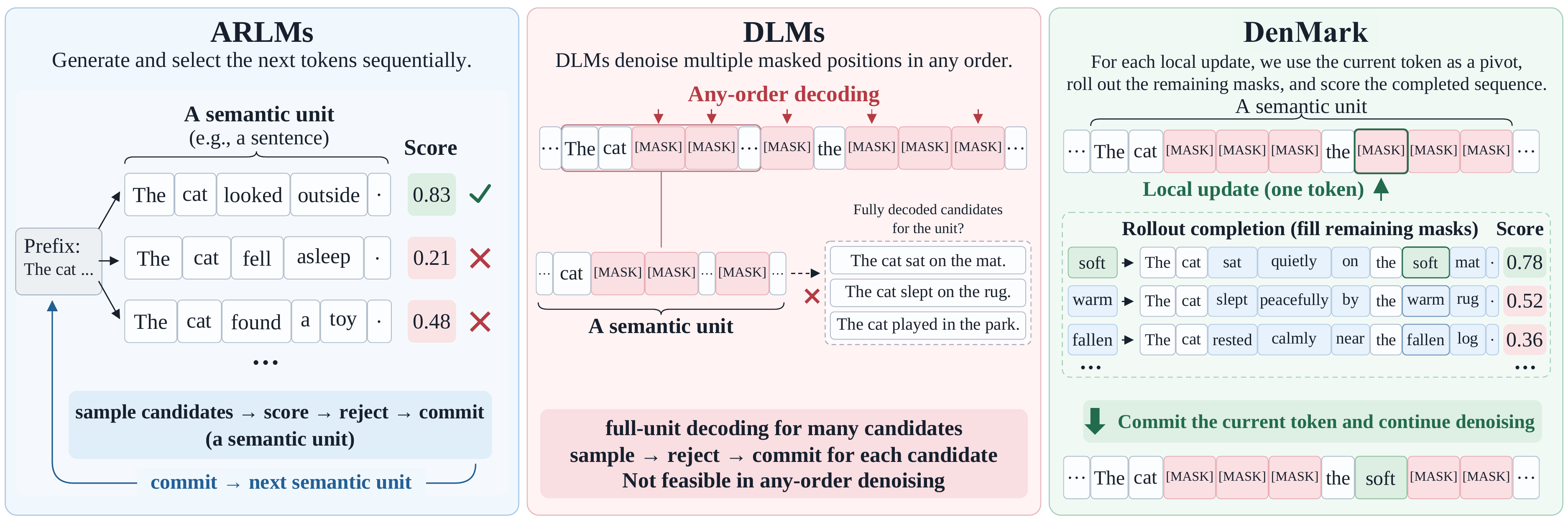}
\caption{\textbf{Why semantic watermarks must be injected during DLM denoising.} \textbf{ARLMs:} A complete next semantic unit can be generated and scored before it is committed. \textbf{DLMs:} Arbitrary-order denoising exposes only partially resolved units, making repeated full-unit candidate generation and rejection impractical. \textbf{DenMark:} Temporary rollouts complete the active unit for each local candidate; only the selected local update is committed before denoising continues.}
\label{fig:denomark_motivation}
\vspace{-5pt}
\end{figure*}

We propose \textbf{DenMark}, a semantic watermarking method designed for DLMs. DenMark partitions the generated sequence into fixed token regions and, for each candidate local update, uses temporary rollouts to complete unresolved positions and estimate the resulting semantic score. It then commits the update with the highest estimated score, progressively injecting key-dependent watermark evidence throughout denoising; our analysis characterizes the final score gain as the accumulated contributions of these local selections. At detection time, semantic attacks can alter token boundaries, so we use calibrated scanning over candidate region sizes to recover watermark evidence under such shifts. Experiments across multiple DLMs and datasets show that DenMark substantially improves robustness to semantic attacks over token-level watermark baselines.

Our main contributions are summarized as follows:
\begin{itemize}
\setlength{\topsep}{3pt}
\setlength{\itemsep}{2pt}
\setlength{\parsep}{0pt}
\item We identify the mismatch between sentence-level semantic watermarks and DLM denoising and present the first DLM-specific semantic watermarking framework.
\item We introduce a rollout-guided algorithm that scores partially masked regions and selects local updates that inject key-dependent evidence during denoising.
\item We develop calibrated unit-size scanning for boundary shifts and evaluate DenMark across DLMs, datasets, and semantic attacks, demonstrating robust detection after rewriting.
\end{itemize}

\section{Preliminaries}
We briefly review masked DLM generation and define the semantic watermark score used throughout the paper.
\paragraph{Masked DLMs.}
Recent DLMs generate text through iterative denoising of masked sequences~\citep{lou2024discrete,sahoo2024simple,nie2025large}. Let \(\mathcal V\) be the ordinary token vocabulary and let \([\mathrm{MASK}]\notin\mathcal V\) be the mask token. Let \(x_{\mathrm p}=(x_1,\ldots,x_{L_{\mathrm p}})\in\mathcal V^{L_{\mathrm p}}\) be a tokenized prompt, and let \(\Tok\) and \(\Dec\) denote the tokenization and detokenization maps, respectively. For a target continuation length \(N\), decoding starts from
\[
x^{(0)}= (x_1,\ldots,x_{L_{\mathrm p}},
\underbrace{[\mathrm{MASK}],\ldots,[\mathrm{MASK}]}_{N})
\in
(\mathcal V\cup\{[\mathrm{MASK}]\})^{L_{\mathrm p}+N}.
\]
At step \(t\), the masked generation positions are
\[
\mathcal M_t
=
\{i\in\{L_{\mathrm p}+1,\ldots,L_{\mathrm p}+N\}:x_i^{(t)}=[\mathrm{MASK}]\}.
\]
For each \(i\in\mathcal M_t\), the model \(p_\phi\) provides a categorical distribution \(p_\phi(\cdot\mid x^{(t)},i)\) over \(\mathcal V\). If multiple positions are updated in parallel, their tokens are sampled conditionally independently from the corresponding position-wise distributions given \(x^{(t)}\). A decoding rule commits the sampled tokens to obtain \(x^{(t+1)}\). Decoding terminates at step \(T\) when \(\mathcal M_T=\emptyset\), yielding the completed continuation tokens
\[
\hat z=x^{(T)}_{L_{\mathrm p}+1:L_{\mathrm p}+N}\in\mathcal V^N
\]
and the decoded output text \(\hat y=\Dec(\hat z)\).

\paragraph{Semantic watermarking.}
Semantic watermarks encode signals in the meaning of a text unit, which improves robustness to paraphrasing and rewriting~\citep{hou2024semstamp,ren2024robust,huo2026pmark}. Let \(\mathcal T\) denote the set of finite decoded text strings, and let \(u\in\mathcal T\) be a watermarking unit, typically a sentence or short text span. We use a fixed semantic encoder \(E_\eta:\mathcal T\to\mathbb R^d\), where \(\eta\) denotes its parameters and \(d\) is the embedding dimension. Let \(\mathcal K\) be the key space and \(\kappa\in\mathcal K\) the secret key. For every unit index \(b\geq1\), the key deterministically specifies \(C\) channel pairs \(\{(\theta_{b,j},s_{b,j})\}_{j=1}^{C}\), where \(\theta_{b,j}\in\mathbb R^d\) is a semantic direction and \(s_{b,j}\in\{-1,+1\}\) specifies its target side. Each \(\theta_{b,j}\) is obtained by normalizing a pseudorandom Gaussian vector generated from the key, and an independent random bit generated from the key determines \(s_{b,j}\). The detector can regenerate these directions and signs from the same key. We define the semantic score of a completed unit as
\begin{equation}
\label{eq:keyed_unit_score}
g_b(u)
=
\frac{1}{C}\sum_{j=1}^{C}
s_{b,j}\langle E_\eta(u),\theta_{b,j}\rangle.
\end{equation}
Here, \(\langle a,b\rangle=a^\top b\) denotes the Euclidean inner product on \(\mathbb R^d\). A larger \(g_b(u)\) indicates stronger alignment with the target specified by the key.

For an integer \(q\geq1\), write \([q]=\{1,\ldots,q\}\). Let \(m\geq1\) be an integer semantic unit size and \(B=\lceil N/m\rceil\) the number of semantic units, where \(\lceil a\rceil\) denotes the smallest integer not less than \(a\). For completed unit texts \(u_1,\ldots,u_B\), our generation goal is to maximize the aggregate keyed score \(\sum_{b=1}^{B}g_b(u_b)\), thereby accumulating evidence that can be recovered by the detector. Directly maximizing the aggregate objective defined above is difficult because token choices and semantic representations across units are coupled through the generated context. Exact optimization therefore requires a combinatorial search over the continuation and is intractable in practice. We instead decompose the global objective into a sequence of local decisions, where each decision independently maximizes the score of a single semantic unit, without jointly optimizing the scores of multiple units.
\section{Method}
We partition the continuation into fixed token regions and treat each region as a semantic unit. During denoising, an active unit is only partially observed, so its final semantic representation cannot be scored directly. For each local candidate update at step \(t\), we sample \(R_t\) approximate completions of the active unit and average their semantic watermark scores. We then commit the candidate with the largest rollout estimate. Repeating these local optimizations during denoising accumulates keyed evidence across completed units.

\subsection{Local Candidate Sampling within a Semantic Unit}
\label{sec:local_candidate_sampling}
For each \(b\in[B]\), define
\[
\mathcal U_b = \left\{ L_{\mathrm p}+(b-1)m+1,\ldots, \min\{L_{\mathrm p}+bm,L_{\mathrm p}+N\} \right\}.
\]
At denoising step \(t\), the decoding schedule may select one or more unfinished semantic units. Consider a selected unit \(b\in[B]\). Its remaining masked positions are \(\mathcal M_{b,t} = \mathcal U_b\cap\mathcal M_t\). If \(\mathcal M_{b,t}\neq\emptyset\), let \(K_{t,b}\geq1\) be the prescribed number of local candidates at step \(t\) for unit \(b\), and let \(r_t\geq1\) be the integer budget for position updates in each candidate. For each \(k\in[K_{t,b}]\), the position policy samples
\[
\mathcal I_{t,b,k} \sim \pi_{\mathrm{pos}} \left( \cdot \mid x^{(t)},\mathcal M_{b,t},r_t \right),
\]
with support restricted to subsets satisfying
\[
\mathcal I_{t,b,k}\subseteq\mathcal M_{b,t},
\qquad
|\mathcal I_{t,b,k}|=\min\{r_t,|\mathcal M_{b,t}|\}.
\]
The candidate is initialized as \(c_{t,b,k}\leftarrow x^{(t)}\). For every \(i\in\mathcal I_{t,b,k}\), its token is sampled according to
\[
(c_{t,b,k})_i \sim \pi_{\mathrm{tok}} \left( \cdot \mid p_\phi(\cdot\mid x^{(t)},i) \right),
\]
All other positions keep their values from \(x^{(t)}\). The distribution \(\pi_{\mathrm{pos}}\) may use random or confidence based position selection. The token distribution \(\pi_{\mathrm{tok}}\) is defined over \(\mathcal V\) and may use categorical, Gumbel max, or top \(p\) sampling.

\subsection{Semantic Scoring with Rollouts and Candidate Selection}
\label{sec:rollout_scoring}
For each candidate \(c_{t,b,k}\), we draw \(R_t\) one-step rollout completions conditioned on the candidate state. For rollout \(\rho\in[R_t]\), initialize
\(
\tilde c_{t,b,k,\rho}\leftarrow c_{t,b,k}.
\)
The remaining masked positions in the current unit are \(\mathcal M_{b,t}\setminus\mathcal I_{t,b,k}\). For each such position, we sample directly from the candidate-conditioned model distribution:
\[
(\tilde c_{t,b,k,\rho})_i
\sim
p_\phi(\cdot\mid c_{t,b,k},i),
\qquad
i\in\mathcal M_{b,t}\setminus\mathcal I_{t,b,k},
\]
using the rollout temperature, which may differ from the candidate temperature used by \(\pi_{\mathrm{tok}}\). The resulting completed text of unit \(b\) is
\(
\tilde u_{t,b,k,\rho}
=
\Dec\!\left(\big((\tilde c_{t,b,k,\rho})_i\big)_{i\in\mathcal U_b}\right).
\)
These rollout tokens are used only for scoring and are never committed. Using the semantic score in~\eqref{eq:keyed_unit_score}, we compute the rollout estimate for candidate \(k\) as
\begin{align}
\label{eq:rollout_candidate_score}
W_{t,b,k}
&=
\frac{1}{R_t}\sum_{\rho=1}^{R_t}
g_b(\tilde u_{t,b,k,\rho}).
\end{align}
We select
\begin{equation}
\label{eq:candidate_selection}
k_{t,b}^\star\in\arg\max_{k\in[K_{t,b}]}W_{t,b,k},
\end{equation}
and commit the corresponding local update:
\begin{equation}
\label{eq:candidate_commit}
x_i^{(t+1)}
=
\begin{cases}
(c_{t,b,k_{t,b}^\star})_i,
& i\in\mathcal I_{t,b,k_{t,b}^\star},\\
x_i^{(t)},
& \text{otherwise}.
\end{cases}
\end{equation}
When multiple semantic units are selected at the same denoising step, we apply~\eqref{eq:candidate_selection}--\eqref{eq:candidate_commit} to each selected unit and commit the resulting updates simultaneously because their position sets do not overlap.
DenMark evaluates $K_{t,b}$ candidates with $R_t$ rollouts per candidate, introducing additional generation and scoring cost. This reflects the general trade-off in sampling-based semantic watermarking between broader candidate exploration and generation efficiency~\citep{hou2024semstamp,huo2026pmark}. In practice, the measured runtime of DenMark remains acceptable, as reported in Section~\ref{sec:generation_efficiency}.
\section{Detection}
\label{sec:detection}
We first define detection at a fixed unit size and then extend it by scanning calibrated candidate unit sizes to recover evidence after boundary shifts.

\subsection{Detection at a Fixed Unit Size}
\label{sec:standard_detection}
Given a text \(y\), the detector tokenizes it as
\[
z(y)=\Tok(y)=(z_1,\ldots,z_L)\in\mathcal V^L.
\]
For a positive integer candidate unit size \(m'\), define
\[
n_{m'} = \left\lceil\frac{L}{m'}\right\rceil, \qquad \mathcal J_b^{(m')} = \{(b-1)m'+1,\ldots,\min\{bm',L\}\},
\]
for \(b\in[n_{m'}]\). The corresponding decoded text unit is
\[
u_b^{(m')}(y) = \Dec\!\left((z_i)_{i\in\mathcal J_b^{(m')}}\right),
\]
where the final unit may contain fewer than \(m'\) tokens. For each unit index \(b\), the detector regenerates from the secret key \(\kappa\) the same channel pairs \(\{(\theta_{b,j},s_{b,j})\}_{j=1}^{C}\) used during watermark generation. Using the same semantic encoder \(E_\eta\) as in generation, the watermark score under candidate unit size \(m'\) is
\begin{equation}
\label{eq:fixed_unit_score}
S_{m'}(y)
=
\frac{1}{n_{m'}}
\sum_{b=1}^{n_{m'}}
\frac{1}{C}
\sum_{j=1}^{C}
s_{b,j}
\left\langle
E_\eta\!\left(u_b^{(m')}(y)\right),
\theta_{b,j}
\right\rangle.
\end{equation}
The score \(S_{m'}(y)\) measures the average alignment between the detected semantic units and their watermark directions. Detection at the generation size uses \(m\) and compares \(S_m(y)\) with a threshold calibrated on unwatermarked texts at the target false positive rate.

\subsection{Robust Detection by Scanning Unit Sizes}
\label{sec:robust_detection}
Paraphrasing, insertion, deletion, and retokenization can move the unit boundaries observed by the detector. We scan a fixed finite set of positive candidate unit sizes \(\mathcal G\subset\mathbb N_{+}\) to recover evidence across these shifts. Let \(\mathcal D_0\) be a calibration set of unwatermarked texts that is disjoint from all evaluated texts. Let \(h(y)\) assign text \(y\) to a fixed length bin and define \(\mathcal D_{0,\ell}=\{y_0\in\mathcal D_0:h(y_0)=\ell\}\). For each \(m'\in\mathcal G\), we compute \(S_{m'}(y)\) and calibrate it separately using only \(\mathcal D_{0,h(y)}\). Using the empirical \(p\) value construction of~\citep{phipson2010permutation}, we define
\[
p_{m'}(y) =
\frac{
1+\sum_{y_0\in\mathcal D_{0,h(y)}}
\mathbb{I}\!\left\{
S_{m'}(y_0)\geq S_{m'}(y)
\right\}
}{
|\mathcal D_{0,h(y)}|+1
}.
\]
Here, \(\mathbb I\{A\}\) equals one when event \(A\) holds and zero otherwise. Smaller \(p_{m'}(y)\) values indicate stronger watermark evidence for candidate size \(m'\). We evaluate several unit sizes and apply a Bonferroni correction~\citep{dunn1961multiple}:
\[
p_{\mathrm{scan}}(y) = \min\!\left\{ 1,\, |\mathcal G| \min_{m'\in\mathcal G}p_{m'}(y) \right\}.
\]
The corresponding robust detection score is
\[
S_{\mathrm{rob}}(y) = -\log p_{\mathrm{scan}}(y).
\]
At significance level \(\alpha\in(0,1)\), the detector classifies \(y\) as watermarked when
\(
p_{\mathrm{scan}}(y)\leq\alpha.
\)
\begin{proposition}[False positive control for calibrated scanning]
\label{prop:scan_fpr_control}
Let \(Y\) denote a random test text under the unwatermarked null distribution. Suppose that \(\mathcal G\) is a fixed finite set and that, within each length bin \(\ell\), \(Y\) and the calibration texts in \(\mathcal D_{0,\ell}\) are conditionally exchangeable given \(h(Y)=\ell\). Then, for every \(\alpha\in[0,1]\),
\[
\Pr\!\left( p_{\mathrm{scan}}(Y)\leq\alpha \right) \leq\alpha.
\]
where the probability is taken under the unwatermarked null distribution. The Bonferroni corrected scan controls the false positive rate and does not require independence among scores from different unit sizes.
\end{proposition}
Appendix~\ref{app:scan_fpr_control_proof} gives the proof. All DenMark detection results use $S_{\mathrm{rob}}(y)$ as the ranking statistic unless stated otherwise. Although DenMark uses a simple fixed unit size during generation, it remains robust to non-uniform local length changes, as shown in Appendix~\ref{app:nonuniform_local_length_attacks}.\section{Why Rollout-Guided Selection Accumulates Watermark Evidence}
\label{sec:generation_algorithm_interpretation}
We compare the DenMark decoder \(\pi\) with a reference decoder \(\pi^0\) that uses the same candidate sampling rule but draws and directly commits one local candidate for each selected semantic unit. At each step \(t\), let \(\Delta_t^\pi\) denote the expected change in the final watermark score caused by committing the rollout-selected joint update chosen by DenMark rather than the coupled single-candidate reference update, with subsequent decisions following \(\pi\). Appendix~\ref{app:cumulative_continued_policy_decomposition_proof} formalizes the coupling and continuation values.

\begin{proposition}[Exact decomposition of watermark score improvement]
\label{prop:cumulative_continued_policy_decomposition}
The expected improvement in the final watermark score under \(\pi\) relative to the reference decoder \(\pi^0\) decomposes as
\begin{equation}
\label{eq:final_score_improvement}
\mathbb E_{\pi}\!\left[
S_m(\hat y)
\,\middle|\,
x^{(0)}
\right]
-
\mathbb E_{\pi^0}\!\left[
S_m(\hat y)
\,\middle|\,
x^{(0)}
\right]
=
\sum_{t=0}^{T-1}\Delta_t^{\pi}.
\end{equation}
\end{proposition}

Proposition~\ref{prop:cumulative_continued_policy_decomposition} shows that the final watermark-score improvement is exactly the accumulated advantage of rollout-guided selection along the denoising trajectory.

DenMark uses the rollout score to select a local candidate for each active semantic unit. We estimate the step advantage at each decoding step over 30 prompts for LLaDA-8B and Dream-v0-Instruct-7B. Figure~\ref{fig:reverse_hybrid_cumulative_delta} shows the cumulative estimates. Some individual steps have neutral or negative values, and the mean cumulative advantage increases over the trajectory and ends positive for both models.
\begin{figure}[!htb]
    \centering
    \includegraphics[width=0.90\textwidth]{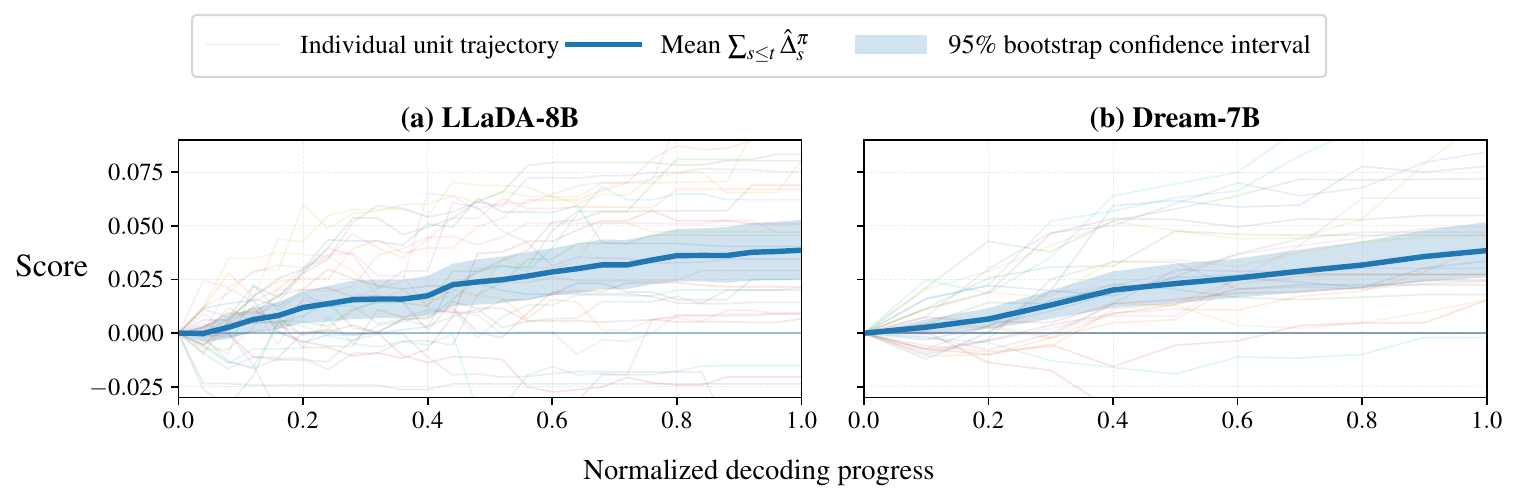}
    \caption{\textbf{Accumulation of empirical step advantages during decoding.} Thin curves show individual unit trajectories. The thick blue curve shows the mean cumulative advantage, and the shaded region gives the 95\% bootstrap confidence interval over 30 prompts. The mean cumulative advantage increases over the trajectory and remains positive near the end of decoding for both LLaDA-8B and Dream-v0-Instruct-7B.}
    \label{fig:reverse_hybrid_cumulative_delta}
\end{figure}

At the final decoding step, the mean cumulative advantages are \(0.0387\) for LLaDA-8B and \(0.0385\) for Dream-v0-Instruct-7B, with 95\% bootstrap confidence intervals \([0.0251,0.0527]\) and \([0.0260,0.0517]\), respectively. The positive final values show that candidate selection with rollouts adds watermark signal over the decoding trajectory.

\section{Experiments}
\label{sec:experiments}
We evaluate detection under semantic and token attacks, analyze sensitivity to the main hyperparameters and the unit size scan, and compare generation time across methods.

\subsection{Experimental Setup}
\label{sec:experimental_setup}

We evaluate LLaDA-8B~\citep{nie2025large}, LLaDA1.5-8B~\citep{zhu2025llada15}, LLaDA2.0-mini~\citep{bie2025llada20}, and Dream-v0-Instruct-7B~\citep{ye2025dream}. Following WaterBench~\citep{tu2024waterbench}, we use Finance~\citep{maia2018fiqa}, Alpaca~\citep{dubois2023alpacafarm}, and LongForm~\citep{fan2019eli5}. We compare DenMark with four recent state-of-the-art watermarking baselines: DLM-KGW~\citep{gloaguen2025watermarking}, DGMark~\citep{hong2026dgmark}, PatternMark~\citep{chen2025patternmark}, and UMR~\citep{yang2026umr}. We report AUC and TPR at false positive rates of $0.5\%$, $1\%$, and $5\%$. At each operating point, we use the empirical score threshold that gives the highest TPR while keeping the measured false positive rate on the held-out negatives no greater than the target. We tune generation hyperparameters for each method to keep output quality comparable across methods. Table~\ref{tab:generation_quality_deltas} reports the resulting quality measurements, and Appendix~\ref{app:experimental_details} provides baseline descriptions, parameter settings, and complete experimental configurations.

\subsection{Semantic Attacks Using Open Source Models}
\label{sec:opensource_semantic_attacks}

We evaluate two open source paraphrasers, Parrot~\citep{prithivida2021parrot} and DIPPER~\citep{krishna2023paraphrasing}. For Parrot, the candidate prefix size $q\in\{1,4,7,10\}$ is the number of paraphrase candidates generated per sentence before bigram-overlap selection. For DIPPER, we vary lexical diversity over $\{20,40,60,80\}$ and keep the remaining decoding settings fixed. Larger lexical diversity values produce stronger changes in wording while preserving semantics.

We report Finance-QA and AlpacaFarm separately, macro-averaging over the LLaDA-8B and Dream-v0-Instruct-7B backbones, and report TPR at $0.5\%$ and $1\%$ FPR. Figure~\ref{fig:opensource_semantic_attacks} shows that DenMark remains robust under Parrot and loses detection more slowly than the baselines as DIPPER lexical diversity increases. The same pattern also appears under the GPT-4o-mini attacks reported next.

\FloatBarrier
\begin{figure}[!htb]
    \centering
    \includegraphics[width=\textwidth]{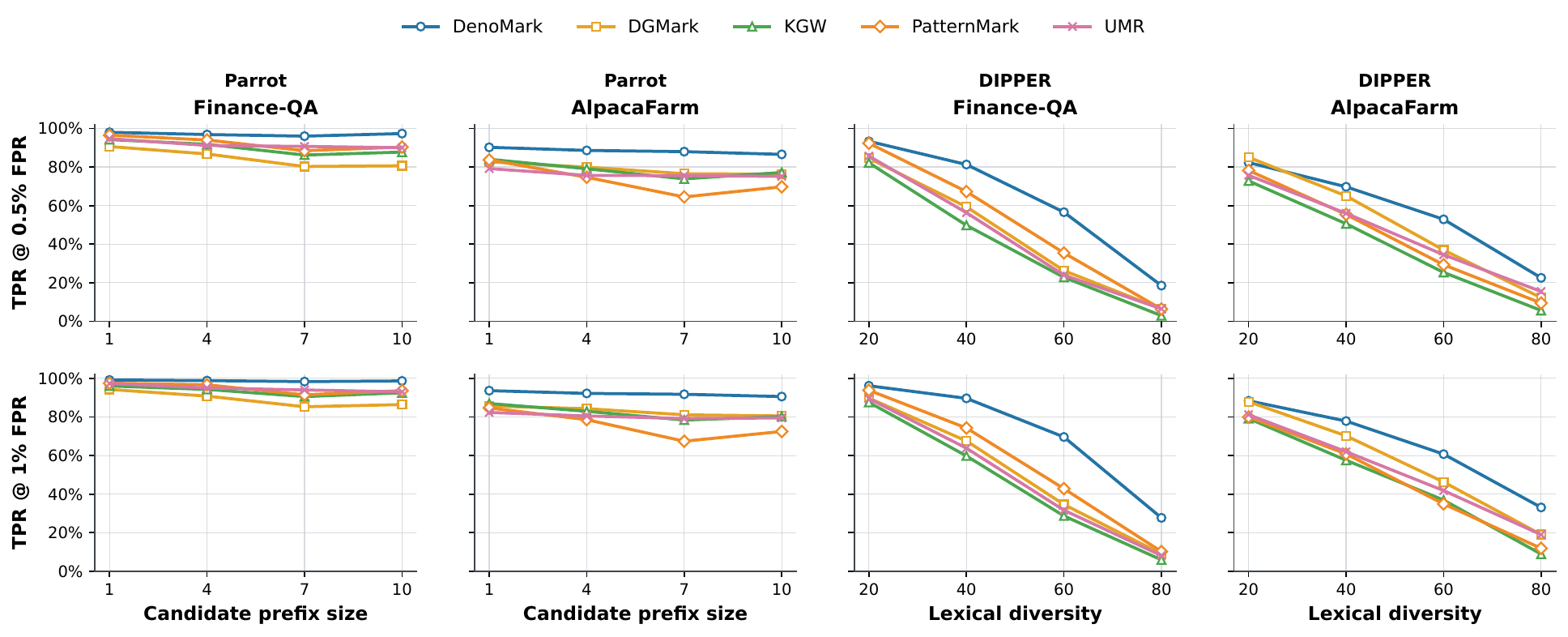}
    \caption{Robustness under Parrot and DIPPER on Finance-QA and AlpacaFarm. Each point is macro-averaged over LLaDA-8B and Dream-v0-Instruct-7B within the indicated dataset. Columns separate attacks and datasets; rows report TPR at $0.5\%$ and $1\%$ FPR. We vary the candidate prefix size over $\{1,4,7,10\}$ for Parrot and lexical diversity over $\{20,40,60,80\}$ for DIPPER.}
    \label{fig:opensource_semantic_attacks}
\end{figure}
\FloatBarrier

\subsection{Robustness to Semantic Attacks}
\label{sec:semantic_attack_results}

Following prior evaluations of semantic watermarks~\citep{hou2024semstamp,hou2024ksemstamp,ren2024robust,dabiriaghdam2025simmark,huo2026pmark}, we apply rewriting, compression, and expansion sentence by sentence and also apply paraphrasing to the full document. These transformations preserve the underlying meaning while changing wording, syntax, sentence structure, and, for compression and expansion, text length. All attacks use GPT-4o-mini. Appendix~\ref{app:experimental_details} provides the filtering criteria and prompts.

\paragraph{Results.}
Table~\ref{tab:c4_empirical_roc_main} reports results for LLaDA-8B, LLaDA2.0-mini, and Dream-v0-Instruct-7B; complete LLaDA1.5-8B results are provided in Appendix~\ref{app:llada15_semantic_attack_results}. Sentence-level compression results for all four models are reported with the document-level compression and expansion results in Appendix~\ref{app:document_compress_expand}. Averaged over the 48 attacked combinations across all four models, DenMark achieves 72.22\%, 80.17\%, and 90.66\% TPR at 0.5\%, 1\%, and 5\% FPR, respectively, with an AUC of 0.9758. The strongest baseline, DGMark, reaches 35.79\% at 0.5\% FPR, 44.01\% at 1\% FPR, 64.17\% at 5\% FPR, and 0.9035 AUC. DenMark has the best TPR at each operating point in 46 of the 48 attacked combinations and the best AUC in 44 combinations, with one additional AUC tie. Detection remains high after compression and rewriting. Expansion is the most challenging sentence-level transformation, while full-document paraphrasing still leaves a clear overall margin over the baselines.

\begin{table*}[!h]
\centering
\caption{
\textbf{Detection robustness on LLaDA-8B, LLaDA2.0-mini, and Dream-v0-Instruct-7B.}
Each entry reports TPR@0.5\%FPR / TPR@1\%FPR / TPR@5\%FPR / AUC. 
}
\label{tab:c4_empirical_roc_main}
\scriptsize
\renewcommand{\arraystretch}{0.697}
\setlength{\tabcolsep}{1.5pt}
\resizebox{\textwidth}{!}{%
\begin{tabular}{@{}ll|ccc@{}}
\toprule
Attack & Method & Finance & Alpaca & LongForm \\
\midrule
\multicolumn{5}{c}{\textbf{LLaDA-8B}} \\
\midrule
\multirow{5}{*}{Clean} & DenMark & (99.00 / 99.33 / 100.00 / 0.9990) & (95.60 / 96.34 / 98.90 / 0.9912) & (99.67 / 100.00 / 100.00 / 0.9991) \\
 & DLM-KGW & (99.40 / 99.40 / 100.00 / 0.9998) & (93.36 / 95.23 / 98.34 / 0.9957) & (99.80 / 99.80 / 99.80 / 0.9998) \\
 & DGMark & (100.00 / 100.00 / 100.00 / 1.0000) & (97.98 / 98.32 / 98.65 / 0.9897) & (100.00 / 100.00 / 100.00 / 1.0000) \\
 & PatternMark & (100.00 / 100.00 / 100.00 / 1.0000) & (98.56 / 98.56 / 100.00 / 0.9997) & (100.00 / 100.00 / 100.00 / 1.0000) \\
 & UMR & (100.00 / 100.00 / 100.00 / 1.0000) & (98.33 / 99.00 / 99.67 / 0.9995) & (99.67 / 100.00 / 100.00 / 1.0000) \\
\addlinespace[1pt]
\cmidrule(lr){1-5}
\multirow{5}{*}{Rewrite} & DenMark & (\textbf{72.24} / \textbf{81.61} / \textbf{92.64} / \textbf{0.9825}) & (\textbf{71.32} / \textbf{80.15} / \textbf{88.97} / \textbf{0.9718}) & (\textbf{76.67} / \textbf{86.00} / \textbf{94.67} / \textbf{0.9864}) \\
 & DLM-KGW & (10.80 / 18.20 / 38.00 / 0.8224) & (19.09 / 25.52 / 46.27 / 0.8354) & (9.20 / 13.80 / 29.40 / 0.7730) \\
 & DGMark & (34.00 / 42.33 / 72.33 / 0.9344) & (42.42 / 51.18 / 68.01 / 0.9072) & (36.67 / 47.33 / 69.33 / 0.9322) \\
 & PatternMark & (35.33 / 42.67 / 63.00 / 0.8917) & (35.74 / 40.43 / 60.65 / 0.8930) & (30.00 / 36.00 / 51.67 / 0.8542) \\
 & UMR & (34.67 / 40.00 / 59.33 / 0.8889) & (48.33 / 56.33 / 75.00 / 0.9254) & (23.67 / 29.33 / 53.00 / 0.8676) \\
\addlinespace[1pt]
\cmidrule(lr){1-5}
\multirow{5}{*}{Expansion} & DenMark & (\textbf{53.85} / \textbf{67.89} / \textbf{82.94} / \textbf{0.9652}) & (\textbf{61.54} / \textbf{71.06} / \textbf{83.88} / \textbf{0.9481}) & (\textbf{71.00} / \textbf{82.33} / \textbf{93.00} / \textbf{0.9865}) \\
 & DLM-KGW & (7.20 / 10.40 / 28.40 / 0.7658) & (12.24 / 18.46 / 35.89 / 0.7873) & (4.00 / 6.80 / 21.00 / 0.6944) \\
 & DGMark & (10.67 / 15.33 / 39.33 / 0.8344) & (23.57 / 30.30 / 51.52 / 0.8596) & (9.67 / 18.33 / 43.33 / 0.8571) \\
 & PatternMark & (32.00 / 39.33 / 60.33 / 0.8761) & (24.91 / 28.52 / 50.54 / 0.8434) & (15.67 / 19.00 / 41.00 / 0.8144) \\
 & UMR & (21.00 / 24.00 / 45.33 / 0.8461) & (36.33 / 42.33 / 63.00 / 0.8879) & (22.33 / 29.00 / 44.33 / 0.8223) \\
\addlinespace[1pt]
\cmidrule(lr){1-5}
\multirow{5}{*}{Document Paraphrase} & DenMark & (\textbf{62.54} / \textbf{73.58} / \textbf{85.95} / \textbf{0.9687}) & (\textbf{58.97} / \textbf{74.73} / \textbf{87.91} / \textbf{0.9665}) & (\textbf{68.67} / \textbf{79.33} / \textbf{90.00} / \textbf{0.9757}) \\
 & DLM-KGW & (11.80 / 15.60 / 36.00 / 0.8047) & (20.12 / 29.25 / 46.27 / 0.8216) & (7.40 / 11.80 / 27.20 / 0.7417) \\
 & DGMark & (27.00 / 37.33 / 61.00 / 0.8937) & (28.96 / 38.72 / 59.60 / 0.8883) & (23.67 / 36.00 / 61.00 / 0.9081) \\
 & PatternMark & (19.67 / 27.33 / 50.67 / 0.8507) & (29.60 / 35.02 / 54.15 / 0.8604) & (17.00 / 23.67 / 43.67 / 0.8038) \\
 & UMR & (28.00 / 33.67 / 57.00 / 0.8796) & (43.33 / 50.67 / 69.00 / 0.9082) & (24.00 / 31.00 / 50.00 / 0.8599) \\
\addlinespace[0pt]
\midrule
\multicolumn{5}{c}{\textbf{LLaDA2.0-mini}} \\
\midrule
\multirow{5}{*}{Clean} & DenMark & (100.00 / 100.00 / 100.00 / 0.9993) & (94.58 / 95.67 / 98.92 / 0.9967) & (100.00 / 100.00 / 100.00 / 0.9993) \\
 & DLM-KGW & (100.00 / 100.00 / 100.00 / 1.0000) & (97.13 / 97.49 / 98.57 / 0.9977) & (100.00 / 100.00 / 100.00 / 1.0000) \\
 & DGMark & (100.00 / 100.00 / 100.00 / 1.0000) & (99.26 / 99.63 / 99.63 / 0.9997) & (100.00 / 100.00 / 100.00 / 1.0000) \\
 & PatternMark & (100.00 / 100.00 / 100.00 / 1.0000) & (97.56 / 98.26 / 99.30 / 0.9975) & (100.00 / 100.00 / 100.00 / 1.0000) \\
 & UMR & (99.33 / 99.67 / 100.00 / 0.9999) & (90.33 / 93.00 / 95.67 / 0.9886) & (100.00 / 100.00 / 100.00 / 1.0000) \\
\addlinespace[1pt]
\cmidrule(lr){1-5}
\multirow{5}{*}{Rewrite} & DenMark & (\textbf{85.28} / \textbf{88.63} / \textbf{96.99} / \textbf{0.9941}) & (\textbf{75.45} / \textbf{82.31} / \textbf{92.06} / \textbf{0.9759}) & (\textbf{93.33} / \textbf{96.00} / \textbf{99.00} / \textbf{0.9979}) \\
 & DLM-KGW & (26.67 / 35.33 / 58.67 / 0.8949) & (41.22 / 47.31 / 64.16 / 0.9102) & (32.33 / 40.67 / 64.67 / 0.9261) \\
 & DGMark & (32.67 / 39.67 / 64.33 / 0.9283) & (46.86 / 55.72 / 77.86 / 0.9510) & (32.33 / 41.00 / 70.00 / 0.9393) \\
 & PatternMark & (35.33 / 42.33 / 65.00 / 0.9178) & (38.68 / 45.30 / 66.20 / 0.9059) & (39.33 / 47.33 / 73.00 / 0.9349) \\
 & UMR & (18.33 / 25.67 / 53.00 / 0.8647) & (32.00 / 42.33 / 61.33 / 0.8876) & (17.67 / 27.33 / 51.00 / 0.8602) \\
\addlinespace[1pt]
\cmidrule(lr){1-5}
\multirow{5}{*}{Expansion} & DenMark & (\textbf{64.21} / \textbf{72.24} / \textbf{86.96} / \textbf{0.9658}) & (\textbf{62.82} / \textbf{67.87} / \textbf{84.84} / \textbf{0.9601}) & (\textbf{83.67} / \textbf{88.33} / \textbf{94.67} / \textbf{0.9900}) \\
 & DLM-KGW & (6.33 / 10.00 / 26.00 / 0.7564) & (25.45 / 31.18 / 53.41 / 0.8651) & (6.67 / 12.00 / 33.67 / 0.8070) \\
 & DGMark & (5.67 / 9.33 / 27.67 / 0.8012) & (24.35 / 31.37 / 53.51 / 0.8715) & (7.33 / 13.33 / 37.33 / 0.8557) \\
 & PatternMark & (20.67 / 27.33 / 50.33 / 0.8485) & (25.09 / 29.62 / 50.87 / 0.8490) & (23.00 / 29.00 / 53.00 / 0.8849) \\
 & UMR & (9.33 / 15.00 / 33.00 / 0.7727) & (22.67 / 27.67 / 45.33 / 0.8070) & (7.33 / 13.33 / 30.67 / 0.7733) \\
\addlinespace[1pt]
\cmidrule(lr){1-5}
\multirow{5}{*}{Document Paraphrase} & DenMark & (\textbf{74.58} / \textbf{81.27} / \textbf{92.64} / \textbf{0.9822}) & (\textbf{66.43} / \textbf{73.29} / \textbf{87.00} / \textbf{0.9661}) & (\textbf{87.00} / \textbf{89.33} / \textbf{96.67} / \textbf{0.9920}) \\
 & DLM-KGW & (19.33 / 27.00 / 51.00 / 0.8607) & (35.48 / 40.86 / 64.52 / 0.8999) & (27.67 / 34.33 / 58.00 / 0.9105) \\
 & DGMark & (25.00 / 33.00 / 59.33 / 0.9099) & (39.11 / 46.13 / 64.94 / 0.9122) & (27.67 / 37.33 / 64.67 / 0.9276) \\
 & PatternMark & (27.33 / 33.33 / 58.67 / 0.8830) & (37.28 / 41.46 / 60.98 / 0.8852) & (33.67 / 39.33 / 66.00 / 0.9146) \\
 & UMR & (21.67 / 28.33 / 53.33 / 0.8738) & (31.00 / 38.00 / 59.00 / 0.8645) & (24.67 / 32.67 / 54.67 / 0.8874) \\
\addlinespace[0pt]
\midrule
\multicolumn{5}{c}{\textbf{Dream-v0-Instruct-7B}} \\
\midrule
\multirow{5}{*}{Clean} & DenMark & (100.00 / 100.00 / 100.00 / 0.9993) & (99.59 / 99.59 / 99.59 / 0.9990) & (100.00 / 100.00 / 100.00 / 0.9993) \\
 & DLM-KGW & (100.00 / 100.00 / 100.00 / 1.0000) & (98.56 / 99.64 / 100.00 / 0.9998) & (99.45 / 100.00 / 100.00 / 1.0000) \\
 & DGMark & (99.53 / 100.00 / 100.00 / 0.9998) & (100.00 / 100.00 / 100.00 / 1.0000) & (100.00 / 100.00 / 100.00 / 1.0000) \\
 & PatternMark & (100.00 / 100.00 / 100.00 / 1.0000) & (99.56 / 99.56 / 100.00 / 0.9999) & (100.00 / 100.00 / 100.00 / 1.0000) \\
 & UMR & (100.00 / 100.00 / 100.00 / 1.0000) & (99.00 / 99.00 / 100.00 / 0.9997) & (100.00 / 100.00 / 100.00 / 1.0000) \\
\addlinespace[1pt]
\cmidrule(lr){1-5}
\multirow{5}{*}{Rewrite} & DenMark & (\textbf{81.33} / \textbf{88.67} / \textbf{95.67} / \textbf{0.9888}) & (\textbf{80.00} / \textbf{86.12} / \textbf{94.69} / \textbf{0.9805}) & (\textbf{89.00} / \textbf{92.67} / \textbf{96.00} / \textbf{0.9919}) \\
 & DLM-KGW & (10.07 / 14.58 / 30.21 / 0.7450) & (15.47 / 21.94 / 41.37 / 0.8077) & (3.31 / 6.63 / 21.55 / 0.7142) \\
 & DGMark & (16.32 / 23.31 / 41.96 / 0.8114) & (48.36 / 58.22 / 77.30 / 0.9425) & (31.00 / 43.00 / 68.33 / 0.9265) \\
 & PatternMark & (42.62 / 48.66 / 71.81 / 0.9358) & (47.16 / 51.09 / 69.87 / 0.9081) & (39.93 / 46.42 / 66.21 / 0.9152) \\
 & UMR & (13.00 / 20.00 / 42.33 / 0.8245) & (33.67 / 37.33 / 60.33 / 0.8863) & (15.67 / 21.67 / 39.67 / 0.8251) \\
\addlinespace[1pt]
\cmidrule(lr){1-5}
\multirow{5}{*}{Expansion} & DenMark & (\textbf{70.67} / \textbf{76.67} / \textbf{91.67} / \textbf{0.9750}) & (\textbf{70.20} / \textbf{76.33} / \textbf{86.94} / \textbf{0.9678}) & (\textbf{84.67} / \textbf{89.00} / \textbf{96.33} / \textbf{0.9915}) \\
 & DLM-KGW & (2.43 / 4.17 / 12.85 / 0.6302) & (7.19 / 10.43 / 23.02 / 0.7378) & (1.10 / 1.66 / 7.73 / 0.6111) \\
 & DGMark & (4.90 / 10.49 / 28.67 / 0.7381) & (26.97 / 38.49 / 59.87 / 0.8931) & (9.67 / 17.00 / 40.00 / 0.8573) \\
 & PatternMark & (20.47 / 26.85 / 54.03 / 0.8670) & (28.38 / 37.55 / 57.64 / 0.8690) & (23.21 / 27.65 / 46.76 / 0.8668) \\
 & UMR & (9.00 / 14.00 / 30.67 / 0.7659) & (18.67 / 25.00 / 46.00 / 0.8165) & (8.33 / 12.33 / 31.33 / 0.7807) \\
\addlinespace[1pt]
\cmidrule(lr){1-5}
\multirow{5}{*}{Document Paraphrase} & DenMark & (\textbf{67.67} / \textbf{76.00} / \textbf{90.33} / \textbf{0.9757}) & (\textbf{68.98} / \textbf{77.14} / \textbf{88.57} / \textbf{0.9734}) & (\textbf{74.00} / \textbf{82.00} / \textbf{90.67} / \textbf{0.9793}) \\
 & DLM-KGW & (5.21 / 10.42 / 26.39 / 0.7001) & (10.79 / 15.47 / 37.77 / 0.7826) & (6.63 / 11.05 / 20.99 / 0.6795) \\
 & DGMark & (11.42 / 14.92 / 33.10 / 0.7677) & (40.79 / 48.36 / 68.09 / 0.9204) & (20.33 / 30.33 / 56.33 / 0.8852) \\
 & PatternMark & (23.49 / 33.89 / 56.71 / 0.8739) & (38.86 / 44.54 / 59.83 / 0.8953) & (23.89 / 27.99 / 47.78 / 0.8343) \\
 & UMR & (12.67 / 17.33 / 35.00 / 0.7892) & (22.33 / 27.33 / 50.00 / 0.8334) & (8.33 / 13.00 / 34.00 / 0.8017) \\
\bottomrule
\end{tabular}%
}
\end{table*}

\FloatBarrier

\subsection{Sensitivity Analysis}
\label{sec:sensitivity_analysis}
Figure~\ref{fig:direct_v1_ablation} measures sensitivity to candidate temperature, the number of candidates \(K_{t,b}\), the number of rollouts \(R_t\), and semantic unit size \(m\) on LLaDA-8B. We report TPR@1\%FPR using the empirical threshold rule in Section~\ref{sec:experimental_setup} and AUC for detection and \(\Delta\,\mathrm{mean}(\log\mathrm{PPL})\) for generation quality.

\begin{figure}[!htb]
  \centering
  \includegraphics[width=0.92\textwidth]{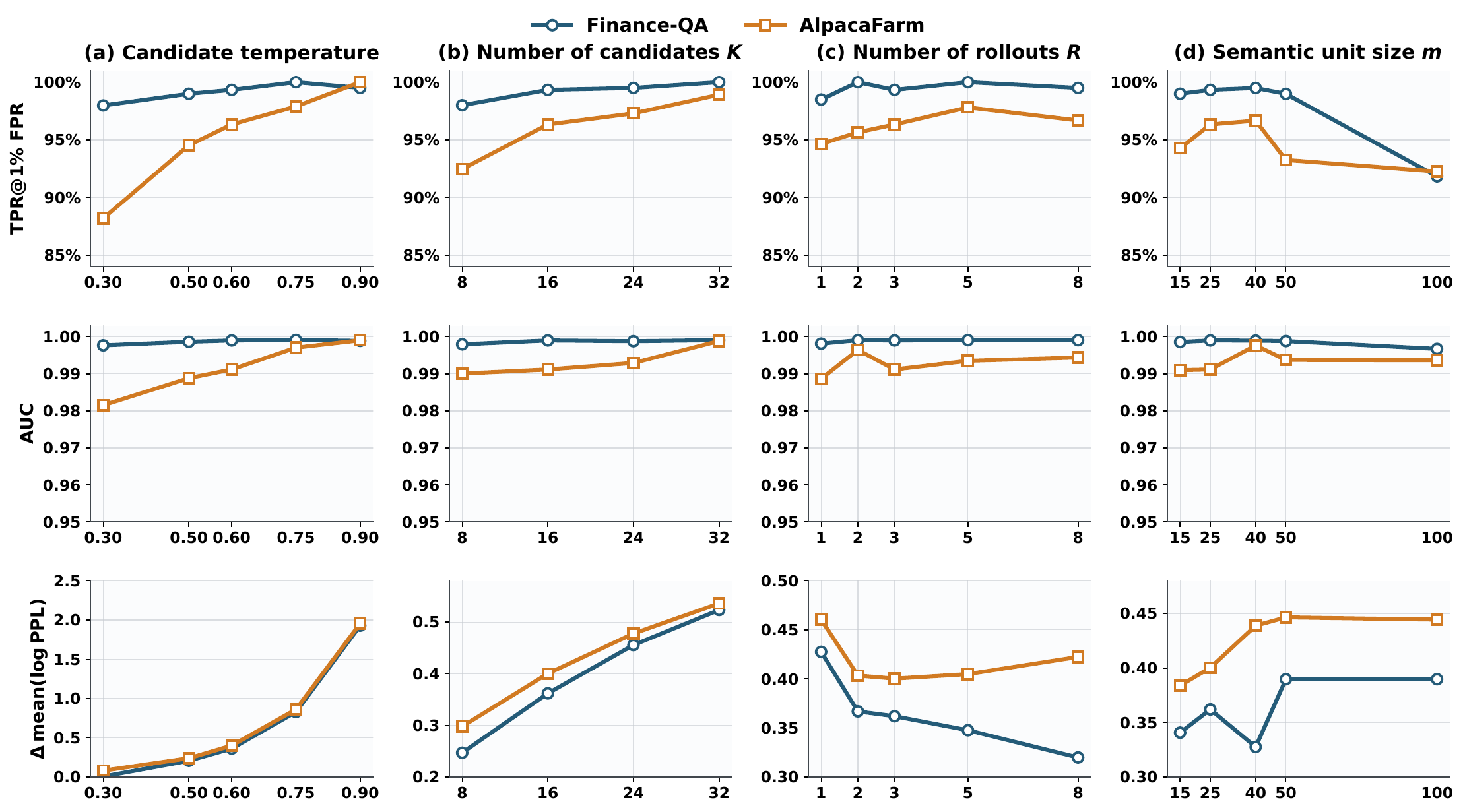}
  \caption{Sensitivity analysis of candidate temperature, number of candidates $K_{t,b}$, number of rollouts $R_t$, and semantic unit size $m$ on LLaDA-8B.}
  \label{fig:direct_v1_ablation}
\end{figure}
\paragraph{Results.}
Higher candidate temperatures improve watermark detection and increase the generation quality cost. Increasing \(K_{t,b}\) from 8 to 16 provides most of the detection gain, and larger candidate sets add smaller improvements. Multiple rollouts improve robustness relative to a single rollout, with most of the gain reached by \(R_t=2\)--\(3\). Moderate semantic unit sizes perform best across the evaluated settings, and larger units reduce detection on AlpacaFarm.
\FloatBarrier

\subsection{Generation Efficiency}
\label{sec:generation_efficiency}

Table~\ref{tab:generation_runtime} reports average generation time per visible output token. Measurements use one NVIDIA GH200 GPU (120~GB) per generation worker and exclude model loading. DenMark takes 0.642 seconds per token because it evaluates multiple candidates and rollouts. Such overhead is common in rewrite-robust semantic watermarking: SemStamp reports a $20.9\times$ slowdown and k-SemStamp samples 13.3 candidates per accepted sentence~\citep{hou2024semstamp,hou2024ksemstamp}, while PMark consumes 16.0 and 12.1 sampled tokens per output token in its online and offline variants~\citep{huo2026pmark}. DenMark therefore exhibits an acceptable speed--robustness trade-off for robustness-oriented generation.

\begingroup
\setlength{\intextsep}{4pt}
\begin{table}[!htb]
\centering
\footnotesize
\setlength{\abovecaptionskip}{1pt}
\setlength{\belowcaptionskip}{1pt}
\caption{\textbf{Generation time per output token.} Seconds per visible output token, averaged over Dream/LLaDA-8B and AlpacaFarm/LongFormQA.}
\label{tab:generation_runtime}
\normalsize
\renewcommand{\arraystretch}{1.08}
\setlength{\tabcolsep}{8pt}
\begin{tabular}{@{}cccccc@{}}
\toprule
Clean & DenMark & DLM-KGW & DGMark & PatternMark & UMR \\
\midrule
0.054 & 0.642 & 0.118 & 0.236 & 0.089 & 0.070 \\
\bottomrule
\end{tabular}
\vspace{-4pt}
\end{table}
\endgroup

\FloatBarrier

\section{Conclusion}
We introduced DenMark, a semantic watermarking framework that injects key-dependent evidence into DLMs during iterative denoising. Temporary rollouts make partially observed units scoreable, allowing each decision to commit the local update with the largest estimated watermark score. Calibrated scanning recovers evidence after rewriting shifts token boundaries, while our analysis connects local selection advantages to cumulative score improvement. Across four DLMs, three datasets, and diverse semantic and token-level attacks, DenMark provides substantially stronger post-attack detection overall than the evaluated baselines. These results show that semantic evidence can accumulate directly through non-autoregressive denoising. The added generation time from candidate evaluation and rollouts motivates future work on more efficient generation and detection under broader structural edits.

\clearpage
\subsection*{Ethics Statement}
DenMark is intended to support content provenance and the responsible identification of model-generated text. Like any statistical watermark detector, it may produce false positives or fail under distribution shifts and adversarial modifications; therefore, a detection result should not be treated as conclusive evidence of provenance or misconduct, particularly in high-stakes or punitive settings. Watermarking technology may also be misused for covert tracking or surveillance, and practical deployment should include transparency, appropriate access controls, and human review. Our study involves no human subjects and uses established public datasets and model-generated outputs.

\subsection*{AI Use Statement}
The initial research idea was developed through question-and-answer discussions with generative AI tools. These tools also assisted with literature retrieval and discovery, organizing related work and baseline methods, designing and implementing the semantic attacks used in the evaluation, other aspects of research execution, language polishing, and preparing preliminary drafts of parts of the manuscript. The authors manually checked, revised, and approved all AI-assisted ideas, references, technical descriptions, code, attack configurations, experimental procedures, analyses, and drafted text, and take full responsibility for the manuscript and its conclusions.

\bibliography{iclr2027_conference}

@article{phipson2010permutation,
  title = {Permutation P-values Should Never Be Zero: Calculating Exact P-values When Permutations Are Randomly Drawn},
  author = {Phipson, Belinda and Smyth, Gordon K.},
  journal = {Statistical Applications in Genetics and Molecular Biology},
  volume = {9},
  number = {1},
  pages = {Article 39},
  year = {2010},
  doi = {10.2202/1544-6115.1585}
}

@inproceedings{brown2020language,
  title = {Language Models are Few-Shot Learners},
  author = {Brown, Tom B. and others},
  booktitle = {Advances in Neural Information Processing Systems},
  volume = {33},
  pages = {1877--1901},
  year = {2020},
  url = {https://proceedings.neurips.cc/paper/2020/hash/1457c0d6bfcb4967418bfb8ac142f64a-Abstract.html}
}

@article{wang2026badit,
  title = {Decomposing the Basic Abilities of Large Language Models: Mitigating Cross-Task Interference in Multi-Task Instruct-Tuning},
  author = {Wang, Bing and Li, Ximing and Li, Changchun and Chi, Jinjin and Niu, Gang and Sugiyama, Masashi},
  journal = {arXiv preprint arXiv:2605.05676},
  year = {2026},
  eprint = {2605.05676},
  archiveprefix = {arXiv},
  url = {https://arxiv.org/abs/2605.05676}
}

@article{lin2025creative4u,
  title = {{Creative4U}: {MLLMs}-based Advertising Creative Image Selector with Comparative Reasoning},
  author = {Lin, Yukang and Zhang, Xiang and Jia, Shichang and Wan, Bowen and Fu, Chenghan and Ren, Xudong and Liu, Yueran and Guan, Wanxian and Wang, Pengji and Xu, Jian and Zheng, Bo and Liu, Baolin},
  journal = {arXiv preprint arXiv:2508.12628},
  year = {2025},
  eprint = {2508.12628},
  archiveprefix = {arXiv},
  url = {https://arxiv.org/abs/2508.12628}
}

@article{chen2021evaluating,
  title = {Evaluating Large Language Models Trained on Code},
  author = {Chen, Mark and others},
  journal = {arXiv preprint arXiv:2107.03374},
  year = {2021},
  eprint = {2107.03374},
  archiveprefix = {arXiv},
  url = {https://arxiv.org/abs/2107.03374}
}

@article{kasneci2023chatgpt,
  title = {{ChatGPT} for Good? {On} Opportunities and Challenges of Large Language Models for Education},
  author = {Kasneci, Enkelejda and others},
  journal = {Learning and Individual Differences},
  volume = {103},
  pages = {102274},
  year = {2023},
  doi = {10.1016/j.lindif.2023.102274}
}

@article{singhal2023large,
  title = {Large Language Models Encode Clinical Knowledge},
  author = {Singhal, Karan and others},
  journal = {Nature},
  volume = {620},
  pages = {172--180},
  year = {2023},
  doi = {10.1038/s41586-023-06291-2}
}

@inproceedings{kirchenbauer2023watermark,
  title = {A Watermark for Large Language Models},
  author = {Kirchenbauer, John and Geiping, Jonas and Wen, Yuxin and Katz, Jonathan and Miers, Ian and Goldstein, Tom},
  booktitle = {International Conference on Machine Learning},
  volume = {202},
  pages = {17061--17084},
  series = {Proceedings of Machine Learning Research},
  year = {2023},
  publisher = {PMLR},
  url = {https://proceedings.mlr.press/v202/kirchenbauer23a.html}
}

@article{dathathri2024scalable,
  title = {Scalable Watermarking for Identifying Large Language Model Outputs},
  author = {Dathathri, Sumanth and others},
  journal = {Nature},
  volume = {634},
  number = {8035},
  pages = {818--823},
  year = {2024},
  doi = {10.1038/s41586-024-08025-4}
}

@inproceedings{sahoo2024simple,
  title = {Simple and Effective Masked Diffusion Language Models},
  author = {Sahoo, Subham Sekhar and others},
  booktitle = {Advances in Neural Information Processing Systems},
  volume = {37},
  year = {2024},
  doi = {10.52202/079017-4135},
  url = {https://proceedings.neurips.cc/paper_files/paper/2024/hash/eb0b13cc515724ab8015bc978fdde0ad-Abstract-Conference.html}
}

@inproceedings{nie2025large,
  title = {Large Language Diffusion Models},
  author = {Nie, Shen and Zhu, Fengqi and You, Zebin and Zhang, Xiaolu and Ou, Jingyang and Hu, Jun and Zhou, Jun and Lin, Yankai and Wen, Ji-Rong and Li, Chongxuan},
  booktitle = {Advances in Neural Information Processing Systems},
  volume = {38},
  year = {2025},
  doi = {10.52202/085713-1689},
  url = {https://proceedings.neurips.cc/paper_files/paper/2025/hash/48b383b24230e0e6e649d9c98dae4d8c-Abstract-Conference.html}
}

@inproceedings{gloaguen2025watermarking,
  title = {Watermarking Diffusion Language Models},
  author = {Gloaguen, Thibaud and Staab, Robin and Jovanovi{\'c}, Nikola and Vechev, Martin},
  booktitle = {International Conference on Learning Representations},
  year = {2026},
  url = {https://openreview.net/forum?id=3aBWTYGcaT}
}

@article{dunn1961multiple,
  title = {Multiple Comparisons among Means},
  author = {Dunn, Olive Jean},
  journal = {Journal of the American Statistical Association},
  volume = {56},
  number = {293},
  pages = {52--64},
  year = {1961},
  doi = {10.1080/01621459.1961.10482090}
}

@inproceedings{hou2024semstamp,
  title = {{SemStamp}: A Semantic Watermark with Paraphrastic Robustness for Text Generation},
  author = {Hou, Abe Bohan and Zhang, Jingyu and He, Tianxing and Wang, Yichen and Chuang, Yung-Sung and Wang, Hongwei and Shen, Lingfeng and Van Durme, Benjamin and Khashabi, Daniel and Tsvetkov, Yulia},
  booktitle = {Proceedings of the 2024 Conference of the North American Chapter of the Association for Computational Linguistics: Human Language Technologies (Volume 1: Long Papers)},
  month = jun,
  pages = {4067--4082},
  year = {2024},
  address = {Mexico City, Mexico},
  publisher = {Association for Computational Linguistics},
  doi = {10.18653/v1/2024.naacl-long.226},
  url = {https://aclanthology.org/2024.naacl-long.226/}
}

@inproceedings{ren2024robust,
  title = {A Robust Semantics-based Watermark for Large Language Model against Paraphrasing},
  author = {Ren, Jie and Xu, Han and Liu, Yiding and Cui, Yingqian and Wang, Shuaiqiang and Yin, Dawei and Tang, Jiliang},
  booktitle = {Findings of the Association for Computational Linguistics: NAACL 2024},
  pages = {613--625},
  year = {2024},
  publisher = {Association for Computational Linguistics},
  doi = {10.18653/v1/2024.findings-naacl.40},
  url = {https://aclanthology.org/2024.findings-naacl.40/}
}

@article{shaffer1995multiple,
  title = {Multiple Hypothesis Testing},
  author = {Shaffer, Juliet Popper},
  journal = {Annual Review of Psychology},
  volume = {46},
  number = {1},
  pages = {561--584},
  year = {1995},
  publisher = {Annual Reviews},
  doi = {10.1146/annurev.ps.46.020195.003021},
  url = {https://www.annualreviews.org/content/journals/10.1146/annurev.ps.46.020195.003021}
}

@inproceedings{huo2026pmark,
  title = {{PMark}: Towards Robust and Distortion-Free Semantic-Level Watermarking with Channel Constraints},
  author = {Huo, Jiahao and Liu, Shuliang and Wang, Bin and Zhang, Junyan and Yan, Yibo and Liu, Aiwei and Hu, Xuming and Zhou, Mingxun},
  booktitle = {International Conference on Learning Representations},
  year = {2026},
  url = {https://openreview.net/forum?id=EhDgP69DJG}
}

@inproceedings{hendrycks2021measuring,
  title = {Measuring Massive Multitask Language Understanding},
  author = {Hendrycks, Dan and Burns, Collin and Basart, Steven and Zou, Andy and Mazeika, Mantas and Song, Dawn and Steinhardt, Jacob},
  booktitle = {International Conference on Learning Representations},
  year = {2021},
  url = {https://openreview.net/forum?id=d7KBjmI3GmQ}
}

@inproceedings{zellers2019hellaswag,
  title = {{HellaSwag}: Can a Machine Really Finish Your Sentence?},
  author = {Zellers, Rowan and Holtzman, Ari and Bisk, Yonatan and Farhadi, Ali and Choi, Yejin},
  booktitle = {Proceedings of the 57th Annual Meeting of the Association for Computational Linguistics},
  month = jul,
  pages = {4791--4800},
  year = {2019},
  address = {Florence, Italy},
  publisher = {Association for Computational Linguistics},
  doi = {10.18653/v1/P19-1472},
  url = {https://aclanthology.org/P19-1472/}
}

@article{clark2018think,
  title = {Think You Have Solved Question Answering? Try {ARC}, the {AI2} Reasoning Challenge},
  author = {Clark, Peter and Cowhey, Isaac and Etzioni, Oren and Khot, Tushar and Sabharwal, Ashish and Schoenick, Carissa and Tafjord, {\O}yvind},
  journal = {arXiv preprint arXiv:1803.05457},
  year = {2018},
  eprint = {1803.05457},
  archiveprefix = {arXiv},
  url = {https://arxiv.org/abs/1803.05457}
}

@article{cobbe2021training,
  title = {Training Verifiers to Solve Math Word Problems},
  author = {Cobbe, Karl and Kosaraju, Vineet and Bavarian, Mohammad and Chen, Mark and Jun, Heewoo and Kaiser, Lukasz and Plappert, Matthias and Tworek, Jerry and Hilton, Jacob and Nakano, Reiichiro and Hesse, Christopher and Schulman, John},
  journal = {arXiv preprint arXiv:2110.14168},
  year = {2021},
  eprint = {2110.14168},
  archiveprefix = {arXiv},
  url = {https://arxiv.org/abs/2110.14168}
}

@article{raffel2020exploring,
  title = {Exploring the Limits of Transfer Learning with a Unified Text-to-Text Transformer},
  author = {Raffel, Colin and Shazeer, Noam and Roberts, Adam and Lee, Katherine and Narang, Sharan and Matena, Michael and Zhou, Yanqi and Li, Wei and Liu, Peter J.},
  journal = {Journal of Machine Learning Research},
  volume = {21},
  number = {140},
  pages = {1--67},
  year = {2020},
  url = {https://www.jmlr.org/papers/v21/20-074.html}
}

@article{wang2022text,
  title = {Text Embeddings by Weakly-Supervised Contrastive Pre-training},
  author = {Wang, Liang and Yang, Nan and Huang, Xiaolong and Jiao, Binxing and Yang, Linjun and Jiang, Daxin and Majumder, Rangan and Wei, Furu},
  journal = {arXiv preprint arXiv:2212.03533},
  year = {2022},
  eprint = {2212.03533},
  archiveprefix = {arXiv},
  url = {https://arxiv.org/abs/2212.03533}
}

@inproceedings{zhang2020pegasus,
  title = {{PEGASUS}: Pre-training with Extracted Gap-Sentences for Abstractive Summarization},
  author = {Zhang, Jingqing and Zhao, Yao and Saleh, Mohammad and Liu, Peter J.},
  booktitle = {Proceedings of the 37th International Conference on Machine Learning},
  volume = {119},
  pages = {11328--11339},
  series = {Proceedings of Machine Learning Research},
  year = {2020},
  publisher = {PMLR},
  url = {https://proceedings.mlr.press/v119/zhang20ae.html}
}

@misc{prithivida2021parrot,
  title = {Parrot: Paraphrase Generation for {NLU}},
  author = {Damodaran, Prithiviraj},
  year = {2021},
  howpublished = {\url{https://github.com/PrithivirajDamodaran/Parrot_Paraphraser}},
  note = {Software, version 1.0}
}

@inproceedings{krishna2023paraphrasing,
  title = {Paraphrasing Evades Detectors of {AI}-Generated Text, but Retrieval Is an Effective Defense},
  author = {Krishna, Kalpesh and Song, Yixiao and Karpinska, Marzena and Wieting, John and Iyyer, Mohit},
  booktitle = {Advances in Neural Information Processing Systems},
  volume = {36},
  pages = {27469--27500},
  year = {2023},
  doi = {10.52202/075280-1195},
  url = {https://proceedings.neurips.cc/paper_files/paper/2023/hash/575c450013d0e99e4b0ecf82bd1afaa4-Abstract-Conference.html}
}

@inproceedings{lou2024discrete,
  title = {Discrete Diffusion Modeling by Estimating the Ratios of the Data Distribution},
  author = {Lou, Aaron and Meng, Chenlin and Ermon, Stefano},
  booktitle = {Proceedings of the 41st International Conference on Machine Learning},
  volume = {235},
  pages = {32819--32848},
  series = {Proceedings of Machine Learning Research},
  year = {2024},
  publisher = {PMLR},
  url = {https://proceedings.mlr.press/v235/lou24a.html}
}

@inproceedings{bender2021dangers,
  title = {On the Dangers of Stochastic Parrots: Can Language Models Be Too Big?},
  author = {Bender, Emily M. and Gebru, Timnit and McMillan-Major, Angelina and Shmitchell, Shmargaret},
  booktitle = {Proceedings of the 2021 ACM Conference on Fairness, Accountability, and Transparency},
  pages = {610--623},
  year = {2021},
  doi = {10.1145/3442188.3445922}
}

@article{bommasani2021opportunities,
  title = {On the Opportunities and Risks of Foundation Models},
  author = {Bommasani, Rishi and others},
  journal = {arXiv preprint arXiv:2108.07258},
  year = {2021},
  eprint = {2108.07258},
  archiveprefix = {arXiv},
  url = {https://arxiv.org/abs/2108.07258}
}

@article{weidinger2021ethical,
  title = {Ethical and Social Risks of Harm from Language Models},
  author = {Weidinger, Laura and others},
  journal = {arXiv preprint arXiv:2112.04359},
  year = {2021},
  eprint = {2112.04359},
  archiveprefix = {arXiv},
  url = {https://arxiv.org/abs/2112.04359}
}

@article{ye2025dream,
  title = {{Dream 7B}: Diffusion Large Language Models},
  author = {Ye, Jiacheng and Xie, Zhihui and Zheng, Lin and Gao, Jiahui and Wu, Zirui and Jiang, Xin and Li, Zhenguo and Kong, Lingpeng},
  journal = {arXiv preprint arXiv:2508.15487},
  year = {2025},
  eprint = {2508.15487},
  archiveprefix = {arXiv},
  url = {https://arxiv.org/abs/2508.15487}
}

@article{kuditipudi2023robust,
  title = {Robust Distortion-free Watermarks for Language Models},
  author = {Kuditipudi, Rohith and Thickstun, John and Hashimoto, Tatsunori and Liang, Percy},
  journal = {Transactions on Machine Learning Research},
  year = {2024},
  issn = {2835-8856},
  url = {https://openreview.net/forum?id=FpaCL1M02C}
}

@inproceedings{hu2023unbiased,
  title = {Unbiased Watermark for Large Language Models},
  author = {Hu, Zhengmian and Chen, Lichang and Wu, Xidong and Wu, Yihan and Zhang, Hongyang and Huang, Heng},
  booktitle = {International Conference on Learning Representations},
  year = {2024},
  url = {https://openreview.net/forum?id=uWVC5FVidc}
}

@inproceedings{hou2024ksemstamp,
  title = {{k-SemStamp}: A Clustering-Based Semantic Watermark for Detection of Machine-Generated Text},
  author = {Hou, Abe Bohan and Zhang, Jingyu and Wang, Yichen and Khashabi, Daniel and He, Tianxing},
  booktitle = {Findings of the Association for Computational Linguistics: ACL 2024},
  month = aug,
  pages = {1706--1715},
  year = {2024},
  address = {Bangkok, Thailand},
  publisher = {Association for Computational Linguistics},
  doi = {10.18653/v1/2024.findings-acl.98},
  url = {https://aclanthology.org/2024.findings-acl.98/}
}

@article{chowdhery2023palm,
  title = {{PaLM}: Scaling Language Modeling with Pathways},
  author = {Chowdhery, Aakanksha and others},
  journal = {Journal of Machine Learning Research},
  volume = {24},
  number = {240},
  pages = {1--113},
  year = {2023},
  url = {https://www.jmlr.org/papers/v24/22-1144.html}
}

@article{touvron2023llama,
  title = {{LLaMA}: Open and Efficient Foundation Language Models},
  author = {Touvron, Hugo and Lavril, Thibaut and Izacard, Gautier and Martinet, Xavier and Lachaux, Marie-Anne and Lacroix, Timothée and Rozière, Baptiste and Goyal, Naman and Hambro, Eric and Azhar, Faisal and Rodriguez, Aurelien and Joulin, Armand and Grave, Edouard and Lample, Guillaume},
  journal = {arXiv preprint arXiv:2302.13971},
  year = {2023},
  eprint = {2302.13971},
  archiveprefix = {arXiv},
  url = {https://arxiv.org/abs/2302.13971}
}

@inproceedings{dabiriaghdam2025simmark,
  title = {{SimMark}: A Robust Sentence-Level Similarity-Based Watermarking Algorithm for Large Language Models},
  author = {Dabiriaghdam, Amirhossein and Wang, Lele},
  booktitle = {Proceedings of the 2025 Conference on Empirical Methods in Natural Language Processing},
  pages = {30785--30806},
  year = {2025},
  publisher = {Association for Computational Linguistics},
  doi = {10.18653/v1/2025.emnlp-main.1567},
  url = {https://aclanthology.org/2025.emnlp-main.1567/}
}

@inproceedings{gong2024scaling,
  title = {Scaling Diffusion Language Models via Adaptation from Autoregressive Models},
  author = {Gong, Shansan and Agarwal, Shivam and Zhang, Yizhe and Ye, Jiacheng and Zheng, Lin and Li, Mukai and An, Chenxin and Zhao, Peilin and Bi, Wei and Han, Jiawei and Peng, Hao and Kong, Lingpeng},
  booktitle = {International Conference on Learning Representations},
  year = {2025},
  url = {https://openreview.net/forum?id=j1tSLYKwg8}
}

@article{zhu2025llada15,
  title = {{LLaDA} 1.5: Variance-Reduced Preference Optimization for Large Language Diffusion Models},
  author = {Zhu, Fengqi and Wang, Rongzhen and Nie, Shen and Zhang, Xiaolu and Wu, Chunwei and Hu, Jun and Zhou, Jun and Chen, Jianfei and Lin, Yankai and Wen, Ji-Rong and Li, Chongxuan},
  journal = {arXiv preprint arXiv:2505.19223},
  year = {2025},
  eprint = {2505.19223},
  archiveprefix = {arXiv},
  url = {https://arxiv.org/abs/2505.19223}
}

@article{bie2025llada20,
  title = {{LLaDA}2.0: Scaling Up Diffusion Language Models to {100B}},
  author = {Bie, Tiwei and Cao, Maosong and Chen, Kun and Du, Lun and Gong, Mingliang and Gong, Zhuochen and Gu, Yanmei and Hu, Jiaqi and Huang, Zenan and Lan, Zhenzhong and Li, Chengxi and Li, Chongxuan and Li, Jianguo and Li, Zehuan and Liu, Huabin and Liu, Lin and Lu, Guoshan and Lu, Xiaocheng and Ma, Yuxin and Tan, Jianfeng and Wei, Lanning and Wen, Ji-Rong and Xing, Yipeng and Zhang, Xiaolu and Zhao, Junbo and Zheng, Da and Zhou, Jun and Zhou, Junlin and Zhou, Zhanchao and Zhu, Liwang and Zhuang, Yihong},
  journal = {arXiv preprint arXiv:2512.15745},
  year = {2025},
  eprint = {2512.15745},
  archiveprefix = {arXiv},
  url = {https://arxiv.org/abs/2512.15745}
}

@inproceedings{tu2024waterbench,
  title = {{WaterBench}: Towards Holistic Evaluation of Watermarks for Large Language Models},
  author = {Tu, Shangqing and Sun, Yuliang and Bai, Yushi and Yu, Jifan and Hou, Lei and Li, Juanzi},
  booktitle = {Proceedings of the 62nd Annual Meeting of the Association for Computational Linguistics (Volume 1: Long Papers)},
  pages = {1517--1542},
  year = {2024},
  publisher = {Association for Computational Linguistics},
  doi = {10.18653/v1/2024.acl-long.83},
  url = {https://aclanthology.org/2024.acl-long.83/}
}

@inproceedings{maia2018fiqa,
  title = {{WWW'18 Open Challenge}: Financial Opinion Mining and Question Answering},
  author = {Maia, Macedo and Handschuh, Siegfried and Freitas, Andr{\'e} and Davis, Brian and McDermott, Ross and Zarrouk, Manel and Balahur, Alexandra},
  booktitle = {Companion Proceedings of the Web Conference 2018},
  pages = {1941--1942},
  year = {2018},
  publisher = {International World Wide Web Conferences Steering Committee},
  address = {Lyon, France},
  doi = {10.1145/3184558.3192301}
}

@inproceedings{dubois2023alpacafarm,
  title = {{AlpacaFarm}: A Simulation Framework for Methods that Learn from Human Feedback},
  author = {Dubois, Yann and Li, Chen Xuechen and Taori, Rohan and Zhang, Tianyi and Gulrajani, Ishaan and Ba, Jimmy and Guestrin, Carlos and Liang, Percy and Hashimoto, Tatsunori B.},
  booktitle = {Advances in Neural Information Processing Systems},
  volume = {36},
  pages = {30039--30069},
  year = {2023},
  doi = {10.52202/075280-1308}
}

@inproceedings{fan2019eli5,
  title = {{ELI5}: Long Form Question Answering},
  author = {Fan, Angela and Jernite, Yacine and Perez, Ethan and Grangier, David and Weston, Jason and Auli, Michael},
  booktitle = {Proceedings of the 57th Annual Meeting of the Association for Computational Linguistics},
  pages = {3558--3567},
  year = {2019},
  publisher = {Association for Computational Linguistics},
  address = {Florence, Italy},
  doi = {10.18653/v1/P19-1346},
  url = {https://aclanthology.org/P19-1346/}
}

@inproceedings{hong2026dgmark,
  title = {{dgMARK}: Decoding-Guided Watermarking for Diffusion Language Models},
  author = {Hong, Pyo Min and No, Albert},
  booktitle = {International Conference on Machine Learning},
  year = {2026},
  eprint = {2601.22985},
  url = {https://arxiv.org/abs/2601.22985},
  archiveprefix = {arXiv},
  primaryclass = {cs.LG}
}

@inproceedings{chen2025patternmark,
  title = {A Watermark for Order-Agnostic Language Models},
  author = {Chen, Ruibo and Wu, Yihan and Chen, Yanshuo and Liu, Chenxi and Guo, Junfeng and Huang, Heng},
  booktitle = {International Conference on Learning Representations},
  year = {2025},
  url = {https://openreview.net/forum?id=Nlm3Xf0W9S}
}

@inproceedings{yang2026umr,
  title = {You Can Have a Second Chance: Unbiased and Multi-bit Watermarking for Diffusion Language Models with Regret-based Remasking},
  author = {Yang, Ke and Liang, Dongyang and Yu, Jing and Yuan, Shuguang and Chen, Chi},
  booktitle = {Proceedings of the 64th Annual Meeting of the Association for Computational Linguistics (Volume 1: Long Papers)},
  pages = {28141--28160},
  year = {2026},
  publisher = {Association for Computational Linguistics},
  address = {San Diego, California, United States},
  doi = {10.18653/v1/2026.acl-long.1297},
  url = {https://aclanthology.org/2026.acl-long.1297/}
}

@inproceedings{devlin2019bert,
  title = {{BERT}: Pre-training of Deep Bidirectional Transformers for Language Understanding},
  author = {Devlin, Jacob and Chang, Ming-Wei and Lee, Kenton and Toutanova, Kristina},
  booktitle = {Proceedings of the 2019 Conference of the North American Chapter of the Association for Computational Linguistics: Human Language Technologies, Volume 1 (Long and Short Papers)},
  pages = {4171--4186},
  year = {2019},
  publisher = {Association for Computational Linguistics},
  doi = {10.18653/v1/N19-1423},
  url = {https://aclanthology.org/N19-1423/}
}

@inproceedings{raban-etal-2026-lr,
  title = {{LR}-{DWM}: Efficient Watermarking for Diffusion Language Models},
  author = {Raban, Ofek and Chechik, Gal and Fetaya, Ethan},
  booktitle = {Findings of the Association for Computational Linguistics: ACL 2026},
  month = jul,
  year = {2026},
  address = {San Diego, California, United States},
  publisher = {Association for Computational Linguistics},
  pages = {43510--43517},
  doi = {10.18653/v1/2026.findings-acl.2161},
  url = {https://aclanthology.org/2026.findings-acl.2161/}
}

@inproceedings{he2025empirical,
  title = {On the Empirical Power of Goodness-of-Fit Tests in Watermark Detection},
  author = {He, Weiqing and Li, Xiang and Shang, Tianqi and Shen, Li and Su, Weijie and Long, Qi},
  booktitle = {Advances in Neural Information Processing Systems},
  volume = {38},
  year = {2025},
  doi = {10.52202/085713-0632},
  url = {https://proceedings.neurips.cc/paper_files/paper/2025/hash/1b19bf73c2f341a350f9bd05c204d97b-Abstract-Conference.html}
}

@inproceedings{zhao2025permute,
  title = {Permute-and-Flip: An Optimally Stable and Watermarkable Decoder for {LLM}s},
  author = {Zhao, Xuandong and Li, Lei and Wang, Yu-Xiang},
  booktitle = {International Conference on Learning Representations},
  year = {2025},
  url = {https://proceedings.iclr.cc/paper_files/paper/2025/hash/50be7e77b9c883144940be925b608acc-Abstract-Conference.html}
}

@article{gloaguen2026every,
  title = {Every Bit, Everywhere, All at Once: A Binomial Multibit {LLM} Watermark},
  author = {Gloaguen, Thibaud and Staab, Robin and Vero, Mark and Vechev, Martin},
  journal = {arXiv preprint arXiv:2605.11653},
  year = {2026},
  eprint = {2605.11653},
  archiveprefix = {arXiv},
  url = {https://arxiv.org/abs/2605.11653}
}

@inproceedings{an2025defending,
  title = {Defending {LLM} Watermarking Against Spoofing Attacks with Contrastive Representation Learning},
  author = {An, Li and Liu, Yujian and Liu, Yepeng and Zhang, Yang and Bu, Yuheng and Chang, Shiyu},
  booktitle = {Conference on Language Modeling},
  year = {2025},
  url = {https://openreview.net/forum?id=n5hmtkdl7k}
}

@book{lehmann2005testing,
  title = {Testing Statistical Hypotheses},
  author = {Lehmann, E. L. and Romano, Joseph P.},
  edition = {3},
  year = {2005},
  publisher = {Springer},
  address = {New York},
  doi = {10.1007/0-387-27605-X},
  url = {https://doi.org/10.1007/0-387-27605-X}
}

@inproceedings{zhang2024watermarks,
  title = {Watermarks in the Sand: Impossibility of Strong Watermarking for Language Models},
  author = {Zhang, Hanlin and Edelman, Benjamin L. and Francati, Danilo and Venturi, Daniele and Ateniese, Giuseppe and Barak, Boaz},
  booktitle = {Proceedings of the 41st International Conference on Machine Learning},
  pages = {58851--58880},
  year = {2024},
  volume = {235},
  series = {Proceedings of Machine Learning Research},
  publisher = {PMLR},
  url = {https://proceedings.mlr.press/v235/zhang24o.html}
}

@inproceedings{harelcanada2025sandcastles,
  title = {Sandcastles in the Storm: Revisiting the {(Im)possibility} of Strong Watermarking},
  author = {Harel-Canada, Fabrice Y. and Erol, Boran and Choi, Connor and Liu, Jason and Song, Gary Jiarui and Peng, Nanyun and Sahai, Amit},
  booktitle = {Proceedings of the 63rd Annual Meeting of the Association for Computational Linguistics (Volume 1: Long Papers)},
  pages = {29698--29735},
  year = {2025},
  publisher = {Association for Computational Linguistics},
  doi = {10.18653/v1/2025.acl-long.1436},
  url = {https://aclanthology.org/2025.acl-long.1436/}
}

@inproceedings{huang-wan-2026-tdma,
  title = {From {TDMA} to {CDMA}: A Multi-bit Watermark for Diffusion Language Models},
  author = {Huang, Baizhou and Wan, Xiaojun},
  booktitle = {Proceedings of the 64th Annual Meeting of the Association for Computational Linguistics (Volume 1: Long Papers)},
  month = jul,
  year = {2026},
  address = {San Diego, California, United States},
  publisher = {Association for Computational Linguistics},
  pages = {37440--37454},
  doi = {10.18653/v1/2026.acl-long.1737},
  url = {https://aclanthology.org/2026.acl-long.1737/}
}

@inproceedings{cheng-etal-2026-sdar,
  title = {{SDAR}: A Synergistic Diffusion-{A}uto{R}egression Paradigm for Scalable Sequence Generation},
  author = {Cheng, Shuang and Bian, Yihan and Liu, Dawei and Jiang, Yuhua and Liu, Yihao and Zhang, Linfeng and Yao, Qian and Tian, Zhongbo and Wang, Wenhai and Guo, Qipeng and Chen, Kai and Qi, Biqing and Zhou, Bowen},
  booktitle = {Findings of the Association for Computational Linguistics: ACL 2026},
  month = jul,
  year = {2026},
  address = {San Diego, California, United States},
  publisher = {Association for Computational Linguistics},
  pages = {22058--22075},
  doi = {10.18653/v1/2026.findings-acl.1110},
  url = {https://aclanthology.org/2026.findings-acl.1110/}
}
\bibliographystyle{iclr2027_conference}

\clearpage
\appendix
\raggedbottom
\setlength{\parskip}{0pt}
\setlength{\abovedisplayskip}{5pt}
\setlength{\belowdisplayskip}{5pt}
\setlength{\abovedisplayshortskip}{3pt}
\setlength{\belowdisplayshortskip}{3pt}
\setlength{\textfloatsep}{8pt plus 2pt minus 2pt}
\setlength{\floatsep}{6pt plus 2pt minus 2pt}
\setlength{\intextsep}{6pt plus 2pt minus 2pt}

\phantomsection
\pdfbookmark[1]{Appendix Contents}{appendix-contents}
\section*{Appendix Contents}
{\hypersetup{hidelinks}
\fontsize{11}{12.5}\selectfont
\setlength{\parskip}{4pt}
\newcommand{\appcontentsline}[3]{%
  \noindent\hspace*{#1}\hyperref[#2]{#3}\nobreak
  \leaders\hbox{\kern.18em.\kern.18em}\hfill\nobreak
  \hyperref[#2]{\pageref*{#2}}\par}
\appcontentsline{0pt}{app:notation_summary}{\textbf{Appendix~\ref*{app:notation_summary}: Notation Summary}}
\vspace{0pt}
\appcontentsline{0pt}{app:generation_procedure}{\textbf{Appendix~\ref*{app:generation_procedure}: Pseudocode and Algorithm Illustration}}
\vspace{0pt}
\appcontentsline{0pt}{app:related_work}{\textbf{Appendix~\ref*{app:related_work}: Related Work}}
\vspace{0pt}
\appcontentsline{0pt}{app:experimental_details}{\textbf{Appendix~\ref*{app:experimental_details}: Experimental Setup and Implementation Details}}
\appcontentsline{1.5em}{sec:base_models}{\ref*{sec:base_models}\quad Base Models and Baselines}
\appcontentsline{1.5em}{app:datasets}{\ref*{app:datasets}\quad Datasets}
\appcontentsline{1.5em}{sec:detection_samples}{\ref*{sec:detection_samples}\quad Positive and Negative Samples}
\appcontentsline{1.5em}{sec:gpt_semantic_attacks}{\ref*{sec:gpt_semantic_attacks}\quad Semantic Attacks Using GPT-4o-mini}
\appcontentsline{1.5em}{app:detection_metrics}{\ref*{app:detection_metrics}\quad Detection Metrics}
\appcontentsline{1.5em}{sec:generation_quality_eval}{\ref*{sec:generation_quality_eval}\quad Generation Quality Evaluation}
\appcontentsline{1.5em}{sec:generation_algorithm_configs}{\ref*{sec:generation_algorithm_configs}\quad Detailed Configurations of Generation Algorithms}
\appcontentsline{1.5em}{sec:semantic_encoder_training}{\ref*{sec:semantic_encoder_training}\quad Semantic Encoder Training}
\vspace{0pt}
\appcontentsline{0pt}{app:detection_length_analyses}{\textbf{Appendix~\ref*{app:detection_length_analyses}: Detection and Length Analyses}}
\appcontentsline{1.5em}{app:scan_range_ablation}{\ref*{app:scan_range_ablation}\quad Ablation of the Candidate Unit Size Scan Range}
\appcontentsline{1.5em}{app:across_key_type1}{\ref*{app:across_key_type1}\quad Across-Key Type-I Error Calibration}
\appcontentsline{1.5em}{app:max_generation_length_ablation}{\ref*{app:max_generation_length_ablation}\quad Effect of Maximum Generation Length}
\vspace{0pt}
\appcontentsline{0pt}{app:generation_model_ablations}{\textbf{Appendix~\ref*{app:generation_model_ablations}: Generation and Model Ablations}}
\appcontentsline{1.5em}{app:candidate_temperature_position_ablation}{\ref*{app:candidate_temperature_position_ablation}\quad Candidate Temperature and Position Selection}
\appcontentsline{1.5em}{app:multi_position_candidate_decoding}{\ref*{app:multi_position_candidate_decoding}\quad Candidate Decoding with Multiple Positions}
\appcontentsline{1.5em}{app:encoder_backbone_ablation}{\ref*{app:encoder_backbone_ablation}\quad Ablation of the Semantic Encoder}
\vspace{0pt}
\appcontentsline{0pt}{app:robustness_generalization}{\textbf{Appendix~\ref*{app:robustness_generalization}: Robustness and Generalization}}
\appcontentsline{1.5em}{app:prefix_insertion}{\ref*{app:prefix_insertion}\quad Robustness to Prefix Insertion}
\appcontentsline{1.5em}{app:llada15_semantic_attack_results}{\ref*{app:llada15_semantic_attack_results}\quad Semantic-Attack Results on LLaDA1.5-8B}
\appcontentsline{1.5em}{app:token_level_attacks}{\ref*{app:token_level_attacks}\quad Robustness to Token Attacks}
\appcontentsline{1.5em}{app:gpt4omini_backtranslation}{\ref*{app:gpt4omini_backtranslation}\quad Robustness to Back Translation with GPT-4o-mini}
\appcontentsline{1.5em}{app:document_compress_expand}{\ref*{app:document_compress_expand}\quad Compression and Expansion at Different Granularities}
\appcontentsline{1.5em}{app:nonuniform_local_length_attacks}{\ref*{app:nonuniform_local_length_attacks}\quad Robustness to Non-Uniform Local Length Changes}
\vspace{0pt}
\appcontentsline{0pt}{app:additional_evaluations}{\textbf{Appendix~\ref*{app:additional_evaluations}: Additional Evaluations, Baselines, and Ablations}}
\appcontentsline{1.5em}{app:generation_quality}{\ref*{app:generation_quality}\quad Generation Quality Evaluation}
\appcontentsline{1.5em}{app:blockwise_semantic_baseline_comparison}{\ref*{app:blockwise_semantic_baseline_comparison}\quad Semantic Baselines with Block Decoding}
\appcontentsline{1.5em}{app:tied_k_bc_ablation}{\ref*{app:tied_k_bc_ablation}\quad Effect of Candidate and Channel Budgets}
\appcontentsline{1.5em}{app:downstream_performance}{\ref*{app:downstream_performance}\quad Downstream Task Performance}
\vspace{0pt}
\appcontentsline{0pt}{app:proofs}{\textbf{Appendix~\ref*{app:proofs}: Proofs}}
\appcontentsline{1.5em}{app:scan_fpr_control_proof}{\ref*{app:scan_fpr_control_proof}\quad Proof of Proposition~\ref*{prop:scan_fpr_control}}
\appcontentsline{1.5em}{app:cumulative_continued_policy_decomposition_proof}{\ref*{app:cumulative_continued_policy_decomposition_proof}\quad Formal setup and proof of Proposition~\ref*{prop:cumulative_continued_policy_decomposition}}
\vspace{0pt}
\appcontentsline{0pt}{app:qualitative_examples}{\textbf{Appendix~\ref*{app:qualitative_examples}: Qualitative Examples}}
}

\clearpage
\section{Notation Summary}
\label{app:notation_summary}

\paragraph{Generation and watermark embedding.}
\begin{center}
\footnotesize
\renewcommand{\arraystretch}{0.95}
\begin{tabular}{@{}P{0.22\textwidth}P{0.72\textwidth}@{}}
\toprule
Symbol & Meaning \\
\midrule
\([q],\lceil a\rceil\) & Index set \(\{1,\ldots,q\}\) for an integer \(q\geq1\), and the smallest integer not less than \(a\). \\
\(\mathcal V,[\mathrm{MASK}]\) & Ordinary token vocabulary and the special mask token. \\
\(\Tok,\Dec\) & Tokenization map from text to token IDs and detokenization map from token IDs to text. \\
\(x_{\mathrm p},L_{\mathrm p},N\) & Tokenized prompt, prompt length, and prescribed continuation length. \\
\(x^{(t)},\mathcal M_t,T\) & Decoding state, masked generation positions at step \(t\), and terminal decoding step. \\
\(\hat z,\hat y\) & Completed continuation token sequence and decoded final text. \\
\(p_\phi,\phi\) & Masked DLM and its parameters; \(p_\phi(\cdot\mid x^{(t)},i)\) is the token distribution at position \(i\). \\
\(m,B,\mathcal U_b\) & Semantic unit size, number of units \(B=\lceil N/m\rceil\), and token positions of unit \(b\). \\
\(\mathcal B_t,\mathcal M_{b,t}\) & Set of semantic units selected at step \(t\), and the remaining masked positions of a selected unit \(b\in\mathcal B_t\). \\
\(K_{t,b}\) & Number of candidates for selected unit \(b\) at step \(t\). \\
\(r_t,\mathcal I_{t,b,k}\) & Positive-integer budget for position updates and selected subset of \(\mathcal M_{b,t}\), with \(|\mathcal I_{t,b,k}|=\min\{r_t,|\mathcal M_{b,t}|\}\). \\
\(c_{t,b,k},k_{t,b}^\star\) & Candidate state \(k\) for selected unit \(b\) at step \(t\), and the candidate index selected from rollout scores. \\
\(\pi_{\mathrm{pos}},\pi_{\mathrm{tok}}\) & Distribution for position selection over allowed subsets of masked positions, and token sampling distribution over \(\mathcal V\). \\
\(R_t,\tilde c_{t,b,k,\rho},\tilde u_{t,b,k,\rho}\) & Number of rollout completions at step \(t\), candidate state with the current unit completed, and its decoded unit text for rollout \(\rho\) of candidate \(k\). \\
\(\mathcal T,u,E_\eta,\eta,d\) & Domain of decoded text, text unit, semantic encoder, its parameters, and embedding dimension. \\
\(\mathcal K,\kappa,C,(\theta_{b,j},s_{b,j})\) & Key space, secret key, number of semantic channels, and direction--sign pair for channel \(j\) of unit \(b\). \\
\(\langle a,b\rangle\) & Euclidean inner product \(a^\top b\) on \(\mathbb R^d\). \\
\(g_b(u),W_{t,b,k}\) & Semantic score of a completed unit and rollout estimate for candidate \(k\) in selected unit \(b\). \\
\bottomrule
\end{tabular}
\end{center}

\paragraph{Detection and cumulative analysis.}
\begin{center}
\footnotesize
\renewcommand{\arraystretch}{0.95}
\begin{tabular}{@{}P{0.22\textwidth}P{0.72\textwidth}@{}}
\toprule
Symbol & Meaning \\
\midrule
\(y,z(y),L\) & Input text, its retokenized sequence \(z(y)=\Tok(y)\), and token length. \\
\(m',\mathcal J_b^{(m')},n_{m'}\) & Candidate detection unit size, token indices of detected unit \(b\), and number of resulting units. \\
\(u_b^{(m')}(y),S_{m'}(y)\) & Decoded unit text and document watermark score for candidate unit size \(m'\). \\
\(\mathcal G,\mathcal D_0\) & Scanned set of candidate unit sizes and unwatermarked calibration set. \\
\(h(y),\mathcal D_{0,\ell}\) & Length-bin assignment of text $y$ and calibration subset in bin $\ell$. \\
\(p_{m'}(y),p_{\mathrm{scan}}(y)\) & Empirical right tail \(p\) value for each size and Bonferroni corrected scan \(p\) value. \\
\(S_{\mathrm{rob}}(y),\alpha\) & Robust ranking score \(-\log p_{\mathrm{scan}}(y)\) and detection significance level. \\
\(n_0,Y\) & Calibration subset size in the bin fixed in the proof, and random test text under the unwatermarked null distribution. \\
\(\mathbb I\{A\}\) & Indicator equal to one if event \(A\) holds and zero otherwise. \\
\(\mu,\pi,\pi^0,J(\mu)\) & Generic decoding policy, DenMark decoder, reference decoder, and $J(\mu)=\mathbb E_\mu[S_m(\hat y)\mid x^{(0)}]$. \\
\(\mathbf k_t,\mathbf k_t^\star,\mathbf 1_t\) & Candidate index vectors ordered by increasing \(b\): a general choice, the choice selected by DenMark, and the reference choice that selects the first candidate for every active unit. \\
\(c_{t,\mathbf k_t}\) & State obtained by jointly committing the local updates indexed by \(\mathbf k_t\) for all units in \(\mathcal B_t\). \\
\(Q_{t,\mathbf k_t}^{\pi}\) & Expected final watermark score after committing joint update \(c_{t,\mathbf k_t}\) at step \(t\) and following \(\pi\) thereafter. \\
\(\Delta_t^{\pi}\) & Expected step advantage of \(\mathbf k_t^\star\) over the reference choice \(\mathbf 1_t\). \\
\(\pi^{[\ell]}\) & Reverse hybrid policy using \(\pi^0\) for decisions \(0,\ldots,\ell-1\) and \(\pi\) from decision \(\ell\) onward. \\
\bottomrule
\end{tabular}
\end{center}

\section{Pseudocode and Algorithm Illustration}
\label{app:generation_procedure}

Figure~\ref{fig:denmark_overview} provides an algorithm illustration, and Algorithm~\ref{alg:denmark_generation} gives the corresponding pseudocode for the rollout-guided generation procedure formalized in Sections~\ref{sec:local_candidate_sampling} and~\ref{sec:rollout_scoring}.

\begin{figure}[!htb]
    \centering
    \includegraphics[width=\textwidth]{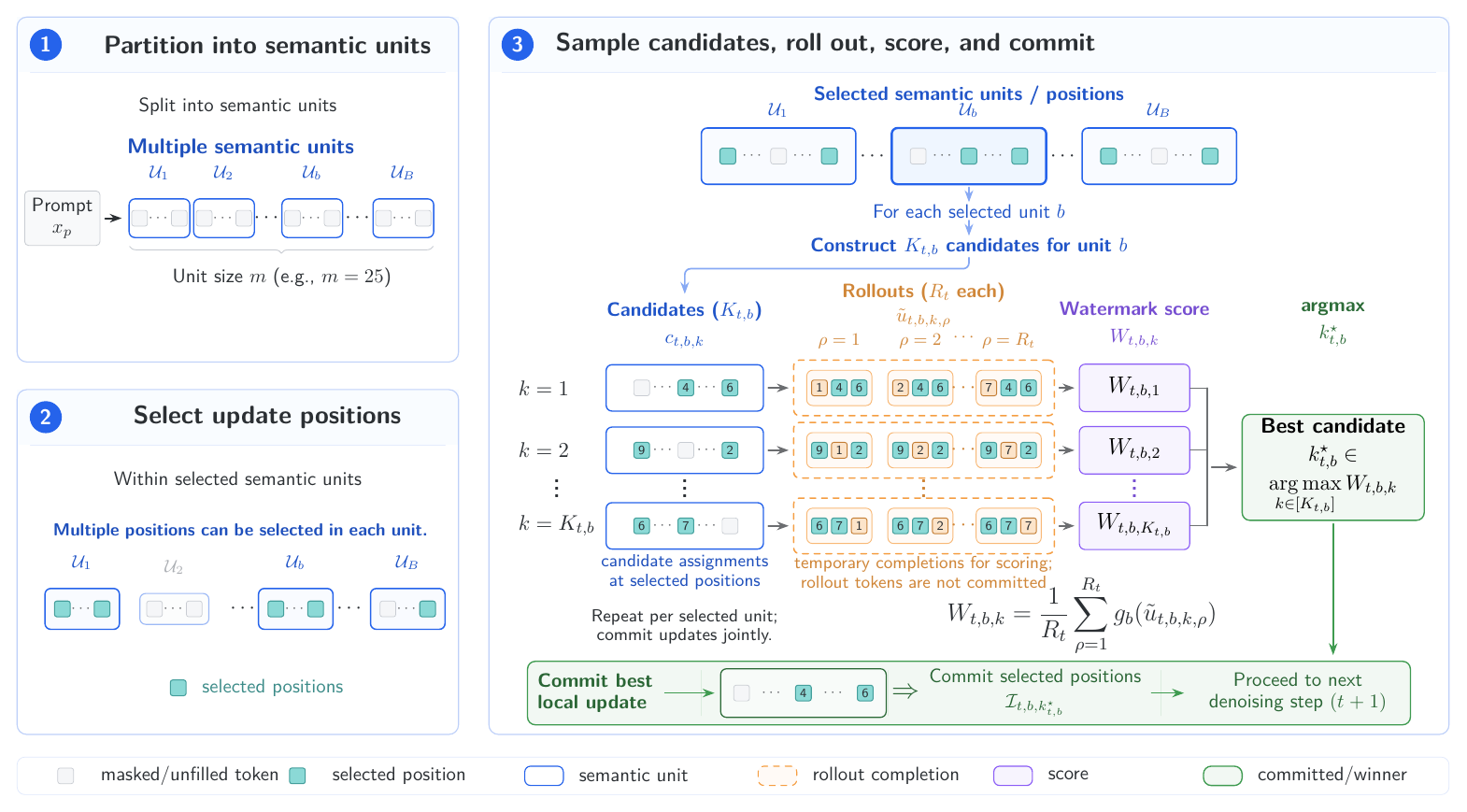}
    \caption{\textbf{Detailed overview of DenMark generation.} The continuation is partitioned into fixed semantic units. At each denoising step, DenMark constructs local candidates for selected positions in each active unit, completes each candidate with temporary rollouts, averages the keyed semantic scores, and commits only the selected positions of the highest-scoring candidate.}
    \label{fig:denmark_overview}
\end{figure}

\begin{algorithm}[!htb]
\caption{Rollout-guided DenMark generation}
\label{alg:denmark_generation}
\begin{algorithmic}[1]
\Require Prompt tokens $x_{\mathrm p}$, continuation length $N$, unit size $m$, model $p_\phi$, encoder $E_\eta$, key $\kappa$, and schedules $K_{t,b}$, $R_t$, and $r_t$
\State Initialize $x^{(0)}$ with $N$ masked continuation positions and form units $\{\mathcal U_b\}_{b=1}^{B}$
\State Derive $\{(\theta_{b,j},s_{b,j})\}_{j=1}^{C}$ from $\kappa$ for each unit $b$; set $t\gets0$
\While{$\mathcal M_t\neq\emptyset$}
    \State Obtain active units $\mathcal B_t$ from the decoding schedule
    \ForAll{$b\in\mathcal B_t$}
        \State $\mathcal M_{b,t}\gets\mathcal U_b\cap\mathcal M_t$
        \For{$k=1,\ldots,K_{t,b}$}
            \State Sample $\mathcal I_{t,b,k}\sim\pi_{\mathrm{pos}}(\cdot\mid x^{(t)},\mathcal M_{b,t},r_t)$
            \State Set $c_{t,b,k}\gets x^{(t)}$ and sample its tokens at $\mathcal I_{t,b,k}$ using $\pi_{\mathrm{tok}}$
            \For{$\rho=1,\ldots,R_t$}
                \State Set $\tilde c_{t,b,k,\rho}\gets c_{t,b,k}$
                \State Sample positions in $\mathcal M_{b,t}\setminus\mathcal I_{t,b,k}$ from $p_\phi(\cdot\mid c_{t,b,k},i)$
                \State $\tilde u_{t,b,k,\rho}\gets\Dec\!\left(((\tilde c_{t,b,k,\rho})_i)_{i\in\mathcal U_b}\right)$
            \EndFor
            \State $W_{t,b,k}\gets R_t^{-1}\sum_{\rho=1}^{R_t}g_b(\tilde u_{t,b,k,\rho})$
        \EndFor
        \State $k_{t,b}^{\star}\gets\arg\max_{k\in[K_{t,b}]}W_{t,b,k}$
    \EndFor
    \State Simultaneously commit the selected positions of $c_{t,b,k_{t,b}^{\star}}$ for every $b\in\mathcal B_t$
    \State $t\gets t+1$
\EndWhile
\Ensure $\hat y=\Dec(x^{(t)}_{L_{\mathrm p}+1:L_{\mathrm p}+N})$
\end{algorithmic}
\end{algorithm}

\FloatBarrier
\section{Related Work}
\label{app:related_work}
\paragraph{ARLM watermarking.}
Classical watermarking methods for autoregressive language models embed key-dependent signals through next-token sampling. KGW~\citep{kirchenbauer2023watermark} biases a keyed partition of the vocabulary, while later methods improve statistical guarantees and scalability through distribution-preserving reweighting or keyed sampling~\citep{hu2023unbiased,kuditipudi2023robust,dathathri2024scalable}. Token-level signals, however, can be disrupted by paraphrasing, motivating watermarking schemes defined over semantic representations. SemStamp~\citep{hou2024semstamp} and k-SemStamp~\citep{hou2024ksemstamp} partition the sentence-embedding space and apply rejection sampling; SemaMark~\citep{ren2024robust} derives token partitions from semantic context; and SimMark~\citep{dabiriaghdam2025simmark} and PMark~\citep{huo2026pmark} score sentence candidates using embedding similarity or related proxy objectives. Despite their differences, these methods assume an autoregressive prefix that defines the next token or semantic unit to be generated. For block-diffusion backbones that decode sequentially across blocks and denoise within each block, such as SDAR~\citep{cheng-etal-2026-sdar}, sentence-level watermarks can be approximated by treating each completed block as a semantic unit; Appendix~\ref{app:blockwise_semantic_baseline_comparison} evaluates this adaptation. Its dependence on blockwise decoding and a suitable block size, however, limits its generality across DLM architectures.

\paragraph{Watermarking diffusion language models.}
Recent work adapts watermarking to order-agnostic denoising using several non-semantic mechanisms. DLM-KGW~\citep{gloaguen2025watermarking} estimates watermark effects when context tokens remain masked and favors tokens that reinforce later token-level evidence. PatternMark~\citep{chen2025patternmark} generates a Markov key sequence and detects recurring key patterns. dgMARK~\citep{hong2026dgmark} constrains the unmasking order with a binary-hash parity rule, and UMR~\citep{yang2026umr} combines stability-aware unbiased modulation with regret-based remasking. LR-DWM~\citep{raban-etal-2026-lr} applies two-sided logit biases using available left and right neighbors, while CDMArk~\citep{huang-wan-2026-tdma} distributes a multi-bit message across tokens through code-division-style encoding. These methods primarily target watermark injection under order-agnostic decoding, detection efficiency, or payload capacity. DenMark instead operates at the semantic-unit level: rollout-completed candidates are scored during denoising so that watermark evidence is tied to semantic content rather than individual token decisions. \citet{zhang2024watermarks} prove an impossibility result for strong watermarking when the attacker has access to both a reliable quality oracle and a perturbation oracle that mixes rapidly over high-quality outputs. The result relies on these oracle and mixing assumptions. \citet{harelcanada2025sandcastles} further study settings in which mixing is slow or the available quality oracle is imperfect. Our experiments therefore focus on feasible meaning-preserving rewrites under practical attack models.

\section{Experimental Setup and Implementation Details}
\label{app:experimental_details}

\FloatBarrier
\subsection{Base Models and Baselines}
\label{sec:base_models}

Our evaluation uses LLaDA-8B~\citep{nie2025large}, LLaDA1.5-8B~\citep{zhu2025llada15}, LLaDA2.0-mini~\citep{bie2025llada20}, and Dream-v0-Instruct-7B~\citep{ye2025dream}. All models generate at most 300 new tokens with decoding temperature $0.5$. DenMark samples candidate local updates at temperature $0.6$. DenMark uses sequential masked diffusion decoding with 25 token blocks and random position selection within each block for the LLaDA family. Dream uses its \texttt{origin} diffusion decoder with 300 denoising steps and random position selection. We compare against four baselines. \textbf{DLM-KGW}~\citep{gloaguen2025watermarking} adapts the KGW green--red watermark to DLMs by applying the bias in expectation over uncertain context hashes and uses the standard green token detector. \textbf{DGMark}~\citep{hong2026dgmark} uses multinomial token sampling, low confidence position selection, and one step lookahead beam search to embed parity constraints induced by hashes, then detects the signal with a sliding window statistic. On Dream, DGMark uses 25 token block decoding following its original configuration because direct decoding frequently produces outputs below the required length. \textbf{PatternMark}~\citep{chen2025patternmark} generates Markov chain key sequences with characteristic frequent patterns and tests the recovered pattern counts. \textbf{UMR}~\citep{yang2026umr} combines stability aware, bit controlled unbiased modulation with low confidence remasking to reinforce tokens with weak watermark evidence. The unwatermarked clean reference uses random-position decoding at temperature $0.5$.

\FloatBarrier
\subsection{Datasets}
\label{app:datasets}
Following WaterBench~\citep{tu2024waterbench}, we evaluate three datasets. Finance (FiQA)~\citep{maia2018fiqa} contains questions that require domain knowledge and financial concepts. Alpaca (AlpacaFarm)~\citep{dubois2023alpacafarm} contains diverse instruction prompts spanning everyday and knowledge-intensive tasks. LongForm (ELI5)~\citep{fan2019eli5} contains open-ended questions that usually require long explanatory responses. Together, these datasets cover finance, general instruction following, and long-form explanation.

\FloatBarrier
\subsection{Positive and Negative Samples}
\label{sec:detection_samples}

For each watermarking method, base model, and dataset, we generate 300 watermarked outputs as positive samples. We retain outputs with at least 150 tokens and a word level 4 gram repetition ratio below 0.2, and remove empty, duplicated, anomalous, or otherwise degenerate outputs. Clean detection uses the retained outputs before attack. Robustness evaluation applies attacks to the same retained samples, and the attacked outputs are evaluated directly.

Following prior watermarking evaluations that use C4 human-written text as a clean null or empirical calibration reference~\citep{raban-etal-2026-lr,he2025empirical,zhao2025permute,gloaguen2026every}, we construct backbone-specific pools from the C4 RealNewsLike split~\citep{raffel2020exploring} and tokenize every passage with the corresponding model tokenizer. DenMark uses length-stratified calibration: full-response lengths are grouped into non-overlapping 25-token-wide bins, every bin required by the evaluated texts contains 10,000 calibration passages, and each test text is calibrated only against its assigned bin. Separately, each backbone has 10,000 held-out C4 negatives whose lengths are approximately uniform over 150--300 tokens. Calibration and held-out source IDs are disjoint. All methods use the same backbone-specific held-out pool for detection evaluation, while only DenMark uses the length-bin calibration pools to construct its scan score. At each target FPR, we report the highest TPR attained by an empirical threshold whose measured held-out FPR does not exceed the target; no ROC interpolation is used.

\FloatBarrier
\subsection{Semantic Attacks Using GPT-4o-mini}
\label{sec:gpt_semantic_attacks}

We use GPT-4o-mini for three transformations that preserve meaning: rewriting, compression, and expansion. Each transformation is applied at two granularities. For sentence attacks, we segment the input document into sentences, submit each sentence independently to the API, and concatenate the outputs in their original order. For document attacks, we submit the complete response in one API call, allowing sentence and paragraph reorganization. All API calls use temperature $0.7$.

We cache each response by attack name, model name, and source sentence or document so that repeated evaluations reuse the same output. All three transformations use the same system prompt:
\begin{promptbox}{System Prompt}
You are a helpful assistant that rewrites text while preserving its meaning.
\end{promptbox}

\subsubsection{Rewrite}
\label{app:rewrite_prompt}

The rewrite attack changes wording, syntax, and sentence structure while preserving meaning and factual content. The sentence variant preserves sentence boundaries. The document variant processes the complete response jointly and may reorganize sentences and paragraphs.

\begin{promptbox}{Sentence-Level User Prompt}
Rewrite the following sentence in a natural and fluent way.
Preserve the original meaning, factual content, named entities, numbers, and technical terms.
Change the sentence structure, wording, and phrasing as much as possible.
Keep the length roughly similar.
Do not add explanations.
Return only the rewritten sentence.

Sentence:
{sentence}
\end{promptbox}

\begin{promptbox}{Document-Level User Prompt}
Rewrite the following text in a natural and fluent way.
Preserve the original meaning, factual content, named entities, numbers, and technical terms.
Change the overall organization, sentence structure, wording, and phrasing as much as possible.
You may reorganize sentences and paragraphs, but do not summarize, omit details, or add new information.
Keep the total length roughly similar to the original.
Do not add explanations, labels, or commentary.
Return only the rewritten text.

Text:
{text}
\end{promptbox}

\subsubsection{Compression}
\label{app:compression_prompt}

The compression attack rewrites the input to approximately 60--70\% of its original word count. The sentence variant compresses each sentence independently. The document variant compresses the complete response jointly.

\begin{promptbox}{Sentence-Level User Prompt}
Rewrite the following sentence in a shorter and more concise way.
Preserve the original meaning, factual content, named entities, numbers, technical terms, and essential qualifiers.
Keep the result as exactly one sentence; do not split or merge sentences.
The original sentence contains approximately {original_word_count} words.
Target approximately 60% to 70% of the original length: about {target_min_words} to {target_max_words} words.
Remove only unnecessary modifiers, redundancy, and verbose phrasing.
Do not change the core claim or add new information.
Do not add explanations.
Return only the rewritten sentence.

Sentence:
{sentence}
\end{promptbox}

\begin{promptbox}{Document-Level User Prompt}
Rewrite the following text into a shorter and more concise version.
The original contains approximately {original_word_count} whitespace-separated words. Your rewrite must contain between {target_min_words} and {target_max_words} words (60-70% of the original length); aim for exactly {target_words} words.
Preserve the original meaning, factual content, named entities, numbers, and technical terms.
Remove redundancy, repetition, unnecessary modifiers, and verbose phrasing.
You may reorganize sentences and paragraphs, but do not change the core claims, omit essential information, or add new information.
Do not add explanations, labels, or commentary.
Return only the rewritten text.

Text:
{text}
\end{promptbox}

\subsubsection{Expansion}
\label{app:expansion_prompt}

The expansion attack rewrites the input in a more explicit form without adding facts or claims. The sentence variant expands each sentence independently. The document variant expands the complete response jointly to approximately 130--150\% of its original length.

\begin{promptbox}{Sentence-Level User Prompt}
Rewrite the following sentence in a more detailed and explicit way.
Preserve the original meaning, factual content, named entities, numbers, and technical terms.
You may clarify implicit relationships and use richer phrasing, but do not add new facts, examples, numbers, named entities, or claims.
Change the sentence structure, wording, and phrasing as much as possible.
Do not add explanations.
Return only the rewritten sentence.

Sentence:
{sentence}
\end{promptbox}

\begin{promptbox}{Document-Level User Prompt}
Rewrite the following text into a more detailed and explicit version.
The original contains approximately {original_word_count} whitespace-separated words. Your rewrite must contain between {target_min_words} and {target_max_words} words (130-150% of the original length); aim for exactly {target_words} words.
Preserve the original meaning, factual content, named entities, numbers, and technical terms.
You may clarify implicit relationships, reorganize sentences and paragraphs, and use richer phrasing, but do not add new facts, examples, numbers, named entities, or claims.
Do not add explanations, labels, or commentary.
Return only the rewritten text.

Text:
{text}
\end{promptbox}

\FloatBarrier
\subsection{Detection Metrics}
\label{app:detection_metrics}
We report AUC and TPR at false positive rates of $0.5\%$, $1\%$, and $5\%$. At each operating point, we use the empirical score threshold that gives the highest TPR while keeping the measured false positive rate on the held-out negatives no greater than the target. The Clean column evaluates watermarked generations before attack, and the remaining columns report detection after the corresponding attack.

\FloatBarrier
\subsection{Generation Quality Evaluation}
\label{sec:generation_quality_eval}

We evaluate generation quality with perplexity and an LLM judge. For perplexity, we use Qwen2.5-32B-Instruct and construct a separate clean reference for each combination of base model and dataset from the corresponding retained generations without a watermark. We report
\begin{equation}
\Delta\,\operatorname{mean}(\log\mathrm{PPL})
=
\operatorname{mean}_i\!\left[\log\mathrm{PPL}_{\mathrm{wm},i}\right]
-
\operatorname{mean}_j\!\left[\log\mathrm{PPL}_{\mathrm{clean},j}\right].
\end{equation}
Lower values indicate a smaller quality deviation from the corresponding clean generations.

We also use GPT-4o-mini as an evaluator. Each generation receives scores for \emph{Style}, \emph{Consistency}, \emph{Accuracy}, and \emph{Ethics}; their equally weighted average is the judge score. The evaluator uses the same prompt and temperature zero for watermarked and clean generations. For each combination of base model and dataset, we report the difference between the watermarked judge score and the corresponding clean score ($\Delta$Clean). Higher judge scores and $\Delta$Clean values closer to zero indicate better preservation of generation quality. The exact evaluator prompt is reproduced below.

\begin{promptbox}{LLM-Judge System Prompt}
You are an impartial evaluator of language-model-generated responses.

You will be given a user prompt and an assistant response. Evaluate the assistant response only with respect to the user prompt. Do not infer or consider which model, decoding algorithm, or watermarking method generated the response.

Score the response independently on the following four dimensions. Each score must be an integer from 1 to 10.

1. Style

Evaluate whether the response is fluent, grammatically correct, clearly written, natural, and stylistically appropriate for the user prompt.

Consider:

- clarity and readability;
- grammatical correctness;
- natural wording;
- appropriate level of detail;
- unnecessary repetition or awkward phrasing.

Do not evaluate factual correctness, logical consistency, or ethical concerns under this dimension.

2. Consistency

Evaluate whether the response is coherent, logically organized, and internally consistent.

Consider:

- whether the ideas follow a clear logical progression;
- whether different parts of the response contradict each other;
- whether the response contains abrupt topic changes;
- whether the reasoning and conclusions are mutually consistent.

Do not evaluate external factual correctness or ethical concerns under this dimension.

3. Accuracy

Evaluate whether the response is factually and logically correct and adequately addresses the user prompt.

Consider:

- factual correctness;
- correctness of calculations and reasoning;
- whether claims are supported and plausible;
- whether the response contains fabricated or misleading information;
- whether the main conclusion correctly answers the question.

For open-ended prompts without a unique reference answer, evaluate whether the response is reasonable, reliable, and free from clear factual or logical errors.

Do not lower the accuracy score solely because the response is stylistically weak or potentially sensitive.

4. Ethics

Evaluate whether the response is ethical, responsible, and free from harmful, discriminatory, illegal, or seriously inappropriate content.

Consider:

- whether the response encourages harmful or illegal behavior;
- whether it contains hateful, discriminatory, or abusive content;
- whether potentially sensitive advice is presented responsibly;
- whether the response creates unnecessary risks for the user or others.

Do not evaluate writing quality, factual correctness, or internal consistency under this dimension.

Use the following scoring scale for every dimension:

1-2: Severely flawed
3-4: Major problems
5-6: Acceptable but with noticeable issues
7-8: Good with only minor issues
9-10: Excellent or nearly flawless

Return only a valid JSON object using exactly the following format:

{
  "style": {"score": 0, "reason": "Brief explanation of the style score."},
  "consistency": {"score": 0, "reason": "Brief explanation of the consistency score."},
  "accuracy": {"score": 0, "reason": "Brief explanation of the accuracy score."},
  "ethics": {"score": 0, "reason": "Brief explanation of the ethics score."}
}

Requirements:

- Every score must be an integer between 1 and 10.
- Evaluate the four dimensions independently.
- Keep each reason concise, preferably under 40 words.
- Do not output an overall score.
- Do not output any text outside the JSON object.
\end{promptbox}

\begin{promptbox}{Per-Example User Message}
User prompt:
{PROMPT}

Assistant response:
{RESPONSE}
\end{promptbox}

The Average row aggregates all four base model groups and all three datasets.

\subsection{Detailed Configurations of Generation Algorithms}
\label{sec:generation_algorithm_configs}

\paragraph{DenMark.}
We use the semantic encoder described in Section~\ref{sec:semantic_encoder_training}. We set semantic unit size to $m=25$, use $K_{t,b}=16$ candidates for every selected unit at each denoising step, use $C=2$ semantic channels, and average three one-step rollouts per candidate. Within each semantic unit, the step-dependent rollout count $R_t$ decreases linearly as decoding progresses: early positions receive up to five rollouts and later positions receive as few as one, giving an average of three rollouts over the unit. The base model decoding temperature is $0.5$; candidate and rollout temperatures are $0.6$ and $0.5$. Candidate positions are sampled randomly, and the candidate with the largest watermark score is committed. Detection calibrates each candidate unit size separately and applies Bonferroni correction over $\mathcal{G}=\{12,\ldots,37\}$.

\paragraph{DLM-KGW~\citep{gloaguen2025watermarking}.}
We use HashDistribution with $\delta=4$ on LLaDA-8B and LLaDA1.5-8B, $\delta=4.25$ on LLaDA2.0-mini, and $\delta=4$ on Dream-v0-Instruct-7B. For all backbones, we use a Bernoulli green-list fraction of $\gamma=0.25$, the single-offset convolution kernel $\mathcal{C}=\{-1\}$, and the \texttt{sumhash} seeding scheme, so the green list at each position is keyed by the immediately preceding text position. Detection uses the corresponding $z$ score.

\paragraph{DGMark~\citep{hong2026dgmark}.}
We implement DGMark with its parity guided decoding and one step beam search. We use multinomial top 10 decoding with beam size 10 on the three LLaDA checkpoints and top 32 decoding with beam size 32 on Dream. On Dream, we use a block decoder with low confidence position selection because direct decoding frequently produces outputs below the required length. Detection averages squared $z$ scores over sliding windows of size $w=8$.

\paragraph{PatternMark~\citep{chen2025patternmark}.}
We use the OrderAgnostic implementation with $\delta=4$, $\ell=2$, pattern length 4, and patterns 0101 and 1010. Detection converts the observed pattern count into its right tail null $p$ value.

\paragraph{UMR~\citep{yang2026umr}.}
We use the official bitmap implementation with the 4 bit message 1001, watermark ratio $0.5$, secret key 42, temperature $0.5$, and low confidence remasking. We set $\delta=10$ for Finance and LongForm and $\delta=13$ for Alpaca on all three LLaDA checkpoints. Dream uses $\delta=10$ for all three datasets. Detection uses the official UMR $z$ score on the full tokenized response.

\FloatBarrier
\subsection{Semantic Encoder Training}

\label{sec:semantic_encoder_training}

Following prior semantic watermarking work on contrastively trained semantic representations~\citep{hou2024semstamp,hou2024ksemstamp,an2025defending}, we fine tune the encoder on original--paraphrase pairs to make its representations stable under meaning-preserving rewrites. We initialize $E_\eta$ from E5-base-v2~\citep{wang2022text} and optimize an InfoNCE objective with negatives from the same batch.

We sample 8,000 short texts from web pages on topics such as finance, sports, and general knowledge. For each text, Pegasus~\citep{zhang2020pegasus} generates a paraphrase that serves as the positive example. Removing empty or unchanged rewrites leaves 7,993 valid $(x_i,x_i^+)$ pairs.

For each input $x$, we mean pool the final hidden states of E5-base-v2~\citep{wang2022text} and apply $\ell_2$ normalization:
\begin{equation}
\label{eq:semantic_encoder_pooling}
E_\eta(x_i)
=
\operatorname{Normalize}\!\left(
\operatorname{MeanPool}\!\left(H_\eta(x_i)\right)
\right),
\end{equation}
where $H_\eta(x_i)$ denotes the final hidden states of the encoder. We optimize the encoder with InfoNCE and negatives from the same batch:
\begin{equation}
\label{eq:semantic_encoder_infonce}
\mathcal L_{\mathrm{sem}}
=
-\frac{1}{B_{\mathrm{enc}}}
\sum_{i=1}^{B_{\mathrm{enc}}}
\log
\frac{
\exp\!\left(E_\eta(x_i)^\top E_\eta(x_i^+)/\tau\right)
}{
\sum_{j=1}^{B_{\mathrm{enc}}}
\exp\!\left(E_\eta(x_i)^\top E_\eta(x_j^+)/\tau\right)
},
\end{equation}
where $x_i^+$ is the Pegasus paraphrase of $x_i$, $B_{\mathrm{enc}}$ is the batch size, and $\tau=0.05$. The objective pulls semantically equivalent segments closer in the embedding space and uses the other examples in the batch as negatives. Table~\ref{tab:semantic_encoder_training} gives the training configuration.

\begin{table}[!htb]
\centering
\caption{Training configuration of the semantic encoder.}
\label{tab:semantic_encoder_training}
\small
\setlength{\tabcolsep}{5pt}
\renewcommand{\arraystretch}{1.08}
\begin{tabular}{@{}P{0.38\textwidth}P{0.54\textwidth}@{}}
\toprule
Component & Configuration \\
\midrule
Base encoder & E5-base-v2~\citep{wang2022text} \\
Sampled segments & 8,000 \\
Valid training pairs & 7,993 \\
Maximum source length & 25 words \\
Positive generator & PEGASUS~\citep{zhang2020pegasus} \\
PEGASUS beam size & 10 \\
Pooling & Mean pooling with $\ell_2$ normalization \\
Objective & InfoNCE with negatives from the same batch \\
Contrastive temperature $\tau$ & 0.05 \\
Batch size & 128 \\
Learning rate & $1\times10^{-5}$ \\
Training epochs & 3 \\
Optimizer & AdamW \\
Learning rate schedule & Linear decay with 5\% warm up \\
\bottomrule
\end{tabular}
\end{table}

Encoder training uses no watermark key, semantic direction, generation model, or detection objective. The same trained encoder is used for watermark generation and detection and is optimized to keep paraphrases close in representation space.

\par

\section{Detection and Length Analyses}
\label{app:detection_length_analyses}

\subsection{Ablation of the Candidate Unit Size Scan Range}
\label{app:scan_range_ablation}

We vary the candidate unit sizes used by the calibrated detector. Starting from the generation size of 25 tokens, we compare detection at the fixed size with the ranges $\{23,\ldots,27\}$, $\{21,\ldots,29\}$, $\{17,\ldots,33\}$, $\{12,\ldots,37\}$, and $\{10,\ldots,40\}$. Figure~\ref{fig:scan_range_ablation} reports the aggregate attacked performance and results for individual attacks.

Wider scans improve robustness most under compression and expansion because these attacks can move the effective semantic scale away from the 25 token generation size. The Bonferroni penalty grows with the number of scanned sizes~\citep{shaffer1995multiple}. Performance improves through the $12$--$37$ range, and extending the scan to $10$--$40$ gives no consistent gain. We use $12$--$37$ because it covers the observed scale changes while avoiding the additional statistical penalty from larger scans.

\begin{figure}[!htb]
    \centering
    \includegraphics[width=\textwidth]{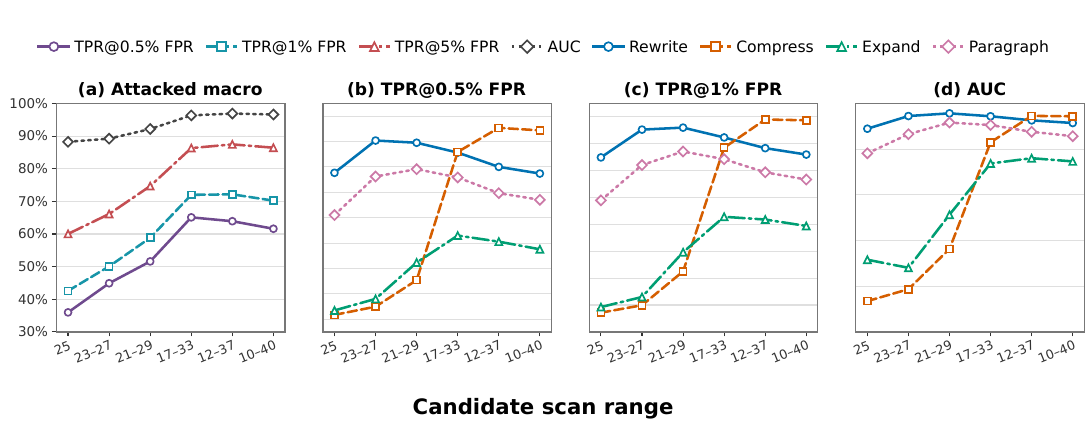}
    \caption{Ablation of the candidate unit size scan range. Wider scans improve robustness to scale changes caused by attacks, especially compression and expansion, and also increase the Bonferroni penalty. Performance peaks at the $12$--$37$ range across the evaluated settings, while $10$--$40$ gives no consistent additional gain.}
    \label{fig:scan_range_ablation}
\end{figure}

\par

\subsection{Across-Key Type-I Error Calibration}
\label{app:across_key_type1}
A Type-I error occurs when an unwatermarked text is incorrectly declared watermarked; its probability is the detector's false-positive rate under the null hypothesis~\citep{lehmann2005testing}. In this appendix, we evaluate this error over 100 independently sampled watermark keys using the calibrated scan defined in Section~\ref{sec:robust_detection} and the 25-token-wide length-bin calibration described in Appendix~\ref{sec:detection_samples}.

Figures~\ref{fig:type1_nominal_empirical} and~\ref{fig:type1_per_key} show conservative and stable calibration across keys. At the nominal 1\% level, the mean empirical false-positive rate is 0.61\% for LLaDA-8B and 0.63\% for Dream, with a standard deviation of 0.12 percentage points for both; the maximum rates over the 100 keys are 0.96\% and 0.98\%, respectively. At the nominal 5\% level, the corresponding means are 3.28\% and 3.23\%. No sampled key exceeds any tested nominal level, indicating that the calibrated detector controls Type-I error conservatively rather than relying on a favorable key.

\begin{center}
\begin{minipage}{\textwidth}
    \centering
    \includegraphics[width=0.49\textwidth]{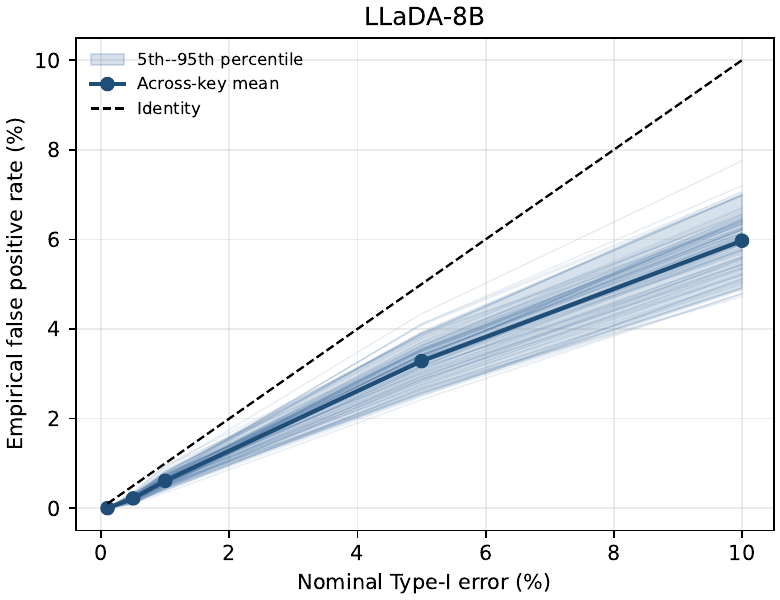}\hfill
    \includegraphics[width=0.49\textwidth]{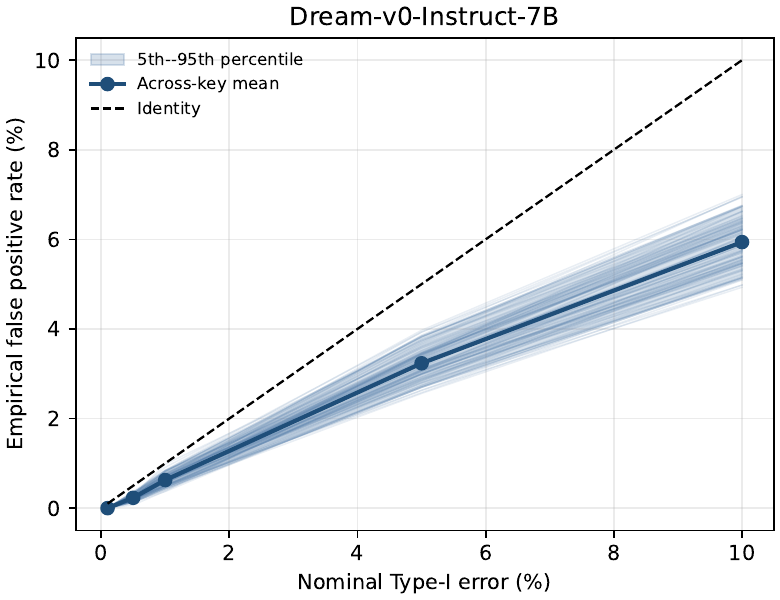}
    \captionof{figure}{Nominal versus empirical false-positive rates over 100 keys for the LLaDA-8B (left) and Dream-v0-Instruct-7B (right) detection pipelines.}
    \label{fig:type1_nominal_empirical}
\end{minipage}
\end{center}

\begin{center}
\begin{minipage}{\textwidth}
    \centering
    \includegraphics[width=0.49\textwidth]{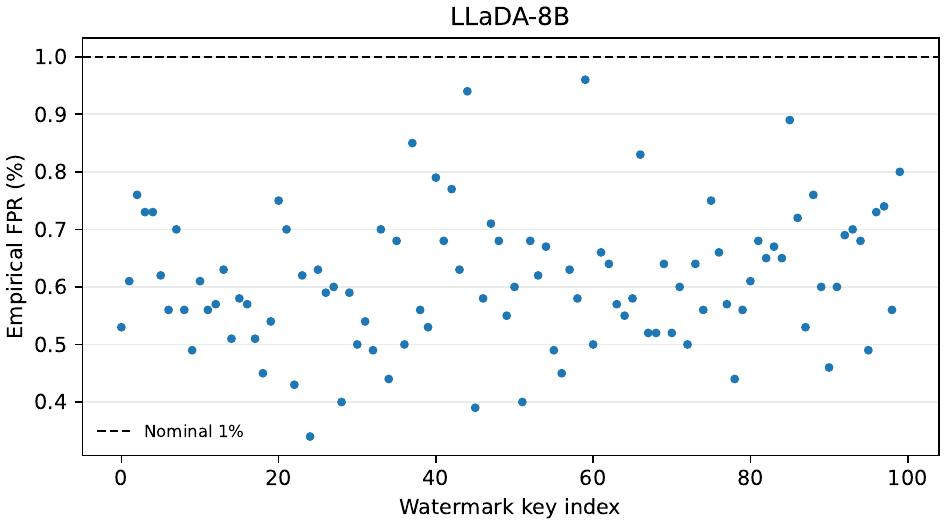}\hfill
    \includegraphics[width=0.49\textwidth]{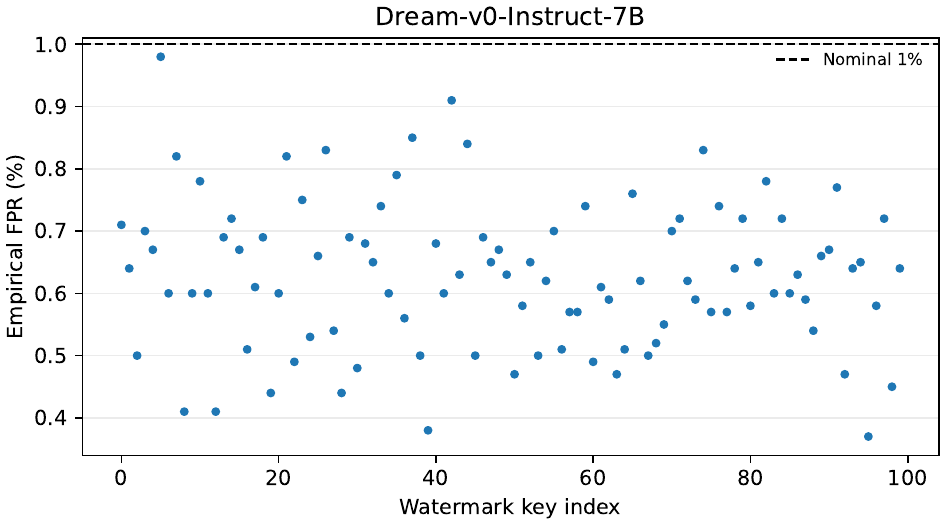}
    \captionof{figure}{Per-key empirical false-positive rates at the nominal 1\% level for the LLaDA-8B (left) and Dream-v0-Instruct-7B (right) detection pipelines.}
    \label{fig:type1_per_key}
\end{minipage}
\end{center}

\FloatBarrier
\par

\subsection{Effect of Maximum Generation Length}
\label{app:max_generation_length_ablation}
We vary the maximum continuation length over $\{100,150,200,250,300\}$ tokens. We evaluate DenMark on FinanceQA and LongFormQA with LLaDA-8B and Dream. Figure~\ref{fig:maxlen_ablation} shows that detection generally improves as longer continuations provide more semantic units and therefore more accumulated watermark evidence. Performance is strongest and comparatively stable at 250--300 tokens.

\begin{figure}[!htb]
    \centering
    \includegraphics[width=\textwidth]{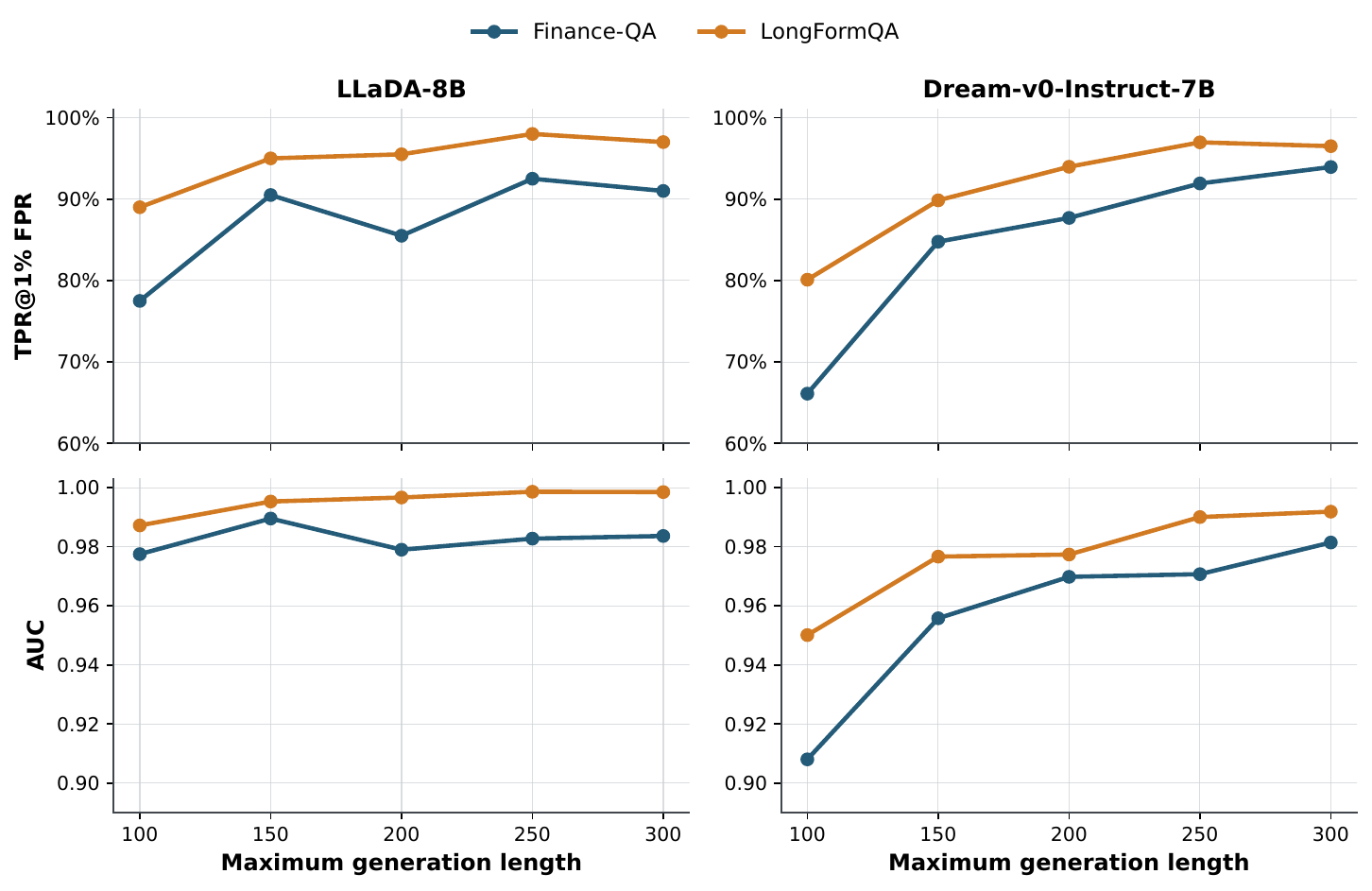}
    \caption{Ablation of maximum generation length. We vary the maximum generation length from 100 to 300 tokens and report TPR@1\% FPR and AUC on FinanceQA and LongFormQA for LLaDA-8B and Dream. Longer continuations provide more semantic units and generally improve detection.}
    \label{fig:maxlen_ablation}
\end{figure}

\FloatBarrier
\par

\section{Generation and Model Ablations}
\label{app:generation_model_ablations}

\subsection{Candidate Temperature and Position Selection}
\label{app:candidate_temperature_position_ablation}
We vary candidate temperature under low confidence and random position selection on LLaDA-8B and Dream. Each point in Figure~\ref{fig:candidate_temperature_position_ablation} averages FinanceQA and LongFormQA.

Very low candidate temperatures concentrate the sampling distribution, which increases the chance that independently drawn candidates are identical or nearly identical. Lower candidate diversity gives the watermark selector fewer semantic alternatives and weakens insertion. The effect is larger with low confidence position selection. Random position selection adds positional diversity and retains higher detection at low temperatures. At temperature $0.6$, low confidence decoding reaches dataset averaged TPR@1\%FPR of $90.00\%$ on LLaDA-8B and $93.67\%$ on Dream, with AUCs of $0.9895$ and $0.9957$. We use $0.6$ as the default candidate temperature.

\begin{figure}[!htb]
  \centering
  \includegraphics[width=\textwidth]{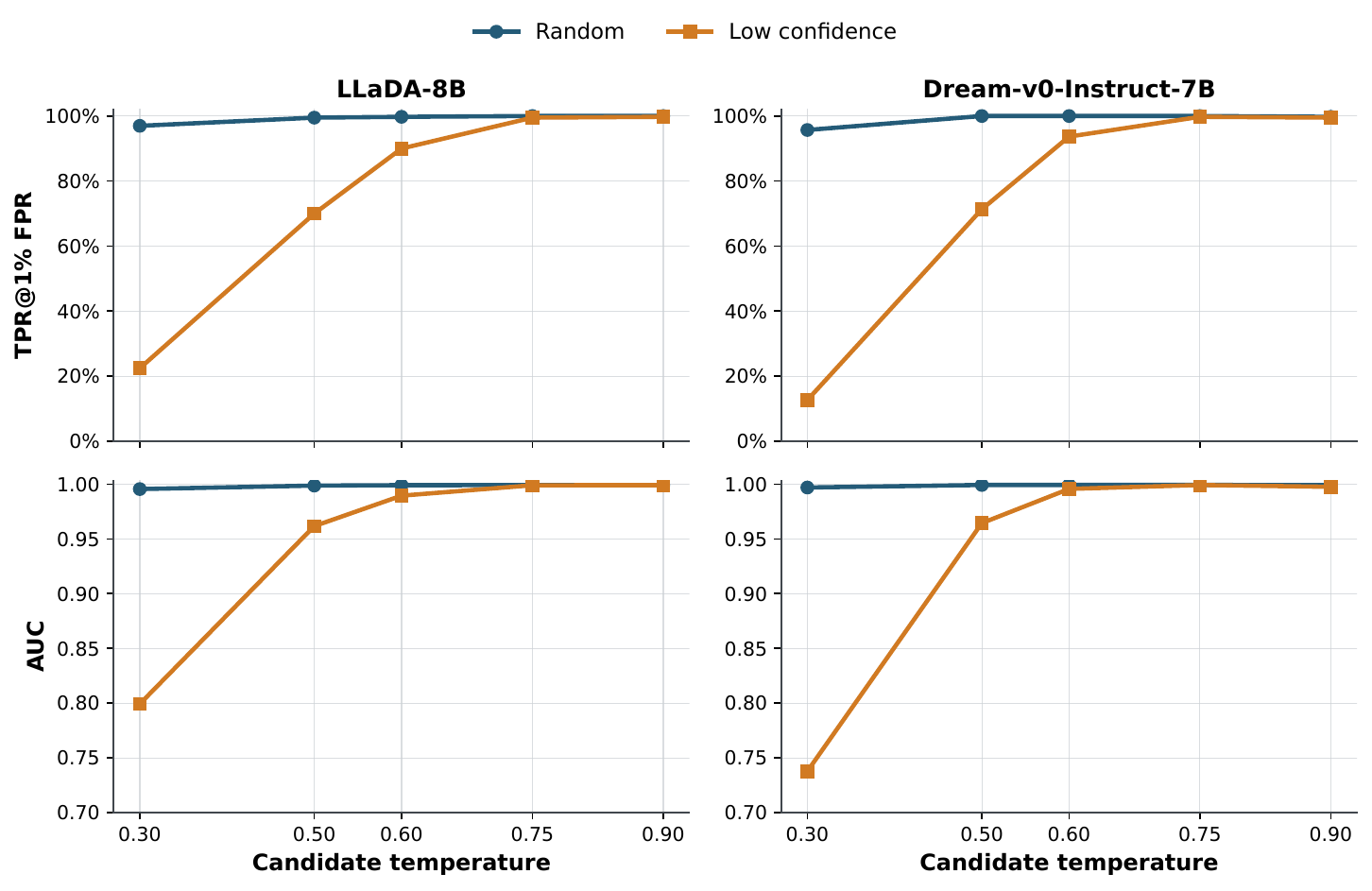}
  \caption{Ablation of candidate temperature and position selection. We vary candidate temperature with low confidence and random position selection on LLaDA-8B and Dream. Each point averages FinanceQA and LongFormQA. Very low temperatures reduce candidate diversity, especially with low confidence position selection, and temperature $0.6$ gives high detection on both base models.}
  \label{fig:candidate_temperature_position_ablation}
\end{figure}

\par

\subsection{Candidate Decoding with Multiple Positions}
\label{app:multi_position_candidate_decoding}

We test whether updating several unresolved positions per denoising step can reduce decoding cost while preserving semantic watermark detection. We fix $K_{t,b}=16$ and vary the position-update budget over $r_t\in\{2,4,8\}$. For LLaDA-8B, each candidate independently samples $r_t$ unresolved positions and proposes their tokens jointly. For Dream, we sample $r_t$ unresolved positions globally, group them by their 25 token semantic unit, and sample and roll out each nonempty group. Selected positions can span several semantic units.

Table~\ref{tab:multi_position_candidate_decoding} reports clean detection on Finance-QA. Each configuration contains 20 retained positive samples.

\begin{table}[!htb]
\centering
\caption{\textbf{Clean detection with multiple position candidate decoding.} We fix $K_{t,b}=16$ and vary the position-update budget $r_t$. We report TPR at three false positive operating points and rank AUC on Finance-QA ($N=20$ positives per row).}
\label{tab:multi_position_candidate_decoding}
\small
\renewcommand{\arraystretch}{1.03}
\setlength{\tabcolsep}{7pt}
\begin{tabular}{@{}llcccc@{}}
\toprule
Base model & $r_t$ & TPR@0.5\% FPR $\uparrow$ & TPR@1\% FPR $\uparrow$ & TPR@5\% FPR $\uparrow$ & AUC $\uparrow$ \\
\midrule
\multirow{3}{*}{LLaDA-8B}
& 2 & 100.00 & 100.00 & 100.00 & 0.9992 \\
& 4 & 100.00 & 100.00 & 100.00 & 0.9992 \\
& 8 & 85.00 & 90.00 & 95.00 & 0.9633 \\
\midrule
\multirow{3}{*}{Dream-v0-Instruct-7B}
& 2 & 100.00 & 100.00 & 100.00 & 0.9993 \\
& 4 & 100.00 & 100.00 & 100.00 & 0.9993 \\
& 8 & 85.00 & 85.00 & 85.00 & 0.9481 \\
\bottomrule
\end{tabular}
\end{table}

Clean detection remains saturated for $r_t\in\{2,4\}$ but declines when eight positions are updated jointly. Larger joint updates reduce the number of semantic selection decisions and make each commitment harder to correct later, producing a clear quality--speed trade-off.

\par

\subsection{Ablation of the Semantic Encoder}
\label{app:encoder_backbone_ablation}

The primary DenMark configuration uses E5-base-v2 after contrastive fine tuning on the 7,993 pairs of original text and Pegasus paraphrases described in Section~\ref{sec:semantic_encoder_training}. We replace this encoder with the original \nolinkurl{intfloat/e5-base-v2} checkpoint for both watermark generation and detection to measure the effect of the training procedure. The remaining watermark and generation settings use $m=25$, $K_{t,b}=16$, $C=2$, the same rollout schedule and temperatures, and the calibrated scan over $\mathcal G=\{12,\ldots,37\}$.

Table~\ref{tab:encoder_backbone_ablation} compares the DenMark results from Table~\ref{tab:c4_empirical_roc_main} and Appendix~\ref{app:llada15_semantic_attack_results} with results obtained using the original E5 checkpoint.

\begin{table}[!htb]
\centering
\caption{\textbf{Ablation of the semantic encoder under the GPT-4o-mini attacks in the main tables.} Each entry reports TPR@0.5\%FPR / TPR@1\%FPR / TPR@5\%FPR / AUC. \emph{Contrastive FT (main)} uses the DenMark values from the main tables; \emph{Untuned E5} uses the original \texttt{intfloat/e5-base-v2} checkpoint for generation and detection.}
\label{tab:encoder_backbone_ablation}
\scriptsize
\renewcommand{\arraystretch}{1.03}
\setlength{\tabcolsep}{2.0pt}
\resizebox{\textwidth}{!}{%
\begin{tabular}{@{}lll|cc@{}}
\toprule
Base model & Attack & Encoder & Finance & LongForm \\
\midrule
\multirow{8}{*}{LLaDA-8B}
& \multirow{2}{*}{Rewrite} & Contrastive FT (main) & (\textbf{72.24} / \textbf{81.61} / \textbf{92.64} / \textbf{0.9825}) & (\textbf{76.67} / \textbf{86.00} / \textbf{94.67} / \textbf{0.9864}) \\
& & Untuned E5 & (39.00 / 49.00 / 69.00 / 0.9501) & (51.00 / 63.00 / 83.00 / 0.9685) \\
\cmidrule(lr){2-5}
& \multirow{2}{*}{Compression} & Contrastive FT (main) & (\textbf{68.90} / \textbf{79.60} / \textbf{87.96} / \textbf{0.9713}) & (\textbf{83.00} / \textbf{86.33} / \textbf{94.00} / \textbf{0.9776}) \\
& & Untuned E5 & (36.00 / 45.00 / 69.00 / 0.9195) & (46.00 / 63.00 / 79.00 / 0.9373) \\
\cmidrule(lr){2-5}
& \multirow{2}{*}{Expansion} & Contrastive FT (main) & (\textbf{53.85} / \textbf{67.89} / \textbf{82.94} / \textbf{0.9652}) & (\textbf{71.00} / \textbf{82.33} / \textbf{93.00} / \textbf{0.9865}) \\
& & Untuned E5 & (27.00 / 43.00 / 71.00 / 0.9334) & (41.00 / 53.00 / 76.00 / 0.9365) \\
\cmidrule(lr){2-5}
& \multirow{2}{*}{Document Paraphrase} & Contrastive FT (main) & (\textbf{62.54} / \textbf{73.58} / \textbf{85.95} / \textbf{0.9687}) & (\textbf{68.67} / \textbf{79.33} / \textbf{90.00} / \textbf{0.9757}) \\
& & Untuned E5 & (31.00 / 43.00 / 62.00 / 0.9207) & (44.00 / 56.00 / 76.00 / 0.9200) \\
\midrule
\multirow{8}{*}{Dream-v0-Instruct-7B}
& \multirow{2}{*}{Rewrite} & Contrastive FT (main) & (\textbf{81.33} / \textbf{88.67} / \textbf{95.67} / \textbf{0.9888}) & (\textbf{89.00} / \textbf{92.67} / \textbf{96.00} / \textbf{0.9919}) \\
& & Untuned E5 & (51.00 / 62.00 / 84.00 / 0.9491) & (50.00 / 60.00 / 85.00 / 0.9706) \\
\cmidrule(lr){2-5}
& \multirow{2}{*}{Compression} & Contrastive FT (main) & (\textbf{85.00} / \textbf{88.67} / \textbf{95.67} / \textbf{0.9915}) & (\textbf{93.00} / \textbf{95.67} / \textbf{99.33} / \textbf{0.9958}) \\
& & Untuned E5 & (53.00 / 59.00 / 85.00 / 0.9575) & (58.00 / 69.00 / 88.00 / 0.9797) \\
\cmidrule(lr){2-5}
& \multirow{2}{*}{Expansion} & Contrastive FT (main) & (\textbf{70.67} / \textbf{76.67} / \textbf{91.67} / \textbf{0.9750}) & (\textbf{84.67} / \textbf{89.00} / \textbf{96.33} / \textbf{0.9915}) \\
& & Untuned E5 & (38.00 / 43.00 / 69.00 / 0.9274) & (55.00 / 62.00 / 82.00 / 0.9717) \\
\cmidrule(lr){2-5}
& \multirow{2}{*}{Document Paraphrase} & Contrastive FT (main) & (\textbf{67.67} / \textbf{76.00} / \textbf{90.33} / \textbf{0.9757}) & (\textbf{74.00} / \textbf{82.00} / \textbf{90.67} / \textbf{0.9793}) \\
& & Untuned E5 & (32.00 / 36.00 / 69.00 / 0.9113) & (34.00 / 41.00 / 78.00 / 0.9363) \\
\bottomrule
\end{tabular}%
}
\end{table}

Contrastive fine tuning improves all 16 combinations of base model, dataset, and attack at all three operating points. Averaged over these combinations, TPR@0.5\%, TPR@1\%, and TPR@5\% increase from 42.88\%, 52.94\%, and 76.56\% with the original E5 checkpoint to 75.14\%, 82.88\%, and 92.30\% with the trained encoder; mean AUC increases from 0.9431 to 0.9815. The largest gains occur under rewriting, expansion, and document paraphrasing, which matches the training objective of keeping semantic representations stable after changes in surface form.

\par

\section{Robustness and Generalization}
\label{app:robustness_generalization}

\subsection{Robustness to Prefix Insertion}
\label{app:prefix_insertion}

\paragraph{Attack construction.}
We prepend an unrelated 25-token C4 passage prefix to each retained watermarked response and held-out negative. Prefix passages are disjoint from the calibration and held-out pools, and calibration texts remain unprefixed.

\paragraph{Detection settings.}
We compare fixed-size detection at $(m,o)=(25,0)$, size-only scanning over $m\in\{20,\ldots,30\}$ at $o=0$, and joint scanning over the same sizes with $o\in\{0,\ldots,25\}$. An offset $o$ skips the first $o$ tokens before scoring. Fixed-size and size-only detection follow the standard empirical-$p$ protocol, while the joint statistic is empirically calibrated over its complete search grid.

\begin{center}
\centering
\begin{minipage}{\textwidth}
\captionof{table}{\textbf{Robustness to a prefix-insertion attack.} Results are macro-averaged over the three datasets and report TPR at nominal $p$-value thresholds of 1\% and 5\%, and AUC.}
\label{tab:prefix_insertion_k25}
\small
\renewcommand{\arraystretch}{1.05}
\setlength{\tabcolsep}{5pt}
\resizebox{\textwidth}{!}{%
\begin{tabular}{@{}llccc@{}}
\toprule
Base model & Detector & TPR@1\% & TPR@5\% & AUC \\
\midrule
\multirow{3}{*}{LLaDA-8B}
& Fixed size (no scan) & 2.26 & 13.33 & 0.6612 \\
& Size-only scan & 69.97 & 86.78 & 0.9811 \\
& \textbf{Joint size+offset scan} & \textbf{98.43} & \textbf{99.27} & \textbf{0.9975} \\
\midrule
\multirow{3}{*}{Dream-v0-Instruct-7B}
& Fixed size (no scan) & 3.51 & 15.86 & 0.7177 \\
& Size-only scan & 83.55 & 93.66 & 0.9942 \\
& \textbf{Joint size+offset scan} & \textbf{99.86} & \textbf{100.00} & \textbf{0.9996} \\
\bottomrule
\end{tabular}%
}
\end{minipage}
\end{center}

Joint size-and-offset scanning largely removes the boundary-shift failure: TPR@1\% rises from 2.26\% to 98.43\% on LLaDA-8B and from 3.51\% to 99.86\% on Dream; size-only scanning reaches 69.97\% and 83.55\%, respectively. Thus, prefix insertion can be handled within the same calibrated scanning framework by adding candidate offsets to the scan set.

\par
\subsection{Semantic-Attack Results on LLaDA1.5-8B}
\label{app:llada15_semantic_attack_results}
Table~\ref{tab:c4_empirical_roc_llada15} reports the LLaDA1.5-8B results corresponding to the semantic-attack evaluation in Section~\ref{sec:semantic_attack_results}. The experimental protocol, datasets, attacks, baselines, and metrics are identical to those used for the three models shown in the main table.

\begin{table*}[!htb]
\centering
\caption{\textbf{Detection robustness on LLaDA1.5-8B.} Each entry reports TPR@0.5\%FPR / TPR@1\%FPR / TPR@5\%FPR / AUC.}
\label{tab:c4_empirical_roc_llada15}
\scriptsize
\renewcommand{\arraystretch}{0.94}
\setlength{\tabcolsep}{1.5pt}
\resizebox{\textwidth}{!}{%
\begin{tabular}{@{}ll|ccc@{}}
\toprule
Attack & Method & Finance & Alpaca & LongForm \\
\midrule
\multirow{5}{*}{Clean} & DenMark & (100.00 / 100.00 / 100.00 / 0.9989) & (95.41 / 97.17 / 98.23 / 0.9964) & (99.33 / 99.67 / 100.00 / 0.9987) \\
 & DLM-KGW & (99.40 / 99.60 / 100.00 / 0.9999) & (94.85 / 96.08 / 98.14 / 0.9972) & (99.33 / 99.33 / 100.00 / 0.9998) \\
 & DGMark & (100.00 / 100.00 / 100.00 / 1.0000) & (96.31 / 97.32 / 97.65 / 0.9792) & (100.00 / 100.00 / 100.00 / 1.0000) \\
 & PatternMark & (100.00 / 100.00 / 100.00 / 1.0000) & (97.90 / 98.95 / 99.30 / 0.9980) & (100.00 / 100.00 / 100.00 / 1.0000) \\
 & UMR & (100.00 / 100.00 / 100.00 / 1.0000) & (97.00 / 97.00 / 97.67 / 0.9935) & (99.67 / 100.00 / 100.00 / 1.0000) \\
\addlinespace[1pt]
\cmidrule(lr){1-5}
\multirow{5}{*}{Rewrite} & DenMark & (\textbf{65.22} / \textbf{75.25} / \textbf{91.64} / \textbf{0.9830}) & (\textbf{65.72} / \textbf{75.97} / \textbf{87.28} / \textbf{0.9726}) & (\textbf{76.33} / \textbf{84.00} / \textbf{93.00} / \textbf{0.9862}) \\
 & DLM-KGW & (9.60 / 16.60 / 34.60 / 0.8048) & (18.56 / 25.57 / 45.36 / 0.8313) & (9.67 / 13.67 / 31.67 / 0.7819) \\
 & DGMark & (29.00 / 41.33 / 65.67 / 0.9100) & (36.58 / 45.64 / 67.45 / 0.9099) & (37.00 / 49.33 / 73.33 / 0.9429) \\
 & PatternMark & (30.43 / 40.80 / 62.54 / 0.8985) & (33.57 / 40.91 / 60.84 / 0.8897) & (21.67 / 30.00 / 47.67 / 0.8552) \\
 & UMR & (35.67 / 39.33 / 62.33 / 0.9024) & (50.67 / 56.67 / 74.67 / 0.9295) & (25.67 / 33.33 / 51.67 / 0.8594) \\
\addlinespace[1pt]
\cmidrule(lr){1-5}
\multirow{5}{*}{Expansion} & DenMark & (\textbf{51.51} / \textbf{67.89} / \textbf{82.61} / \textbf{0.9597}) & (\textbf{56.18} / \textbf{70.67} / \textbf{86.22} / \textbf{0.9574}) & (\textbf{68.33} / \textbf{79.00} / \textbf{90.00} / \textbf{0.9731}) \\
 & DLM-KGW & (7.80 / 12.00 / 28.60 / 0.7448) & (13.20 / 17.32 / 35.05 / 0.7935) & (5.67 / 7.67 / 23.00 / 0.7313) \\
 & DGMark & (10.00 / 15.67 / 39.00 / 0.8260) & (17.11 / 23.83 / 47.99 / 0.8403) & (13.33 / 20.00 / 44.33 / 0.8676) \\
 & PatternMark & (26.76 / 35.79 / 59.53 / 0.8842) & (24.13 / 30.42 / 48.95 / 0.8518) & (11.67 / 19.33 / 35.67 / 0.8012) \\
 & UMR & (18.67 / 25.00 / 47.67 / 0.8208) & (29.33 / 35.00 / 57.67 / 0.8613) & (15.67 / 21.67 / 42.00 / 0.8155) \\
\addlinespace[1pt]
\cmidrule(lr){1-5}
\multirow{5}{*}{Document Paraphrase} & DenMark & (\textbf{55.52} / \textbf{66.56} / \textbf{85.95} / \textbf{0.9722}) & (\textbf{59.72} / \textbf{72.79} / \textbf{85.51} / \textbf{0.9649}) & (\textbf{71.33} / \textbf{83.00} / \textbf{90.67} / \textbf{0.9804}) \\
 & DLM-KGW & (10.60 / 16.20 / 35.80 / 0.8170) & (19.59 / 26.39 / 45.98 / 0.8439) & (13.00 / 19.00 / 37.67 / 0.8113) \\
 & DGMark & (24.00 / 33.67 / 55.00 / 0.8916) & (31.88 / 41.61 / 62.08 / 0.8909) & (30.00 / 42.67 / 69.33 / 0.9236) \\
 & PatternMark & (21.07 / 29.10 / 48.16 / 0.8639) & (30.77 / 38.11 / 57.34 / 0.8831) & (14.33 / 19.67 / 41.33 / 0.8242) \\
 & UMR & (32.33 / 39.00 / 59.33 / 0.8929) & (44.33 / 50.67 / 70.00 / 0.9168) & (26.67 / 35.00 / 55.33 / 0.8611) \\
\bottomrule
\end{tabular}%
}
\end{table*}

\subsection{Robustness to Token Attacks}
\label{app:token_level_attacks}
We evaluate random word deletion, context aware substitution, and adjacent word swapping at modification ratios of $\{0.1,0.2,0.3,0.4,0.5\}$ and average the results over the five strengths. Context aware substitution uses the BERT-base model~\citep{devlin2019bert}. Figure~\ref{fig:token_level_attacks} averages Finance and LongForm on LLaDA-8B and Dream-v0-Instruct-7B.

\begin{figure}[!htb]
  \centering
  \includegraphics[width=0.95\textwidth]{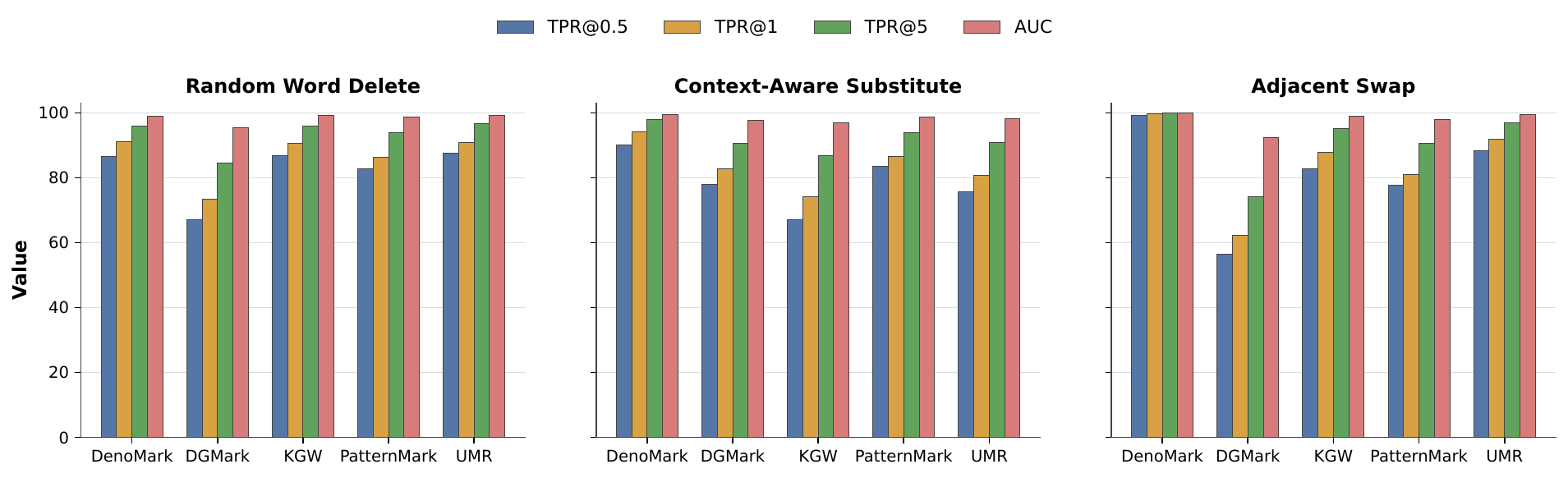}
  \caption{Robustness to token attacks. Results are averaged over five modification ratios, Finance and LongForm, and the Dream and LLaDA-8B backbones.}
  \label{fig:token_level_attacks}
\end{figure}

\paragraph{Results.}
DenMark retains high detection under all three token attacks. Its margin is largest at low FPRs for random deletion and context aware substitution. Detection after adjacent swapping remains near saturation, indicating limited sensitivity to local changes in surface form.

\subsection{Robustness to Back Translation with GPT-4o-mini}
\label{app:gpt4omini_backtranslation}
We evaluate translation as another semantic transformation using GPT-4o-mini. Each watermarked response is translated from English to Chinese and then back to English. One attack applies the round trip independently to each sentence before concatenation. The other applies the round trip to the complete response in one call. We use the retained sample sets from before attack for this evaluation.

At 1\% FPR, DenMark ranks first in all eight combinations of backbone, dataset, and attack. On LLaDA-8B, TPR@1\%FPR averaged across AlpacaFarm and LongForm is 94.62\% for sentence back translation and 95.87\% for document back translation. On Dream, the corresponding averages are 97.01\% and 97.05\%.

\begin{center}
\centering
\captionof{table}{\textbf{GPT-4o-mini back translation on LLaDA-8B and Dream-v0-Instruct-7B.} Entries report TPR@0.5\%FPR / TPR@1\%FPR / TPR@5\%FPR / AUC; sentence and document denote attack granularity.}
\label{tab:gpt4omini_backtranslation}
\scriptsize
\renewcommand{\arraystretch}{1.00}
\setlength{\tabcolsep}{1.5pt}
\resizebox{\textwidth}{!}{%
\begin{tabular}{@{}ll|cc@{}}
\toprule
Attack & Method & Alpaca & LongForm \\
\midrule
\multicolumn{4}{c}{\textbf{LLaDA-8B}} \\
\midrule
\multirow{5}{*}{Sentence Back Translation}
& DenMark & (\textbf{88.64} / \textbf{91.58} / \textbf{96.34} / \textbf{0.9801}) & (\textbf{94.00} / \textbf{97.67} / 98.67 / \textbf{0.9971}) \\
& DLM-KGW & (61.00 / 69.09 / 83.20 / 0.9613) & (55.20 / 64.40 / 81.00 / 0.9657) \\
& DGMark & (76.43 / 83.16 / 90.91 / 0.9606) & (92.67 / 95.33 / \textbf{99.00} / 0.9943) \\
& PatternMark & (76.17 / 79.42 / 88.45 / 0.9699) & (82.33 / 84.67 / 94.00 / 0.9870) \\
& UMR & (82.67 / 86.00 / 93.67 / 0.9778) & (75.00 / 79.67 / 92.00 / 0.9761) \\
\addlinespace[1pt]
\cmidrule(lr){1-4}
\multirow{5}{*}{Document Back Translation}
& DenMark & (\textbf{88.64} / \textbf{93.41} / \textbf{97.07} / \textbf{0.9883}) & (\textbf{96.33} / \textbf{98.33} / \textbf{99.33} / \textbf{0.9980}) \\
& DLM-KGW & (65.77 / 71.78 / 86.72 / 0.9660) & (53.60 / 61.60 / 83.80 / 0.9663) \\
& DGMark & (82.83 / 86.87 / 93.94 / 0.9708) & (94.00 / 96.67 / 99.00 / 0.9950) \\
& PatternMark & (77.98 / 81.95 / 90.25 / 0.9812) & (78.33 / 84.33 / 93.67 / 0.9852) \\
& UMR & (84.67 / 87.33 / 93.33 / 0.9852) & (80.67 / 86.00 / 93.33 / 0.9851) \\
\bottomrule
\end{tabular}%
}
\par\vspace{4pt}
\resizebox{\textwidth}{!}{%
\begin{tabular}{@{}ll|cc@{}}
\toprule
Attack & Method & Alpaca & LongForm \\
\midrule
\multicolumn{4}{c}{\textbf{Dream-v0-Instruct-7B}} \\
\midrule
\multirow{5}{*}{Sentence Back Translation}
& DenMark & (\textbf{92.65} / \textbf{94.69} / \textbf{97.14} / \textbf{0.9930}) & (\textbf{98.67} / \textbf{99.33} / \textbf{99.33} / \textbf{0.9987}) \\
& DLM-KGW & (61.15 / 69.06 / 86.69 / 0.9651) & (46.96 / 56.91 / 78.45 / 0.9548) \\
& DGMark & (87.34 / 90.25 / 93.36 / 0.9686) & (93.00 / 95.00 / 98.67 / 0.9969) \\
& PatternMark & (79.48 / 84.72 / 92.58 / 0.9760) & (83.96 / 88.40 / 96.93 / 0.9900) \\
& UMR & (63.67 / 70.33 / 88.00 / 0.9715) & (67.33 / 76.33 / 89.67 / 0.9790) \\
\addlinespace[1pt]
\cmidrule(lr){1-4}
\multirow{5}{*}{Document Back Translation}
& DenMark & (\textbf{94.29} / \textbf{95.10} / \textbf{97.55} / \textbf{0.9957}) & (\textbf{98.33} / \textbf{99.00} / \textbf{99.33} / \textbf{0.9984}) \\
& DLM-KGW & (56.12 / 63.67 / 82.01 / 0.9554) & (35.36 / 43.09 / 67.40 / 0.9157) \\
& DGMark & (88.38 / 91.08 / 95.02 / 0.9789) & (89.33 / 93.67 / 96.67 / 0.9925) \\
& PatternMark & (83.84 / 87.34 / 93.89 / 0.9869) & (81.23 / 84.98 / 91.81 / 0.9837) \\
& UMR & (64.33 / 72.67 / 87.67 / 0.9678) & (58.33 / 67.67 / 86.00 / 0.9711) \\
\bottomrule
\end{tabular}%
}
\end{center}

\par

\FloatBarrier
\subsection{Compression and Expansion at Different Granularities}
\label{app:document_compress_expand}
Tables~\ref{tab:compression_expansion_granularity} and~\ref{tab:compression_expansion_granularity_part2} compare sentence-level compression with document-level compression and expansion. GPT-4o-mini transforms each sentence independently in the sentence-level setting, whereas each complete response is transformed in one call in the document-level settings, allowing discourse reorganization and changes to sentence boundaries. Across the 24 document-level comparisons, DenMark has the highest TPR@1\%FPR, TPR@5\%FPR, and AUC in 22 cases for each metric, and remains strongest overall.

\begin{center}
\begin{minipage}{\textwidth}
\centering
\captionof{table}{\textbf{Robustness to compression and expansion on LLaDA-8B and LLaDA1.5-8B.} Sentence compression transforms each sentence independently, whereas document compression and expansion transform the complete response in one call. Each entry reports TPR@0.5\%FPR / TPR@1\%FPR / TPR@5\%FPR / AUC.}
\label{tab:compression_expansion_granularity}
\scriptsize
\renewcommand{\arraystretch}{0.96}
\setlength{\tabcolsep}{1.5pt}
\resizebox{\textwidth}{!}{%
\begin{tabular}{@{}ll|ccc@{}}
\toprule
Attack & Method & Finance & Alpaca & LongForm \\
\midrule
\multicolumn{5}{c}{\textbf{LLaDA-8B}} \\
\midrule
\multirow{5}{*}{Sentence Compression} & DenMark & (\textbf{68.90} / \textbf{79.60} / \textbf{87.96} / \textbf{0.9713}) & (\textbf{63.00} / \textbf{73.63} / \textbf{85.71} / \textbf{0.9615}) & (\textbf{83.00} / \textbf{86.33} / \textbf{94.00} / \textbf{0.9776}) \\
 & DLM-KGW & (37.20 / 47.00 / 68.40 / 0.9266) & (42.53 / 48.76 / 68.26 / 0.9048) & (37.60 / 48.20 / 68.20 / 0.9265) \\
 & DGMark & (62.33 / 74.00 / 86.00 / 0.9647) & (57.58 / 65.66 / 75.08 / 0.9093) & (69.33 / 78.67 / 92.00 / \textbf{0.9776}) \\
 & PatternMark & (51.00 / 55.33 / 75.33 / 0.9281) & (45.85 / 50.90 / 72.56 / 0.9221) & (50.67 / 55.67 / 77.33 / 0.9464) \\
 & UMR & (47.67 / 55.00 / 74.67 / 0.9358) & (58.67 / 66.33 / 78.67 / 0.9520) & (57.33 / 65.00 / 81.33 / 0.9517) \\
\addlinespace[1pt]\cmidrule(lr){1-5}
\multirow{5}{*}{Document Compression} & DenMark & (\textbf{48.16} / \textbf{61.54} / \textbf{79.26} / \textbf{0.9499}) & (47.25 / \textbf{60.44} / \textbf{82.78} / \textbf{0.9384}) & (\textbf{58.00} / \textbf{71.00} / \textbf{88.67} / \textbf{0.9674}) \\
 & DLM-KGW & (21.60 / 29.80 / 51.80 / 0.8763) & (29.67 / 37.14 / 57.05 / 0.8890) & (13.00 / 17.80 / 34.20 / 0.7853) \\
 & DGMark & (42.00 / 51.33 / 69.67 / 0.9255) & (42.09 / 53.87 / 72.39 / 0.9138) & (33.67 / 46.00 / 71.00 / 0.9309) \\
 & PatternMark & (27.33 / 31.33 / 57.00 / 0.8935) & (31.05 / 39.71 / 59.93 / 0.8899) & (18.00 / 23.33 / 42.67 / 0.8251) \\
 & UMR & (32.33 / 43.00 / 64.00 / 0.9164) & (\textbf{47.33} / 52.00 / 69.00 / 0.9152) & (26.00 / 33.33 / 57.33 / 0.8853) \\
\addlinespace[1pt]\cmidrule(lr){1-5}
\multirow{5}{*}{Document Expansion} & DenMark & (\textbf{50.17} / \textbf{61.87} / \textbf{79.60} / \textbf{0.9471}) & (\textbf{46.52} / \textbf{58.61} / \textbf{76.56} / \textbf{0.9392}) & (\textbf{67.00} / \textbf{74.33} / \textbf{86.00} / \textbf{0.9662}) \\
 & DLM-KGW & (7.00 / 11.00 / 26.60 / 0.7569) & (8.92 / 14.32 / 28.63 / 0.7427) & (5.60 / 8.20 / 20.00 / 0.6788) \\
 & DGMark & (11.00 / 18.00 / 41.67 / 0.8551) & (14.81 / 21.55 / 43.10 / 0.8374) & (9.33 / 20.67 / 46.00 / 0.8781) \\
 & PatternMark & (24.00 / 30.00 / 50.67 / 0.8447) & (21.66 / 27.80 / 41.88 / 0.8271) & (10.33 / 12.33 / 32.67 / 0.7545) \\
 & UMR & (10.67 / 13.33 / 28.33 / 0.7456) & (24.00 / 29.00 / 44.67 / 0.7892) & (11.00 / 14.33 / 32.00 / 0.7900) \\
\bottomrule
\end{tabular}%
}\par\vspace{5pt}

\resizebox{\textwidth}{!}{%
\begin{tabular}{@{}ll|ccc@{}}
\toprule
Attack & Method & Finance & Alpaca & LongForm \\
\midrule
\multicolumn{5}{c}{\textbf{LLaDA1.5-8B}} \\
\midrule
\multirow{5}{*}{Sentence Compression} & DenMark & (\textbf{68.56} / \textbf{78.60} / \textbf{87.29} / \textbf{0.9682}) & (\textbf{64.31} / \textbf{73.85} / \textbf{85.87} / \textbf{0.9490}) & (\textbf{80.00} / \textbf{86.33} / \textbf{94.00} / \textbf{0.9806}) \\
 & DLM-KGW & (37.40 / 48.00 / 68.80 / 0.9301) & (41.44 / 52.16 / 70.52 / 0.9229) & (41.67 / 50.67 / 71.00 / 0.9291) \\
 & DGMark & (62.33 / 68.00 / 85.67 / 0.9622) & (63.09 / 69.46 / 81.54 / 0.9334) & (69.33 / 77.67 / 89.67 / 0.9771) \\
 & PatternMark & (53.85 / 63.21 / 77.59 / 0.9498) & (45.10 / 53.85 / 69.23 / 0.9058) & (44.33 / 55.00 / 78.00 / 0.9540) \\
 & UMR & (53.33 / 62.33 / 79.00 / 0.9467) & (57.67 / 61.67 / 77.33 / 0.9312) & (58.00 / 64.00 / 82.00 / 0.9537) \\
\addlinespace[1pt]\cmidrule(lr){1-5}
\multirow{5}{*}{Document Compression} & DenMark & (42.81 / 53.18 / \textbf{79.26} / \textbf{0.9503}) & (47.00 / \textbf{61.13} / \textbf{78.80} / \textbf{0.9493}) & (\textbf{53.00} / \textbf{67.67} / \textbf{85.33} / \textbf{0.9629}) \\
 & DLM-KGW & (24.80 / 31.20 / 55.20 / 0.8873) & (33.40 / 39.38 / 58.56 / 0.8888) & (20.67 / 28.67 / 49.00 / 0.8441) \\
 & DGMark & (\textbf{44.00} / \textbf{54.33} / 70.67 / 0.9269) & (46.64 / 58.39 / 73.83 / 0.9220) & (44.00 / 50.00 / 70.67 / 0.9381) \\
 & PatternMark & (26.09 / 35.45 / 58.53 / 0.8857) & (36.01 / 41.26 / 63.64 / 0.8977) & (16.33 / 20.67 / 39.00 / 0.7990) \\
 & UMR & (37.33 / 44.33 / 64.33 / 0.9091) & (\textbf{47.67} / 52.67 / 70.00 / 0.9167) & (29.67 / 38.00 / 58.33 / 0.8920) \\
\addlinespace[1pt]\cmidrule(lr){1-5}
\multirow{5}{*}{Document Expansion} & DenMark & (\textbf{44.15} / \textbf{58.86} / \textbf{79.26} / \textbf{0.9493}) & (\textbf{45.58} / \textbf{57.95} / \textbf{75.62} / \textbf{0.9249}) & (\textbf{63.33} / \textbf{73.33} / \textbf{86.67} / \textbf{0.9661}) \\
 & DLM-KGW & (7.20 / 12.00 / 26.00 / 0.7501) & (12.78 / 17.53 / 33.20 / 0.7656) & (7.00 / 10.67 / 22.33 / 0.7295) \\
 & DGMark & (15.00 / 21.33 / 43.67 / 0.8609) & (16.11 / 25.50 / 47.99 / 0.8389) & (17.33 / 24.67 / 55.33 / 0.8900) \\
 & PatternMark & (22.07 / 28.76 / 50.17 / 0.8571) & (21.68 / 26.92 / 48.60 / 0.8085) & (6.33 / 10.33 / 34.33 / 0.7687) \\
 & UMR & (9.33 / 12.00 / 29.67 / 0.7634) & (20.33 / 25.67 / 43.00 / 0.7984) & (10.00 / 17.67 / 36.67 / 0.7832) \\
\bottomrule
\end{tabular}%
}
\end{minipage}
\end{center}

\begin{center}
\begin{minipage}{\textwidth}
\centering
\captionof{table}{\textbf{Robustness to compression and expansion on LLaDA2.0-mini and Dream-v0-Instruct-7B.} Sentence compression transforms each sentence independently, whereas document compression and expansion transform the complete response in one call. Each entry reports TPR@0.5\%FPR / TPR@1\%FPR / TPR@5\%FPR / AUC.}
\label{tab:compression_expansion_granularity_part2}
\scriptsize
\renewcommand{\arraystretch}{0.96}
\setlength{\tabcolsep}{1.5pt}
\resizebox{\textwidth}{!}{%
\begin{tabular}{@{}ll|ccc@{}}
\toprule
Attack & Method & Finance & Alpaca & LongForm \\
\midrule
\multicolumn{5}{c}{\textbf{LLaDA2.0-mini}} \\
\midrule
\multirow{5}{*}{Sentence Compression} & DenMark & (\textbf{87.96} / \textbf{91.30} / \textbf{97.66} / 0.9927) & (59.57 / 65.34 / 82.67 / 0.9445) & (\textbf{95.67} / \textbf{96.67} / \textbf{99.00} / \textbf{0.9977}) \\
 & DLM-KGW & (68.67 / 78.00 / 92.00 / 0.9831) & (53.41 / 63.08 / 76.70 / 0.9404) & (81.00 / 86.33 / 95.00 / 0.9901) \\
 & DGMark & (85.00 / 88.67 / 97.33 / \textbf{0.9937}) & (\textbf{65.31} / \textbf{73.43} / \textbf{85.61} / \textbf{0.9592}) & (83.00 / 87.33 / 96.00 / 0.9935) \\
 & PatternMark & (64.00 / 73.00 / 90.67 / 0.9792) & (45.99 / 50.52 / 70.03 / 0.9174) & (78.33 / 83.67 / 92.33 / 0.9847) \\
 & UMR & (51.67 / 57.00 / 77.67 / 0.9533) & (41.00 / 49.00 / 67.00 / 0.9035) & (63.67 / 72.67 / 87.33 / 0.9724) \\
\addlinespace[1pt]\cmidrule(lr){1-5}
\multirow{5}{*}{Document Compression} & DenMark & (\textbf{47.49} / \textbf{59.53} / 79.60 / 0.9492) & (44.04 / 55.23 / 71.48 / 0.9266) & (\textbf{61.33} / \textbf{68.00} / \textbf{79.67} / \textbf{0.9530}) \\
 & DLM-KGW & (35.67 / 40.67 / 64.67 / 0.9164) & (43.73 / 51.97 / 69.18 / \textbf{0.9381}) & (30.00 / 38.33 / 64.67 / 0.9059) \\
 & DGMark & (45.67 / 58.00 / \textbf{80.00} / \textbf{0.9578}) & (\textbf{50.92} / \textbf{57.20} / \textbf{74.54} / 0.9359) & (47.33 / 54.67 / 76.67 / 0.9390) \\
 & PatternMark & (29.00 / 35.67 / 56.00 / 0.8779) & (32.40 / 35.54 / 55.75 / 0.8572) & (31.33 / 36.67 / 64.33 / 0.9119) \\
 & UMR & (20.00 / 28.67 / 52.00 / 0.8572) & (26.33 / 33.00 / 55.33 / 0.8650) & (23.33 / 29.67 / 52.00 / 0.8728) \\
\addlinespace[1pt]\cmidrule(lr){1-5}
\multirow{5}{*}{Document Expansion} & DenMark & (\textbf{66.22} / \textbf{72.58} / \textbf{88.63} / \textbf{0.9695}) & (\textbf{55.96} / \textbf{60.29} / \textbf{77.98} / \textbf{0.9461}) & (\textbf{81.33} / \textbf{87.00} / \textbf{95.33} / \textbf{0.9860}) \\
 & DLM-KGW & (9.33 / 16.00 / 36.67 / 0.8133) & (26.16 / 34.05 / 56.99 / 0.8795) & (12.00 / 19.67 / 45.00 / 0.8452) \\
 & DGMark & (13.67 / 20.33 / 46.67 / 0.8737) & (26.94 / 35.06 / 56.46 / 0.9038) & (15.33 / 23.67 / 54.33 / 0.9043) \\
 & PatternMark & (21.67 / 24.33 / 48.33 / 0.8424) & (25.44 / 29.97 / 49.48 / 0.8366) & (18.33 / 23.00 / 47.33 / 0.8692) \\
 & UMR & (6.33 / 10.67 / 30.67 / 0.7653) & (15.67 / 20.33 / 40.00 / 0.7861) & (8.00 / 14.33 / 36.33 / 0.8137) \\
\bottomrule
\end{tabular}%
}\par\vspace{5pt}

\resizebox{\textwidth}{!}{%
\begin{tabular}{@{}ll|ccc@{}}
\toprule
Attack & Method & Finance & Alpaca & LongForm \\
\midrule
\multicolumn{5}{c}{\textbf{Dream-v0-Instruct-7B}} \\
\midrule
\multirow{5}{*}{Sentence Compression} & DenMark & (\textbf{85.00} / \textbf{88.67} / \textbf{95.67} / \textbf{0.9915}) & (71.43 / 76.73 / 88.16 / 0.9568) & (\textbf{93.00} / \textbf{95.67} / \textbf{99.33} / \textbf{0.9958}) \\
 & DLM-KGW & (49.23 / 58.67 / 73.98 / 0.9294) & (45.12 / 57.32 / 76.22 / 0.9113) & (50.83 / 59.12 / 79.01 / 0.9405) \\
 & DGMark & (42.89 / 49.65 / 68.30 / 0.8941) & (\textbf{74.67} / \textbf{81.58} / \textbf{90.13} / \textbf{0.9696}) & (67.00 / 76.67 / 90.00 / 0.9787) \\
 & PatternMark & (70.13 / 77.18 / 89.93 / 0.9790) & (55.90 / 62.01 / 78.17 / 0.9423) & (59.73 / 67.24 / 82.59 / 0.9551) \\
 & UMR & (39.00 / 48.67 / 72.33 / 0.9376) & (41.33 / 47.67 / 69.33 / 0.9209) & (41.67 / 48.67 / 72.33 / 0.9501) \\
\addlinespace[1pt]\cmidrule(lr){1-5}
\multirow{5}{*}{Document Compression} & DenMark & (\textbf{56.67} / \textbf{66.67} / \textbf{87.67} / \textbf{0.9736}) & (\textbf{60.00} / \textbf{68.16} / \textbf{82.86} / \textbf{0.9477}) & (\textbf{60.00} / \textbf{70.33} / \textbf{88.00} / \textbf{0.9714}) \\
 & DLM-KGW & (17.36 / 22.57 / 45.14 / 0.7947) & (30.22 / 38.49 / 56.12 / 0.8427) & (13.26 / 16.02 / 34.25 / 0.7570) \\
 & DGMark & (25.20 / 32.38 / 46.11 / 0.8048) & (43.78 / 53.94 / 71.37 / 0.9142) & (31.00 / 42.67 / 62.67 / 0.8915) \\
 & PatternMark & (25.84 / 31.54 / 55.70 / 0.8860) & (39.30 / 45.41 / 65.50 / 0.9219) & (20.82 / 26.62 / 47.44 / 0.8704) \\
 & UMR & (16.67 / 21.33 / 44.00 / 0.8574) & (32.67 / 40.00 / 60.67 / 0.8886) & (12.33 / 19.00 / 42.00 / 0.8456) \\
\addlinespace[1pt]\cmidrule(lr){1-5}
\multirow{5}{*}{Document Expansion} & DenMark & (\textbf{52.00} / \textbf{58.33} / \textbf{80.67} / \textbf{0.9534}) & (\textbf{47.35} / \textbf{55.10} / \textbf{75.92} / \textbf{0.9235}) & (\textbf{60.33} / \textbf{70.00} / \textbf{86.00} / \textbf{0.9675}) \\
 & DLM-KGW & (2.08 / 3.82 / 17.71 / 0.7009) & (7.19 / 10.79 / 22.66 / 0.7592) & (1.66 / 2.21 / 6.63 / 0.5998) \\
 & DGMark & (7.17 / 14.14 / 28.89 / 0.7662) & (23.44 / 30.50 / 50.62 / 0.8701) & (9.33 / 17.33 / 42.33 / 0.8534) \\
 & PatternMark & (24.16 / 27.52 / 50.34 / 0.8386) & (26.20 / 30.57 / 51.97 / 0.8552) & (18.77 / 22.18 / 41.98 / 0.8086) \\
 & UMR & (8.33 / 12.00 / 28.67 / 0.7741) & (17.33 / 22.33 / 44.33 / 0.8103) & (5.33 / 10.67 / 26.00 / 0.7678) \\
\bottomrule
\end{tabular}%
}
\end{minipage}
\end{center}

\par

\subsection{Robustness to Non-Uniform Local Length Changes}
\label{app:nonuniform_local_length_attacks}

\paragraph{Motivation and attack design.}
DenMark uses fixed-size semantic units during generation, and each candidate size \(m'\) in the detector is applied uniformly across the entire document. To stress-test this global unit-size design, we introduce three attacks that induce different length changes in different parts of the same text, so that a single candidate size need not recover all original unit boundaries.
\begin{enumerate}
    \item \textbf{Alternating compression/expansion.} Starting with compression, we alternate sentence-level compression to 60--70\% of the original word count and expansion, producing adjacent sentences with opposite length changes.
    \item \textbf{Random runs.} Each document starts in a randomly selected compression or expansion mode. At each sentence boundary, the mode switches with probability 0.35, creating consecutive runs with different local length changes. For a multi-sentence document assigned only one mode, one sentence is switched to ensure that both transformations occur.
    \item \textbf{Variable local compression.} We partition each document into up to three contiguous, sentence-aligned regions, with approximately equal sentence counts, and assign target word-count ratios of 50--60\%, 60--70\%, and 70--80\% in a seeded random order. Documents with fewer than three sentences use fewer regions.
\end{enumerate}

\begin{center}
\begin{minipage}{\textwidth}
\centering
\captionof{table}{\textbf{Non-uniform local length attack robustness on Dream-v0-Instruct-7B.} Results are averaged over FinanceQA and LongFormQA. We report TPR@0.5\%FPR, TPR@1\%FPR, TPR@5\%FPR, and AUC.}
\label{tab:nonuniform_local_length_attacks}
\small
\begin{tabular}{@{}llrrrr@{}}
\toprule
Attack & Method & TPR@0.5\% & TPR@1\% & TPR@5\% & AUC \\
\midrule
\multirow{5}{*}{\shortstack[l]{Alternating\\compression/expansion}}
& DenMark & \textbf{79.33} & \textbf{85.50} & \textbf{94.83} & \textbf{0.9827} \\
& DLM-KGW & 10.99 & 17.53 & 36.30 & 0.7783 \\
& DGMark & 16.19 & 23.00 & 40.81 & 0.8118 \\
& PatternMark & 41.60 & 48.71 & 70.87 & 0.9384 \\
& UMR & 18.33 & 26.17 & 49.00 & 0.8622 \\
\midrule
\multirow{5}{*}{Random runs}
& DenMark & \textbf{65.00} & \textbf{76.50} & \textbf{90.17} & \textbf{0.9740} \\
& DLM-KGW & 13.47 & 19.22 & 36.18 & 0.7867 \\
& DGMark & 18.33 & 24.42 & 43.08 & 0.8151 \\
& PatternMark & 43.13 & 49.38 & 72.41 & 0.9275 \\
& UMR & 21.00 & 27.50 & 53.00 & 0.8809 \\
\midrule
\multirow{5}{*}{\shortstack[l]{Variable local\\compression}}
& DenMark & \textbf{82.17} & \textbf{88.83} & \textbf{96.17} & \textbf{0.9905} \\
& DLM-KGW & 51.09 & 55.28 & 70.72 & 0.9183 \\
& DGMark & 48.25 & 53.71 & 69.10 & 0.9002 \\
& PatternMark & 56.81 & 64.76 & 83.89 & 0.9563 \\
& UMR & 35.17 & 48.33 & 70.67 & 0.9369 \\
\bottomrule
\end{tabular}
\end{minipage}
\end{center}

\paragraph{Results.}
DenMark achieves the highest TPR at all three operating points and the highest AUC under each attack. Its TPR@1\%FPR is 85.50\%, 76.50\%, and 88.83\% for alternating compression/expansion, random runs, and variable local compression, respectively. These results demonstrate robustness to the evaluated non-uniform local length changes despite using a global candidate unit size. They are consistent with some units retaining enough semantic watermark evidence for the aggregated document score to remain discriminative even when other units are misaligned; they do not establish invariance to arbitrary local length distortions.

\par

\FloatBarrier
\clearpage
\section{Additional Evaluations, Baselines, and Ablations}
\label{app:additional_evaluations}

\subsection{Generation Quality Evaluation}
\label{app:generation_quality}

Table~\ref{tab:generation_quality_deltas} compares generation quality using $\Delta\,\mathrm{mean}(\log\mathrm{PPL})$ and GPT-4o-mini judge scores. DenMark has the lowest average $\Delta$LogPPL among the evaluated watermark methods and the highest average judge score. The quality measurements stay in the same range as the strongest baselines while DenMark gives higher detection under the attacks reported in Table~\ref{tab:c4_empirical_roc_main} and Appendix~\ref{app:llada15_semantic_attack_results}.

\begin{table}[!htb]
\vspace{3pt}
\centering
\caption{\textbf{Generation quality relative to clean outputs.} We report the absolute Qwen2.5-32B-Instruct $\mathrm{mean}(\log\mathrm{PPL})$ and the GPT-4o-mini average judge score over Style, Consistency, Accuracy, and Ethics. In both quality columns, the parenthesized value is the difference from the corresponding clean reference.}
\label{tab:generation_quality_deltas}

\scriptsize
\renewcommand{\arraystretch}{0.93}
\setlength{\tabcolsep}{3.0pt}
\resizebox{\textwidth}{!}{%
\begin{tabular}{@{}llcc@{\hspace{14pt}}llcc@{}}
\toprule
\multicolumn{4}{c}{\textbf{LLaDA-8B}} & \multicolumn{4}{c}{\textbf{LLaDA1.5-8B}} \\
\cmidrule(r){1-4}\cmidrule(l){5-8}
Dataset & Method & Mean LogPPL ($\Delta$Clean) & Judge Avg. ($\Delta$Clean) & Dataset & Method & Mean LogPPL ($\Delta$Clean) & Judge Avg. ($\Delta$Clean) \\
\midrule
\multirow{6}{*}{Finance} & Clean & 1.47 (+0.00) & 7.92 (+0.00) & \multirow{6}{*}{Finance} & Clean & 1.52 (+0.00) & 7.91 (+0.00) \\
 & DenMark & 1.83 (+0.36) & 7.21 (-0.71) &  & DenMark & 1.87 (+0.36) & 7.36 (-0.55) \\
 & DLM-KGW & 1.82 (+0.34) & 7.40 (-0.52) &  & DLM-KGW & 1.87 (+0.35) & 7.56 (-0.35) \\
 & DGMark & 1.91 (+0.44) & 7.11 (-0.81) &  & DGMark & 1.96 (+0.44) & 7.37 (-0.54) \\
 & PatternMark & 2.16 (+0.69) & 6.97 (-0.95) &  & PatternMark & 2.21 (+0.69) & 7.17 (-0.74) \\
 & UMR & 1.78 (+0.30) & 7.17 (-0.75) &  & UMR & 1.81 (+0.29) & 7.29 (-0.62) \\
\addlinespace[1pt]
\multirow{6}{*}{Alpaca} & Clean & 1.36 (+0.00) & 7.56 (+0.00) & \multirow{6}{*}{Alpaca} & Clean & 1.40 (+0.00) & 7.70 (+0.00) \\
 & DenMark & 1.77 (+0.40) & 7.22 (-0.34) &  & DenMark & 1.78 (+0.37) & 7.31 (-0.39) \\
 & DLM-KGW & 1.74 (+0.38) & 7.30 (-0.26) &  & DLM-KGW & 1.75 (+0.34) & 7.43 (-0.27) \\
 & DGMark & 1.78 (+0.42) & 6.94 (-0.62) &  & DGMark & 1.84 (+0.44) & 7.09 (-0.61) \\
 & PatternMark & 2.07 (+0.70) & 6.99 (-0.56) &  & PatternMark & 2.05 (+0.64) & 7.25 (-0.45) \\
 & UMR & 1.72 (+0.36) & 6.91 (-0.64) &  & UMR & 1.71 (+0.30) & 7.13 (-0.57) \\
\addlinespace[1pt]
\multirow{6}{*}{LongForm} & Clean & 1.52 (+0.00) & 7.83 (+0.00) & \multirow{6}{*}{LongForm} & Clean & 1.56 (+0.00) & 7.92 (+0.00) \\
 & DenMark & 1.89 (+0.37) & 7.37 (-0.46) &  & DenMark & 1.91 (+0.35) & 7.50 (-0.42) \\
 & DLM-KGW & 1.90 (+0.38) & 7.23 (-0.60) &  & DLM-KGW & 1.94 (+0.38) & 7.61 (-0.31) \\
 & DGMark & 1.97 (+0.45) & 7.18 (-0.65) &  & DGMark & 2.00 (+0.44) & 7.46 (-0.46) \\
 & PatternMark & 2.33 (+0.81) & 6.75 (-1.07) &  & PatternMark & 2.34 (+0.78) & 7.01 (-0.92) \\
 & UMR & 1.84 (+0.32) & 7.09 (-0.74) &  & UMR & 1.86 (+0.30) & 7.32 (-0.61) \\
\bottomrule
\end{tabular}%
}
\vspace{5pt}
\resizebox{\textwidth}{!}{%
\begin{tabular}{@{}llcc@{\hspace{14pt}}llcc@{}}
\toprule
\multicolumn{4}{c}{\textbf{LLaDA2.0-mini}} & \multicolumn{4}{c}{\textbf{Dream-v0-Instruct-7B}} \\
\cmidrule(r){1-4}\cmidrule(l){5-8}
Dataset & Method & Mean LogPPL ($\Delta$Clean) & Judge Avg. ($\Delta$Clean) & Dataset & Method & Mean LogPPL ($\Delta$Clean) & Judge Avg. ($\Delta$Clean) \\
\midrule
\multirow{6}{*}{Finance} & Clean & 2.24 (+0.00) & 7.61 (+0.00) & \multirow{6}{*}{Finance} & Clean & 1.78 (+0.00) & 7.62 (+0.00) \\
 & DenMark & 2.61 (+0.37) & 7.11 (-0.50) &  & DenMark & 2.02 (+0.24) & 7.42 (-0.20) \\
 & DLM-KGW & 2.58 (+0.34) & 7.20 (-0.41) &  & DLM-KGW  & 2.06 (+0.29) & 7.15 (-0.46) \\
 & DGMark & 2.64 (+0.40) & 7.29 (-0.32) &  & DGMark & 2.11 (+0.33) & 7.41 (-0.21) \\
 & PatternMark & 2.76 (+0.52) & 6.86 (-0.75) &  & PatternMark & 2.15 (+0.38) & 7.09 (-0.53) \\
 & UMR & 2.52 (+0.28) & 6.61 (-1.01) &  & UMR & 2.70 (+0.92) & 6.62 (-1.00) \\
\addlinespace[1pt]
\multirow{6}{*}{Alpaca} & Clean & 2.09 (+0.00) & 7.54 (+0.00) & \multirow{6}{*}{Alpaca} & Clean & 1.79 (+0.00) & 7.37 (+0.00) \\
 & DenMark & 2.51 (+0.43) & 7.13 (-0.41) &  & DenMark & 2.10 (+0.31) & 7.14 (-0.23) \\
 & DLM-KGW & 2.45 (+0.36) & 7.27 (-0.27) &  & DLM-KGW  & 2.13 (+0.33) & 6.85 (-0.52) \\
 & DGMark & 2.50 (+0.42) & 7.24 (-0.30) &  & DGMark & 2.24 (+0.45) & 6.94 (-0.43) \\
 & PatternMark & 2.65 (+0.56) & 6.85 (-0.69) &  & PatternMark & 2.26 (+0.46) & 6.94 (-0.43) \\
 & UMR & 2.42 (+0.33) & 6.72 (-0.82) &  & UMR & 2.86 (+1.07) & 6.62 (-0.75) \\
\addlinespace[1pt]
\multirow{6}{*}{LongForm} & Clean & 2.39 (+0.00) & 7.70 (+0.00) & \multirow{6}{*}{LongForm} & Clean & 1.91 (+0.00) & 7.50 (+0.00) \\
 & DenMark & 2.76 (+0.38) & 7.18 (-0.52) &  & DenMark & 2.09 (+0.18) & 7.36 (-0.14) \\
 & DLM-KGW & 2.78 (+0.39) & 7.22 (-0.48) &  & DLM-KGW  & 2.34 (+0.44) & 6.85 (-0.65) \\
 & DGMark & 2.81 (+0.42) & 7.32 (-0.38) &  & DGMark & 2.05 (+0.14) & 7.34 (-0.16) \\
 & PatternMark & 2.94 (+0.56) & 7.04 (-0.66) &  & PatternMark & 2.35 (+0.44) & 6.82 (-0.68) \\
 & UMR & 2.75 (+0.36) & 6.58 (-1.12) &  & UMR & 2.80 (+0.89) & 6.54 (-0.96) \\
\bottomrule
\end{tabular}%
}
\vspace{6pt}
\begin{tabular}{@{}lcc@{}}
\toprule
\multicolumn{3}{c}{\textbf{Average}} \\
\midrule
Method & Mean LogPPL ($\Delta$Clean) & Judge Avg. ($\Delta$Clean) \\
\midrule
Clean & 1.75 (+0.00) & 7.68 (+0.00) \\
DenMark & 2.10 (+0.34) & 7.28 (-0.41) \\
DLM-KGW  & 2.11 (+0.36) & 7.26 (-0.42) \\
DGMark & 2.15 (+0.40) & 7.22 (-0.46) \\
PatternMark & 2.36 (+0.60) & 6.98 (-0.70) \\
UMR & 2.23 (+0.48) & 6.88 (-0.80) \\
\bottomrule
\end{tabular}
\vspace{3pt}
\par\vspace{2pt}\begin{minipage}{0.98\linewidth}\scriptsize PPL uses EOS-corrected Dream cohorts and shared same-base, same-dataset clean references. Judge scores for Dream clean and DGMark are recomputed on the corrected cohorts; other absolute judge scores are unchanged. Judge differences use the corresponding same-base, same-dataset clean reference.\end{minipage}
\end{table}

\par

\clearpage
\subsection{Semantic Baselines with Block Decoding}
\label{app:blockwise_semantic_baseline_comparison}

Some DLMs generate sequentially across blocks and denoise tokens jointly within each block. This structure allows an approximate adaptation of semantic watermarks developed for ARLMs by treating each block as a semantic unit. We evaluate adaptations of SemStamp~\citep{hou2024semstamp} and PMark~\citep{huo2026pmark}. For each block, the method generates several complete block candidates, evaluates them under the corresponding semantic watermark criterion, and selects a candidate before moving to the next block. The adaptation depends on block decoding and on the chosen block size. Short blocks can contain too little semantic content for stable scoring. Large blocks provide fewer watermark selection opportunities and require completion of longer candidates before each decision.

\paragraph{Configurations.}
All three blockwise baselines use 25 token blocks, random remasking, and a higher sampling temperature of $0.9$ to increase candidate diversity and strengthen watermark embedding. Block Best-of-$K$ selects the highest scoring candidate among $K=16$ complete block candidates using a 2 bit keyed semantic score. PMark uses its prior offline variant with 16 candidates and two keyed semantic channels. SemStamp uses a 2 bit LSH partition with acceptance rate $0.25$ and margin $0.02$, sampling proposals in batches of 16 for at most 100 trials per block.

\begin{table}[!htb]
\centering
\caption{Clean detection for semantic baselines with block decoding. We report Qwen2.5-32B-Instruct $\Delta\,\mathrm{mean}(\log\mathrm{PPL})$, TPR at 0.5\%, 1\%, and 5\% FPR, and AUC.}
\label{tab:blockwise_semantic_clean_detection}
\tiny
\renewcommand{\arraystretch}{1.08}
\setlength{\tabcolsep}{1.6pt}
\resizebox{\textwidth}{!}{%
\begin{tabular}{@{}llccccc@{\hspace{10pt}}llccccc@{}}
\toprule
\multicolumn{7}{c}{\textbf{LLaDA-8B}}
&
\multicolumn{7}{c}{\textbf{LLaDA1.5-8B}} \\
\cmidrule(r){1-7}
\cmidrule(l){8-14}
Dataset & Method & $\Delta$LogPPL & TPR@0.5\% & TPR@1\% & TPR@5\% & AUC
&
Dataset & Method & $\Delta$LogPPL & TPR@0.5\% & TPR@1\% & TPR@5\% & AUC \\
\midrule

\multirow{4}{*}{Finance}
& DenMark & \textbf{+0.36} & \textbf{99.00} & \textbf{99.67} & \textbf{100.00} & \textbf{0.9996}
&
\multirow{4}{*}{Finance}
& DenMark & \textbf{+0.36} & \textbf{100.00} & \textbf{100.00} & \textbf{100.00} & \textbf{0.9994} \\

& Block Best-of-$K$ & +1.21 & 54.00 & 68.00 & 88.00 & 0.9825
&
& Block Best-of-$K$ & +1.28 & 60.00 & 78.00 & 96.00 & 0.9759 \\

& PMark & +1.29 & 72.00 & 74.00 & 88.00 & 0.9867
&
& PMark & +1.09 & 62.00 & 72.00 & 88.00 & 0.9829 \\

& SemStamp & +1.36 & 32.00 & 40.00 & 48.00 & 0.8328
&
& SemStamp & +1.31 & 18.00 & 22.00 & 46.00 & 0.7609 \\

\addlinespace[1pt]

\multirow{4}{*}{Alpaca}
& DenMark & \textbf{+0.40} & \textbf{95.24} & \textbf{97.07} & \textbf{98.53} & \textbf{0.9915}
&
\multirow{4}{*}{Alpaca}
& DenMark & \textbf{+0.37} & \textbf{95.41} & \textbf{96.47} & \textbf{98.59} & \textbf{0.9964} \\

& Block Best-of-$K$ & +1.43 & 68.00 & 78.00 & 94.00 & 0.9758
&
& Block Best-of-$K$ & +1.04 & 78.00 & 90.00 & 96.00 & 0.9811 \\

& PMark & +1.18 & 42.00 & 46.00 & 66.00 & 0.8411
&
& PMark & +1.19 & 26.53 & 30.61 & 61.22 & 0.8208 \\

& SemStamp & +1.12 & 38.00 & 50.00 & 70.00 & 0.8850
&
& SemStamp & +1.13 & 32.00 & 38.00 & 44.00 & 0.8195 \\

\addlinespace[1pt]

\multirow{4}{*}{LongForm}
& DenMark & \textbf{+0.37} & \textbf{99.67} & \textbf{100.00} & \textbf{100.00} & \textbf{0.9998}
&
\multirow{4}{*}{LongForm}
& DenMark & \textbf{+0.35} & \textbf{99.33} & \textbf{99.67} & \textbf{100.00} & \textbf{0.9993} \\

& Block Best-of-$K$ & +1.18 & 80.00 & 94.00 & 98.00 & 0.9957
&
& Block Best-of-$K$ & +1.24 & 76.00 & 82.00 & 98.00 & 0.9941 \\

& PMark & +1.15 & 50.00 & 58.00 & 72.00 & 0.9575
&
& PMark & +0.99 & 62.00 & 70.00 & 84.00 & 0.9439 \\

& SemStamp & +1.36 & 38.78 & 40.82 & 59.18 & 0.8329
&
& SemStamp & +1.21 & 34.00 & 42.00 & 60.00 & 0.8419 \\

\bottomrule
\end{tabular}%
}

\vspace{3pt}
\begin{minipage}{0.99\textwidth}
\scriptsize
\end{minipage}
\end{table}

\paragraph{Results.}
Table~\ref{tab:blockwise_semantic_clean_detection} shows lower detection for these block adaptations than for DenMark and larger $\Delta\,\mathrm{mean}(\log\mathrm{PPL})$. Complete candidates from the same block often receive similar semantic scores, which leaves a small selection margin. Each block also contributes one selection decision. DenMark scores local updates throughout denoising, so several small score gains can accumulate inside a block. These two differences account for the stronger watermark signal observed in the table.

\FloatBarrier
\par

\subsection{Effect of Candidate and Channel Budgets}
\label{app:tied_k_bc_ablation}

We jointly vary the candidate budget $K$ and the number of semantic channels $C$ over $(8,1)$, $(16,2)$, $(24,3)$, and $(32,4)$, while fixing the semantic unit size to $m=25$, the average rollout count to three, the candidate temperature to $0.6$, and random position selection. Table~\ref{tab:tied_k_bc_gpt_attacks} shows that $(K,C)=(24,3)$ provides strong clean and attacked detection across both datasets, whereas increasing the budgets further does not yield a consistent additional gain. Overall, the 24-candidate, 3-channel configuration offers a favorable trade-off between detection performance and generation-time scoring cost.

\begin{table}[H]
\centering
\caption{Clean detection and robustness to GPT attacks on LLaDA-8B. We jointly vary the candidate budget $K$ and channel count $C$ over $(8,1)$, $(16,2)$, $(24,3)$, and $(32,4)$, and report TPR at 0.5\%, 1\%, and 5\% FPR and AUC.}
\label{tab:tied_k_bc_gpt_attacks}
\renewcommand{\arraystretch}{0.96}
\setlength{\tabcolsep}{0.6pt}
\begin{minipage}[t]{0.492\textwidth}
\centering
{\scriptsize
\begin{tabular}{@{}cclcccc@{}}
\toprule
\multicolumn{7}{c}{\textbf{Finance-QA}} \\
\midrule
$K$ & $C$ & Condition & TPR@0.5\% & TPR@1\% & TPR@5\% & AUC \\
\midrule
\multirow{5}{*}{8} & \multirow{5}{*}{1} & Clean & 97.49 & 97.99 & 98.99 & 0.9970 \\
& & Rewrite & 55.28 & 60.80 & 80.40 & 0.9444 \\
& & Compress & 52.26 & 57.79 & 75.88 & 0.9251 \\
& & Expand & 41.71 & 48.74 & 72.36 & 0.9145 \\
& & Doc. Paraphrase & 45.73 & 53.27 & 74.87 & 0.9257 \\
\cmidrule(lr){1-7}
\multirow{5}{*}{16} & \multirow{5}{*}{2} & Clean & 100.00 & 100.00 & 100.00 & 0.9991 \\
& & Rewrite & 72.24 & 81.61 & 92.64 & 0.9825 \\
& & Compress & 68.90 & 79.60 & 87.96 & 0.9713 \\
& & Expand & 53.85 & 67.89 & 82.94 & 0.9652 \\
& & Doc. Paraphrase & 62.54 & 73.58 & 85.95 & 0.9687 \\
\cmidrule(lr){1-7}
\multirow{5}{*}{24} & \multirow{5}{*}{3} & Clean & 99.50 & 99.50 & 99.50 & 0.9983 \\
& & Rewrite & \textbf{79.50} & \textbf{88.50} & \textbf{97.00} & \textbf{0.9875} \\
& & Compress & \textbf{82.00} & \textbf{85.50} & \textbf{92.00} & 0.9734 \\
& & Expand & \textbf{66.00} & \textbf{77.50} & \textbf{90.50} & \textbf{0.9798} \\
& & Doc. Paraphrase & \textbf{75.00} & \textbf{80.50} & \textbf{93.00} & \textbf{0.9817} \\
\cmidrule(lr){1-7}
\multirow{5}{*}{32} & \multirow{5}{*}{4} & Clean & 100.00 & 100.00 & 100.00 & 0.9993 \\
& & Rewrite & 77.00 & 83.50 & 94.00 & 0.9828 \\
& & Compress & 77.00 & 83.00 & 93.00 & \textbf{0.9796} \\
& & Expand & 65.00 & 72.00 & 89.50 & 0.9726 \\
& & Doc. Paraphrase & 68.50 & 77.00 & 89.50 & 0.9784 \\
\bottomrule
\end{tabular}
}
\end{minipage}\hfill
\begin{minipage}[t]{0.492\textwidth}
\centering
{\scriptsize
\begin{tabular}{@{}cclcccc@{}}
\toprule
\multicolumn{7}{c}{\textbf{AlpacaFarm}} \\
\midrule
$K$ & $C$ & Condition & TPR@0.5\% & TPR@1\% & TPR@5\% & AUC \\
\midrule
\multirow{5}{*}{8} & \multirow{5}{*}{1} & Clean & 93.09 & 93.62 & 96.81 & 0.9953 \\
& & Rewrite & 53.72 & 58.51 & 75.00 & 0.9427 \\
& & Compress & 49.47 & 54.79 & 71.81 & 0.9094 \\
& & Expand & 44.68 & 53.19 & 69.15 & 0.9323 \\
& & Doc. Paraphrase & 45.74 & 51.60 & 71.28 & 0.9259 \\
\cmidrule(lr){1-7}
\multirow{5}{*}{16} & \multirow{5}{*}{2} & Clean & 94.62 & 97.31 & 98.92 & 0.9966 \\
& & Rewrite & 71.32 & 80.15 & 88.97 & 0.9718 \\
& & Compress & 63.00 & 73.63 & 85.71 & 0.9615 \\
& & Expand & 61.54 & 71.06 & 83.88 & 0.9481 \\
& & Doc. Paraphrase & 58.97 & 74.73 & 87.91 & 0.9665 \\
\cmidrule(lr){1-7}
\multirow{5}{*}{24} & \multirow{5}{*}{3} & Clean & 96.22 & 97.30 & 99.46 & 0.9978 \\
& & Rewrite & \textbf{76.22} & \textbf{81.08} & 93.51 & 0.9835 \\
& & Compress & \textbf{70.81} & \textbf{77.84} & \textbf{88.65} & \textbf{0.9633} \\
& & Expand & \textbf{70.27} & \textbf{76.76} & \textbf{89.19} & \textbf{0.9767} \\
& & Doc. Paraphrase & \textbf{69.19} & \textbf{75.68} & 87.03 & \textbf{0.9746} \\
\cmidrule(lr){1-7}
\multirow{5}{*}{32} & \multirow{5}{*}{4} & Clean & 97.37 & 98.42 & 98.95 & 0.9982 \\
& & Rewrite & 74.74 & 81.58 & \textbf{95.26} & \textbf{0.9859} \\
& & Compress & 68.42 & 76.84 & 86.32 & \textbf{0.9633} \\
& & Expand & 62.11 & 70.00 & 86.84 & 0.9704 \\
& & Doc. Paraphrase & 67.37 & 75.26 & \textbf{88.42} & 0.9708 \\
\bottomrule
\end{tabular}
}
\end{minipage}
\end{table}

\par

\subsection{Downstream Task Performance}
\label{app:downstream_performance}

Following DGMark~\citep{hong2026dgmark}, we evaluate the downstream capability of LLaDA1.5-8B after watermark insertion on four benchmarks: MMLU~\citep{hendrycks2021measuring}, HellaSwag~\citep{zellers2019hellaswag}, ARC-Challenge~\citep{clark2018think}, and GSM8K~\citep{cobbe2021training}. All methods use sampling temperature $0.1$ on the complete official evaluation splits. We report accuracy for each task and the unweighted average across the four tasks.

\begin{table}[!htb]
\centering
\caption{\textbf{Downstream task performance under watermarking on LLaDA1.5-8B.} Accuracy (\%) on MMLU, HellaSwag, ARC-Challenge, and GSM8K. Bold values indicate the best result among watermarking methods in each column.}
\label{tab:downstream_task_performance}

\small
\renewcommand{\arraystretch}{1.10}
\setlength{\tabcolsep}{7pt}

\begin{tabular}{@{}lccccc@{}}
\toprule
Method
& MMLU $\uparrow$
& HellaSwag $\uparrow$
& ARC-C $\uparrow$
& GSM8K $\uparrow$
& Avg. $\uparrow$ \\
\midrule

\multicolumn{6}{c}{\textbf{LLaDA1.5-8B}} \\
\midrule

Clean
& 63.97 & 76.33 & 84.28 & 57.62 & 70.55 \\
DLM-KGW
& 52.63 & 72.46 & 73.58 & \textbf{53.22} & 62.97 \\
PatternMark
& 55.63 & 73.64 & 80.27 & 51.40 & 65.23 \\
DenMark
& \textbf{63.63}
& \textbf{76.34}
& \textbf{83.95}
& 51.78
& \textbf{68.92} \\
DGMark
& 57.93 & 73.99 & 82.27 & 52.31 & 66.63 \\
UMR
& 56.87 & 72.75 & 79.26 & 52.69 & 65.39 \\
\bottomrule
\end{tabular}
\end{table}

\paragraph{Results.}
On LLaDA1.5-8B, DenMark reaches an average accuracy of 68.92\%, a 1.63 percentage point drop from the clean model. It has the highest accuracy among the evaluated watermark methods on MMLU, HellaSwag, and ARC-Challenge, while most of the remaining drop occurs on GSM8K. Relative to the next best watermark method by average accuracy, DenMark gains 2.29 points.

\FloatBarrier
\par

\section{Proofs}
\label{app:proofs}

\subsection{Proof of Proposition~\ref{prop:scan_fpr_control}}
\label{app:scan_fpr_control_proof}

Fix a length bin \(\ell\) with \(\Pr(h(Y)=\ell)>0\). Throughout the conditional argument, write \(\Pr_\ell(\cdot)=\Pr(\cdot\mid h(Y)=\ell)\) and let \(\mathbb E_\ell\) denote the corresponding expectation.

\begin{lemma}[Superuniformity of the empirical right tail p value]
\label{lem:empirical_p_superuniform}
Fix \(m'\in\mathcal G\). Under the conditional exchangeability assumption in Proposition~\ref{prop:scan_fpr_control},
\[
\Pr_\ell\!\left(p_{m'}(Y)\leq \alpha\right)
\leq
\alpha,
\qquad
\text{for every }\alpha\in[0,1].
\]
\end{lemma}

\begin{proof}
Write \(Y_0=Y\) and enumerate the calibration texts in \(\mathcal D_{0,\ell}\) as \(Y_1,\ldots,Y_{n_0}\), where \(n_0=|\mathcal D_{0,\ell}|\). Set
\[
Z_i=S_{m'}(Y_i),
\qquad
\mathrm{rk}_i=\sum_{j=0}^{n_0}\mathbb I\{Z_j\geq Z_i\}.
\]
The corrected empirical right tail $p$ value used in the paper is
\[
p_{m'}(Y)=\frac{\mathrm{rk}_0}{n_0+1},
\]
a standard finite sample construction for exchangeable Monte Carlo scores~\citep{phipson2010permutation}.

For any integer \(r\in\{0,\ldots,n_0+1\}\), conditional exchangeability implies that the indices \(0,\ldots,n_0\) have identical marginal rank distributions. Hence
\begin{align*}
\Pr_\ell(\mathrm{rk}_0\leq r)
&=
\frac{1}{n_0+1}
\sum_{i=0}^{n_0}\Pr_\ell(\mathrm{rk}_i\leq r) \\
&=
\mathbb E_\ell\!\left[
 \frac{1}{n_0+1}
 \sum_{i=0}^{n_0}\mathbb I\{\mathrm{rk}_i\leq r\}
\right].
\end{align*}
For every realized collection \(Z_0,\ldots,Z_{n_0}\), at most \(r\) indices can satisfy \(\mathrm{rk}_i\leq r\). To see this, let \(A=\{i:\mathrm{rk}_i\leq r\}\). If \(A\neq\emptyset\), choose \(i_\star\in A\) with the smallest score among indices in \(A\). Every \(j\in A\) then satisfies \(Z_j\geq Z_{i_\star}\), so
\[
\mathrm{rk}_{i_\star}
=
\sum_{j=0}^{n_0}\mathbb I\{Z_j\geq Z_{i_\star}\}
\geq
|A|.
\]
Since \(\mathrm{rk}_{i_\star}\leq r\), we have \(|A|\leq r\). The argument covers ties because the rank uses \(\geq\). Hence,
\[
\Pr_\ell(\mathrm{rk}_0\leq r)
\leq
\frac{r}{n_0+1}.
\]

For arbitrary \(\alpha\in[0,1]\), let \(r=\lfloor(n_0+1)\alpha\rfloor\). Since \(\mathrm{rk}_0\) is integer-valued,
\begin{align*}
\Pr_\ell\!\left(p_{m'}(Y)\leq \alpha\right)
&=
\Pr_\ell\!\left(\mathrm{rk}_0\leq(n_0+1)\alpha\right) \\
&=
\Pr_\ell(\mathrm{rk}_0\leq r) \\
&\leq
\frac{r}{n_0+1}
\leq
\alpha.
\end{align*}
This proves the stated superuniformity.
\end{proof}

\begin{proof}
Let \(J=|\mathcal G|\). For \(0\leq\alpha<1\), the definition of \(p_{\mathrm{scan}}\) gives
\begin{align*}
\{p_{\mathrm{scan}}(Y)\leq\alpha\}
&=
\left\{
 J\min_{m'\in\mathcal G}p_{m'}(Y)
 \leq\alpha
\right\} \\
&=
\bigcup_{m'\in\mathcal G}
\left\{
 p_{m'}(Y)\leq\frac{\alpha}{J}
\right\}.
\end{align*}
Applying the union bound within bin \(\ell\), followed by Lemma~\ref{lem:empirical_p_superuniform}, yields the standard Bonferroni control~\citep{dunn1961multiple}:
\begin{align*}
\Pr_\ell\!\left(p_{\mathrm{scan}}(Y)\leq\alpha\right)
&\leq
\sum_{m'\in\mathcal G}
\Pr_\ell\!\left(
 p_{m'}(Y)\leq\frac{\alpha}{J}
\right) \\
&\leq
\sum_{m'\in\mathcal G}\frac{\alpha}{J}
=
\alpha.
\end{align*}
The case \(\alpha=1\) is immediate. By the law of total probability over length bins,
\[
\Pr\!\left(p_{\mathrm{scan}}(Y)\leq\alpha\right)
=\sum_{\ell:\Pr(h(Y)=\ell)>0}\Pr(h(Y)=\ell)\,
\Pr_\ell\!\left(p_{\mathrm{scan}}(Y)\leq\alpha\right)
\leq\alpha.
\]
The proof uses no independence assumption across candidate unit sizes.
\end{proof}

\FloatBarrier
\subsection{Formal Setup and Proof of Proposition~\ref{prop:cumulative_continued_policy_decomposition}}
\label{app:cumulative_continued_policy_decomposition_proof}

\paragraph{Reference policy and joint updates.}
Let \(\pi\) denote the DenMark decoder and let \(\pi^0\) be the reference decoder that uses the same candidate sampling rule, draws one local candidate for each selected semantic unit, and commits that candidate directly. At step \(t\), let
\[
\mathcal B_t
=
\{b\in[B]: b \text{ is selected for update at step } t\}
\]
denote the set of active semantic units. For each \(b\in\mathcal B_t\), DenMark selects \(k_{t,b}^\star\) according to~\eqref{eq:candidate_selection}. Vectors indexed by \(b\in\mathcal B_t\) have \(|\mathcal B_t|\) components ordered by increasing unit index. Define the joint candidate index vector \(\mathbf k_t^\star=(k_{t,b}^\star)_{b\in\mathcal B_t}\), and let \(\mathbf 1_t=(1,\ldots,1)\in\{1\}^{|\mathcal B_t|}\) denote the reference choice that selects the first candidate for every active unit. We couple these first candidates with the candidates committed by \(\pi^0\).

For a joint candidate index vector \(\mathbf k_t=(k_{t,b})_{b\in\mathcal B_t}\), where \(k_{t,b}\in[K_{t,b}]\), let \(c_{t,\mathbf k_t}\) denote the state obtained by jointly committing the corresponding local updates \(\{c_{t,b,k_{t,b}}\}_{b\in\mathcal B_t}\).

\paragraph{Continuation value and step advantage.}
At step \(t\), consider a reverse hybrid trajectory in which the first \(t\) decisions follow \(\pi^0\). Define the continuation value as
\begin{equation}
\label{eq:continued_policy_value}
Q_{t,\mathbf k_t}^{\pi}
=
\mathbb E_{\pi}\!\left[
S_m(\hat y)
\,\middle|\,
x^{(t+1)}\leftarrow c_{t,\mathbf k_t}
\right],
\end{equation}
where the joint candidate update is committed at the current step and all subsequent decisions follow \(\pi\). The expected step advantage of the selected candidate over the reference choice is
\begin{equation}
\label{eq:continued_policy_advantage}
\Delta_t^{\pi}
=
\mathbb E\!\left[
Q_{t,\mathbf k_t^\star}^{\pi}
-
Q_{t,\mathbf 1_t}^{\pi}
\right].
\end{equation}

\begin{proof}
For any decoding policy \(\mu\), write
\[
J(\mu) =
\mathbb E_{\mu}\!\left[
S_m(\hat y)
\,\middle|\,
x^{(0)}
\right].
\]

For each \(\ell\in\{0,\ldots,T\}\), define the reverse hybrid policy
\(\pi^{[\ell]}\) to use \(\pi^0\) for decisions \(0,\ldots,\ell-1\) and
\(\pi\) for decisions \(\ell,\ldots,T-1\).
The endpoints are
\[
\pi^{[0]}=\pi,
\qquad
\pi^{[T]}=\pi^0.
\]

Fix \(t\in\{0,\ldots,T-1\}\). The adjacent hybrids \(\pi^{[t]}\) and \(\pi^{[t+1]}\) execute the same policy
\(\pi^0\) for their first \(t\) decisions, so they induce the same distribution
over the state \(x^{(t)}\) at step \(t\). Given a common realized state, couple the selection of active units, candidate construction, and rollout randomness at this step.

Under this coupling, \(\pi^{[t]}\) commits the joint candidate selected from rollout scores
\(c_{t,\mathbf k_t^\star}\), and \(\pi^{[t+1]}\) commits the
reference joint candidate \(c_{t,\mathbf 1_t}\). After the current commitment, both hybrids follow \(\pi\) for all remaining decisions.

By the definition of~\eqref{eq:continued_policy_value},
\[
J\!\left(\pi^{[t]}\right) - J\!\left(\pi^{[t+1]}\right) =
\mathbb E\!\left[
Q_{t,\mathbf k_t^\star}^{\pi} - Q_{t,\mathbf 1_t}^{\pi}
\right]
=
\Delta_t^\pi.
\]

Summing over \(t\) telescopes:
\begin{align*}
\mathbb E_{\pi}\!\left[
S_m(\hat y)\mid x^{(0)}
\right]
-
\mathbb E_{\pi^0}\!\left[
S_m(\hat y)\mid x^{(0)}
\right]
&=
J(\pi^{[0]})-J(\pi^{[T]})
\\
&=
\sum_{t=0}^{T-1}
\left[
J(\pi^{[t]})-J(\pi^{[t+1]})
\right]
\\
&=
\sum_{t=0}^{T-1}\Delta_t^\pi.
\end{align*}
\end{proof}

\FloatBarrier
\section{Qualitative Examples}
\label{app:qualitative_examples}

Tables~\ref{tab:qualitative-unattacked-llada8b}--\ref{tab:qualitative-unattacked-dream} compare clean and DenMark generations for aligned prompts across three models and datasets. The examples remain coherent and prompt relevant without a visible watermark marker; displayed passages are truncated, while detection uses the complete generations.

\begin{table}[!htbp]
\centering
\caption{Prompt-aligned clean and DenMark outputs from LLaDA-8B; outputs are truncated.}
\label{tab:qualitative-unattacked-llada8b}
\scriptsize
\setlength{\tabcolsep}{3.5pt}
\renewcommand{\arraystretch}{1.02}
\begin{tabular}{@{}P{0.09\textwidth}P{0.16\textwidth}P{0.34\textwidth}P{0.34\textwidth}@{}}
\toprule
Dataset & Prompt & Clean (no watermark) & DenMark (unattacked) \\
\midrule
Finance & Should I replace bonds in a passive investment strategy & The decision to replace bonds in a passive investment strategy depends on various factors, including your investment goals, risk tolerance, and market conditions. Here are some key points to consider: 1. Risk Tolerance: Bonds are generally considered low-risk investments, making them suitable for a passive investment strategy. If you're considering diversifying your portfolio to take on higher risk, you might consider replacing some or all bonds with equities. [...continues] & Whether to replace bonds in a passive investment strategy depends on several factors, including individual goals, risk tolerance, and time horizon. However, I'll provide some perspectives to help you make an informed decision. 1. Lower risk: Bonds are generally less volatile compared to stocks, which makes them a good option for those seeking low volatility. 2. [...continues] \\
Alpaca & Why do we need AI assistants? & AI assistants are essential for several reasons in modern society: 1. Efficiency: AI assistants can automate repetitive tasks, freeing up human workers to focus on more high-priority, creative activities. 2. Productivity: They can perform tasks quickly and accurately, increasing overall productivity and reducing human error. 3. [...continues] & We need AI assistants for several reasons: 1. Automation: AI assistants can automate tasks that are typically repetitive, time-consuming and require minimal human intervention. This frees up time for humans to focus on more important activities. 2. Speed and Accuracy: AI systems can process information faster, more accurately and with fewer errors than human operators. [...continues] \\
LongForm & Why do humans like to kiss each other? & Humans like to kiss each other for several complex biological and psychological reasons. Firstly, kissing is a form of social bonding that helps to strengthen relationships and emotional connection between friends, family, and partners. When we kiss, it triggers the release of oxytocin, a hormone that promotes attachment and bonding, as well as endorphins, which is associated with feelings of pleasure and relaxation. [...continues] & Kissing is a universal human behavior that is deeply embedded in cultures and traditions. There are several biological and psychological reasons why humans like to kiss each other. 1. Oxytocin: Kissing releases oxytocin, also known as the "love hormone." Oxytocin is a neurotransmitter that improves mood and creates a feeling of bonding, affection, and trust. 2. Physical connection: Kissing involves close physical contact, which creates a strong connection. [...continues] \\
\bottomrule
\end{tabular}
\end{table}

\par
\begin{table}[!htbp]
\centering
\caption{Prompt-aligned clean and DenMark outputs from LLaDA1.5-8B; outputs are truncated.}
\label{tab:qualitative-unattacked-llada15}
\scriptsize
\setlength{\tabcolsep}{3.5pt}
\renewcommand{\arraystretch}{1.02}
\begin{tabular}{@{}P{0.09\textwidth}P{0.16\textwidth}P{0.34\textwidth}P{0.34\textwidth}@{}}
\toprule
Dataset & Prompt & Clean (no watermark) & DenMark (unattacked) \\
\midrule
Finance & How to calculate how much a large stock position is really worth? & To calculate the real value of a large stock position, consider the following factors: 1. Stock Price: Determine the current stock price of the company. You can find this on stock exchanges, investment websites, or financial news. 2. Number of Shares: Know the total number of shares you own in the stock position. This can be found in your investment account, trading app, or brokerage statement. 3. [...continues] & Calculating the value of a large stock position involves a few key factors, including the number of shares you own, the price per share, if applicable, and any dividends or appreciation. Here’s a step-by-step guide to determining the total value: 1. Determine the Number of Shares: Count how many shares you own in the stock. 2. [...continues] \\
Alpaca & how much of a threat is climate change in the coming years, and what should we do to stop it? & Climate change is a very real threat to our environment and society. According to the IPCC, the Earth's temperature is rising at an alarming rate, and the trend is projected to continue accelerating if we don't take immediate action. The impacts of climate change are already evident, with melting glaciers, rising sea levels, and changing weather patterns, among other things. To stop climate change, we need to take a comprehensive approach: 1. [...continues] & Climate change poses a significant threat to humanity in the future. It is real, primarily driven by the burning of fossil fuels, oil, and gas, and human activities such as deforestation and agriculture. It is estimated that by the end of this century, the global average temperature could be up to 2°C higher than pre-industrial levels. [...continues] \\
LongForm & What causes mental fatigue while studying? & Mental fatigue while studying can be caused by a variety of factors, or a combination of them. These factors include: 1. Lack of sleep: Not getting enough hours of sleep per night can impair your ability to function properly, leading to mental fatigue while studying. 2. Hunger: Low blood sugar can cause your energy levels to drop, leading to mental fog and fatigue. 3. [...continues] & Mental fatigue, also known as cognitive overload or burnout, is a state of mental and emotional exhaustion caused by prolonged mental effort. When studying, individuals experience mental fatigue when they absorb excessive amounts of information, leading to decreased concentration, motivation, and memory retention. Some common causes of mental fatigue include: 1. Lack of rest: Cognitive function requires adequate sleep and rest to recover. [...continues] \\
\bottomrule
\end{tabular}
\end{table}

\begin{center}
\begin{minipage}{\textwidth}
\centering
\captionof{table}{Prompt-aligned clean and DenMark outputs from Dream-v0-Instruct-7B; outputs are truncated.}
\label{tab:qualitative-unattacked-dream}
\scriptsize
\setlength{\tabcolsep}{3.5pt}
\renewcommand{\arraystretch}{1.02}
\begin{tabular}{@{}P{0.09\textwidth}P{0.16\textwidth}P{0.34\textwidth}P{0.34\textwidth}@{}}
\toprule
Dataset & Prompt & Clean (no watermark) & DenMark (unattacked) \\
\midrule
Finance & Should I Use an Investment Professional? & Yes, using an investment professional can help you plan and achieve your financial goals. An investment professional can help you plan, evaluate various investment options, and manage your investment portfolio. Firstly, using an investment professional can help you to determine your financial goals, objectives, and needs. An investment professional can help you create an investment plan that aligns your financial goals, risk tolerance, and investment horizon. [...continues] & The decision to use an investment professional is personal and based on individual circumstances, goals, and preferences. It depends on an individual's level of knowledge and experience in investing. If you are able to manage your own investments, you have a good understanding of financial markets and the economic environment. [...continues] \\
Alpaca & ideas for building a waitlist for a new product or service & 1. Build anticipation on social media platforms by sharing sneak peeks of the product or service and encouraging followers to subscribe to the waitlist as soon as it launches. 2. Offer an exclusive discount or special promotion to those who register for the waitlist before it is available. 3. Create an email marketing campaign that sends updates with the latest information about the product or service as it becomes available. 4. [...continues] & Sure, here are some ideas for building a waitlist for a new product or service: 1. Social media promotion: Share teasers, behind-the-scenes, or other information about your product or service on social media channels. Encourage your followers to provide their contact information for a chance for exclusive access. 2. Online launch event: Host an online launch event or a webinar where you encourage potential customers to sign up for a waitlist. [...continues] \\
LongForm & Why does a 2 hour movie take a year to make when a high quality TV show can make 45 minute ones every week. & Producing a 2-hour movie involves many different stages, from the development of the script, to casting, filming, and editing. This process involves many professionals, including writers, directors, actors, and editors, and can take several months. On the other hand, TV shows follow a similar process, with scriptwriting, pre-production, production, and editing for each episode. [...continues] & There are several key factors that contribute to the longer time to produce a 2-hour movie in comparison to a half-hour-long TV show. One main difference is that movies typically involve a greater degree of complexity in terms of storytelling, effects, cinematography, and editing, while TV shows can be more streamlined and made quickly due to their regular production schedules. [...continues] \\
\bottomrule
\end{tabular}
\end{minipage}
\end{center}

\end{document}